\documentclass[onefignum,onetabnum]{siamonline220329}

\usepackage{colortbl}
\usepackage{booktabs}
\usepackage{amsfonts}
\usepackage{dsfont}
\usepackage{graphicx}
\usepackage[caption=false]{subfig}
\usepackage{wrapfig}
\usepackage{pgfplots}
\pgfplotsset{compat=1.17}
\usepackage[algo2e,linesnumbered,lined,noend]{algorithm2e}
\SetKwComment{tcp}{$\triangleright$\ }{}
\ifpdf
  \DeclareGraphicsExtensions{.eps,.pdf,.png,.jpg}
\else
  \DeclareGraphicsExtensions{.eps}
\fi

\usepackage{todonotes}
\usepackage{enumitem}
\setlist[enumerate]{leftmargin=.5in}
\setlist[itemize]{leftmargin=.5in}

\newsiamremark{remark}{Remark}
\newsiamremark{hypothesis}{Hypothesis}
\crefname{hypothesis}{Hypothesis}{Hypotheses}
\newsiamthm{claim}{Claim}

\headers{Conditionally Elicitable Dynamic Risk for Deep RL}{A. Coache, S. Jaimungal, and Á. Cartea}

\title{Conditionally Elicitable Dynamic Risk Measures for Deep Reinforcement Learning\thanks{\today.
\funding{AC acknowledges support from the Natural Sciences and Engineering Research Council of Canada (grants CGSD3-2019-534435). SJ acknowledges support from the Natural Sciences and Engineering Research Council of Canada (grants RGPIN-2018-05705 and RGPAS-2018-522715) and the Data Science Institute Catalyst Grant.}}}

\author{Anthony Coache\thanks{Department of Statistical Sciences, University of Toronto, Canada
  (\email{anthony.coache@mail.utoronto.ca}, \url{anthonycoache.ca}; \email{sebastian.jaimungal@utoronto.ca}, \url{sebastian.statistics.utoronto.ca}).}
\and Sebastian Jaimungal\footnotemark[2]
\and Álvaro Cartea\thanks{Oxford-Man Institute of Quantitative Finance, Mathematical Institute, University of Oxford, United Kingdom 
  (\email{alvaro.cartea@maths.ox.ac.uk}, \url{sites.google.com/site/alvarocartea/home}).}}

\usepackage{amsopn}

\newcommand{\ML}{ML} 
\newcommand{\RL}{RL} 
\newcommand{\MDP}{MDP} 
\newcommand{\DP}{DP} 
\newcommand{\ANN}{ANN} 
\newcommand{\BSDE}{BSDE} 
\newcommand{\SDE}{SDE} 
\newcommand{\CVaR}{\text{CVaR}} 
\newcommand{\VaR}{\text{VaR}} 
\newcommand{\EM}{EM} 
\newcommand{\PnL}{P\&L} 
\newcommand{\CDF}{CDF} 
\newcommand{\triangleineq}{$\triangle$ ineq.} 

\newcommand{\statespace}{\Ss}
\newcommand{\state}{s}
\newcommand{\statedum}{\state'}

\newcommand{\actionspace}{\Aa}
\newcommand{\action}{a}

\newcommand{\costspace}{\Cc}
\newcommand{\costfunc}{c}
\newcommand{\periodspace}{\Tt}
\newcommand{\eplength}{T}
\newcommand{\timeidx}{t}
\newcommand{\policy}{\pi}
\newcommand{\policyparams}{\theta}
\newcommand{\valuefunc}{V}
\newcommand{\valueparams}{\psi}
\newcommand{\Nbatchs}{B}
\newcommand{\batchidx}{b}
\newcommand{\lrate}{\eta}
\newcommand{\lagrangian}{L}
\newcommand{\grad}[1]{\nabla_{#1}}
\newcommand{\riskmeas}{\rho}

\newcommand{\weight}{\xi}

\newcommand{\spectrum}{\varphi}
\newcommand{\Lpspace}{\Yy}
\newcommand{\rvdum}{Z}

\newcommand{\rv}{Y}
\newcommand{\rvsupp}{\YY}
\newcommand{\estim}{\mathfrak{a}}
\newcommand{\estimsupp}{\AA}
\newcommand{\score}{S}
\newcommand{\statmap}{M}
\newcommand{\price}{S}
\newcommand{\pricespace}{\PosReals}
\newcommand{\Nats}{\NN}
\newcommand{\Reals}{\RR}
\newcommand{\PosReals}{\Reals_+}
\newcommand{\dee}{\mathrm{d}}
\DeclareMathOperator*{\argmin}{\arg\min}
\newcommand{\Ind}{\mathds 1}
\newcommand{\suchthat}{\;\ifnum\currentgrouptype=16 \middle\fi|\;}
\DeclareMathOperator*{\esssup}{ess\sup}
\newcommand{\normaldist}{\Nn}
\DeclareMathOperator{\diag}{diag} 
\newcommand{\tr}{^\top} 

\renewcommand{\AA}{\mathbb{A}}

\newcommand{\EE}{\mathbb{E}}
\newcommand{\FF}{\mathbb{F}}

\newcommand{\NN}{\mathbb{N}}

\newcommand{\PP}{\mathbb{P}}

\newcommand{\RR}{\mathbb{R}}

\newcommand{\YY}{\mathbb{Y}}

\newcommand{\Aa}{\mathcal{A}}

\newcommand{\Cc}{\mathcal{C}}

\newcommand{\Ff}{\mathcal{F}}

\newcommand{\Ll}{\mathcal{L}}

\newcommand{\Nn}{\mathcal{N}}

\newcommand{\Pp}{\mathcal{P}}

\newcommand{\Ss}{\mathcal{S}}
\newcommand{\Tt}{\mathcal{T}}
\newcommand{\Uu}{\mathcal{U}}

\newcommand{\Yy}{\mathcal{Y}}

\definecolor{mblue}{rgb}{0.098,0.18,0.357}
\definecolor{mred}{rgb}{0.902,0.4157,0.0196}

\ifpdf
\hypersetup{
  pdftitle={Conditionally Elicitable Dynamic Risk for Deep RL},
  pdfauthor={A. Coache, S. Jaimungal, and Á. Cartea}
}
\fi

\begin{document}

\newcommand{\new}[1]{#1}
\newcommand{\newmath}[1]{#1}

\maketitle




\begin{abstract}
    We propose a novel framework to solve risk-sensitive reinforcement learning (\RL{}) problems where the agent optimises time-consistent dynamic spectral risk measures.
    Based on the notion of conditional elicitability, our methodology constructs (strictly consistent) scoring functions that are used as penalizers in the estimation procedure.
    Our contribution is threefold: we (i) devise an efficient approach to estimate a class of dynamic spectral risk measures with deep neural networks, (ii) prove that these dynamic spectral risk measures may be approximated to any arbitrary accuracy using deep neural networks, and (iii) develop a risk-sensitive actor-critic algorithm that uses full episodes and does not require any additional nested transitions.
    We compare our conceptually improved reinforcement learning algorithm with the nested simulation approach and illustrate its performance in two settings: statistical arbitrage and portfolio allocation on both simulated and real data.
\end{abstract}




\section{Introduction}
\label{sec:introduction}

One principled model-free framework for learning-based control is reinforcement learning (\RL{}) \cite{sutton2018reinforcement}.
In \RL{}, the agent observes feedback in the forms of costs from interactions with an environment, uses this information to update its current behaviour, and aims to discover the best course of action based on a certain objective.
In recent years, deep learning -- i.e. relying on neural network structures to approximate complex functions -- has shown remarkable success in \RL{} applications, ranging from mastering Atari 2600 video games \cite{mnih2015human} to developing autonomous image-learning robots \cite{levine2016end} and defeating world champions Go players \cite{silver2016mastering}.
It also has become an appealing alternative in several financial decision making problems, where one wishes to learn optimal strategies with no explicit assumptions about the environment.
For a thorough survey of recent advances in \RL{} applied to financial problems, see e.g. \cite{hambly2021recent,jaimungal2022reinforcement,hu2022recent}.

In \RL{}, the agent's optimisation problem must take into account the additional randomness due to the uncertainty in the environment.
One can think of, for instance, a trader concerned by the risks associated with financial assets, an autonomous vehicle which must pay attention to weather and road conditions, or a medical worker whose actions and treatments impact the life of its patients.
In most (if not all) real-life applications, there exists inherent uncertainty in the environment, and the agent must adapt its strategy to avoid potentially catastrophic consequences.
For an overview and outlook of safety concerns in \RL{} algorithms, see e.g. \cite{garcia2015comprehensive}.

There are numerous proposals that accounts for risk sensitivity in the literature, often called risk-aware or risk-sensitive \RL{} frameworks (see e.g. \cite{di2019practical,nass2019entropic,kalogerias2020better,bauerle2021minimizing,jaimungal2022robust}).
Risk-aware \RL{} proposes to quantify risk of random costs through risk measures (instead of the usual expectation) to account for environmental uncertainties.
This risk-awareness provides robustness to low-probability but high-cost outcomes in the environment, and allows more flexibility than traditional approaches because the agent may choose the risk measure considering its goals and risk tolerances.

In the extant literature, several authors address risk evaluation for sequential decision making problems by applying risk measures recursively to a sequence of cost random variables, and by optimising the risk in a \new{time-consistent and} model-free dynamic framework as additional information becomes available.
For instance, \cite{ruszczynski2010risk} evaluates the risk at each period using \emph{dynamic Markov coherent risk measures}, \cite{chu2014markov} and \cite{bauerle2021markov} propose iterated coherent risk measures, where they both derive risk-aware \emph{dynamic programming (\DP{}) equations} and provide \emph{policy iteration algorithms}, \cite{ahmadi2021constrained} investigate bounded policy iteration algorithms for \emph{partially observable Markov decision processes}, \cite{kose2021risk} prove the convergence of \emph{temporal difference algorithms} optimising dynamic Markov coherent risk measures, and \cite{cheng2022markov} derive a \DP{} principle for \emph{Kusuoka-type conditional risk mappings}.
However, these works require computing the value function for every possible state of the environment, limiting their applicability to problems with a small number of state-action pairs.
\new{In contrast, here we aim to develop computational approaches to optimise \RL{} problems that allows for both continuous or discrete states and actions.}

A recent development in risk-aware \RL{} is that in \cite{coache2021reinforcement}.
The authors use \emph{dynamic convex risk measures} and devise a model-free approach to solve \emph{finite-horizon} \RL{} problems in a \emph{time-consistent manner}.
This extends the work from \cite{tamar2016sequential} that studies optimal stationary policies under dynamic coherent risk measures. 
They also demonstrate the performance and flexibility of their approach on several benchmark examples, which\new{, by generating strategies that mitigate risk and not simply maximising expectation,} effectively accounts for uncertainty in the data-generating processes.
\new{In both works, one downside of the proposed actor-critic algorithms is the use of a \emph{nested simulation} or \emph{simulation upon simulation} approach.}
Such nested simulation approaches, where one generates (outer) episodes and (inner) transitions for every visited state, are computationally expensive, as it requires a large number of simulations to obtain accurate results.
In many real-world applications, simulations are costly -- the acquisition of new observations may not be possible, for instance in trading markets and medical trials -- making this methodology highly ineffective (in terms of memory and speed) or even impracticable.
One of our motivations is to develop a novel framework which circumvents simulating additional transitions.

If we restrict the broad class of convex risk measures to a narrower class of risk mappings that are \emph{elicitable}, one can achieve significant improvements in the estimation of the risk.
Although the term ``elicitability'' was established recently by \cite{lambert2008eliciting}, the general idea originates from the seminal work of \cite{osband1985providing}.
Elicitable (and conditionally elicitable) statistical functions have a corresponding loss function, often called \emph{scoring function}, that can be used as a penalizer when updating its point estimate -- a common example is the expectation which can be estimated using the mean squared error.
In the \RL{} literature, \cite{shen2014risk} suggest a Q-learning algorithm for a \new{class of elicitable mappings, i.e. \emph{utility-based shortfall risk measures}, in case of tabular Markovian decision processes}, while \cite{marzban2021deep} propose a modification for \emph{dynamic expectile risk measures} and continuous spaces.
Recently, \cite{fissler2016higher} provided characterisations of those scoring functions for many risk measures, including a class of static spectral risk measures.
To the best of our knowledge, there is no work that explores the notion of (strictly consistent) scoring functions for \new{a large class of conditionally elicitable} risk measures in a dynamic framework.

In this paper, we develop a practical and efficient framework to estimate dynamic spectral risk measures, which allows us to devise a conceptually improved methodology for actor-critic algorithms in risk-aware \RL{} problems.
\new{Our contributions are:
(i) devise a composite model regression framework using artificial neural networks (\ANN{}s) to estimate a class of elicitable dynamic spectral risk measures;
(ii) prove that our estimation method can approximate the dynamic risk to an arbitrary accuracy given sufficiently large neural network structures;
(iii) implement a risk-aware actor-critic algorithm using our proposed framework that exclusively uses full (outer) episodes, i.e. does not require additional (inner) transitions; and
(iv) validate our approach on a benchmark example, and illustrate its performance on a portfolio allocation problem using both simulated and real data.}

The remainder of the paper is organised as follows.
We provide a review of risk evaluation in a dynamic framework in \cref{sec:dynamic-risk}.
\Cref{sec:reinforcement-learning} formalises the \RL{} problems we investigate.
We then introduce the key concepts on elicitability in \cref{sec:elicitability}, and explain how one uses deep composite regression models to estimate elicitable mappings in \cref{sec:deep-regression}.
\Cref{sec:algo-framework} presents our proposed actor-critic algorithm, and \cref{sec:experiments} illustrates the performance of our novel framework.
Finally, we discuss the generalisation of the methodology to arbitrary dynamic spectral risk measures in \cref{sec:generalization-spectral}, and present future directions in \cref{sec:conclusion}.


\section{Dynamic Risk Setting}
\label{sec:dynamic-risk}

In this section, we provide a brief overview of dynamic risk measures, which, as additional information becomes available, assess the risk of sequences of random variables, such as cash-flows, in a dynamic framework.

Let $(\Omega, \Ff, \PP)$ be a probability space.
A \emph{static risk measure} is a mapping $\riskmeas: \Lpspace \rightarrow \bar{\Reals}$, quantifying the risk of a certain random variable $\rv \in \Lpspace$, that satisfies additional assumptions.
In what follows, we assume that $\Lpspace$ is the space of bounded $p$-integrable $\Ff$-measurable random variables, with $p \in [1,\infty]$, and all random variables are interpreted as random costs.
Next, we introduce some well-known risk measures commonly used in the literature.
\begin{definition}
    \label{def:value-at-risk}
    The \emph{value-at-risk} \cite{artzner1999coherent} with threshold $\alpha \in (0,1)$ of $\rv$, denoted $\VaR_{\alpha}(\rv)$, is given by the $\alpha$-quantile of the distribution of $\rv$.
\end{definition}
\begin{definition}
    \label{def:conditional-value-at-risk}
    The \emph{conditional value-at-risk} \cite{rockafellar2000optimization} with threshold $\alpha \in (0,1)$ of $\rv$ is given by 
    \begin{equation}
        \CVaR_{\alpha}(\rv) = \frac{1}{1-\alpha} \int_{\alpha}^{1} \VaR_{u}(\rv) \, \dee u.
    \end{equation}
\end{definition}
\begin{definition}
    \label{def:spectral-risk}
    A \emph{spectral risk measure} \cite{kusuoka2001law} of the random variable of $\rv$, denoted $\riskmeas^{\spectrum}$, is defined as
    \begin{equation}
        \riskmeas^{\spectrum} (\rv) = \int_{[0,1]} \CVaR_{\alpha}(\rv) \, \spectrum (\dee \alpha),
    \end{equation}
    where $\spectrum$ is a nonnegative, nonincreasing measure such that $\int_{[0,1]} \spectrum(\dee \alpha) = 1$ (also known as the spectrum).
\end{definition}
Spectral risk measures satisfy several properties shared with other risk measures in the literature.
We list some properties in the following proposition.
\begin{proposition}
    Let $\gamma \in \Reals$, $\beta > 0$ and $\rv,\rvdum \in \Lpspace$. A spectral risk measure $\riskmeas^{\spectrum}$ is said to be
    \begin{enumerate}
    	\item \emph{monotone}, i.e. $\rv \leq \rvdum$ implies $\riskmeas^{\spectrum}(\rv) \leq \riskmeas^{\spectrum}(\rvdum)$;
    	\item \emph{translation invariant}, i.e. $\riskmeas^{\spectrum}(\rv+\gamma) = \riskmeas^{\spectrum}(\rv)+\gamma$;
    	\item \emph{positive homogeneous}, i.e. $\riskmeas^{\spectrum}(\beta \rv) = \beta\, \riskmeas^{\spectrum}(\rv)$; and
    	\item \emph{subadditive}, i.e. $\riskmeas^{\spectrum}(\rv + \rvdum) \leq \riskmeas^{\spectrum}(\rv) + \riskmeas^{\spectrum}(\rvdum)$.
    \end{enumerate}    
\end{proposition}
Any risk measure satisfying these four properties is said to be \emph{coherent} \cite{artzner1999coherent}.
Spectral risk measures are coherent, while the converse does not always hold.

Employing static risk measures as an objective function in sequential decision making problems does not account for the temporal structure of intermediate costs that generates the terminal costs, and moreover leads to time-inconsistent solutions -- we give an illustrative example below.
Therefore, to properly monitor random variables at different times, we must adapt risk assessment to a dynamic framework.
Dynamic risk measures have the advantage of being time-consistent, a property that has a general appeal especially in finance applications, so optimising them leads to time-consistent solutions.

Various classes of risk measures have been extended to the dynamic case, such as coherent risk measures \cite{riedel2004dynamic}, distribution-invariant risk measures \cite{weber2006distribution}, convex risk measures \cite{frittelli2004dynamic}, and dynamic assessment indices \cite{bielecki2016dynamic}, among others.
There also exists an equivalence between backward stochastic differential equations (\BSDE{}s) and dynamic risk measures (see e.g. \cite{peng1997backward,cohen2011backward}), which is exploited by some authors to compute dynamic risk measures via a deep \BSDE{} method (see e.g. \cite{han2020convergence}).
However, the dual representation from \BSDE{}s does not directly help to optimise a dynamic risk measure over policies in a model-agnostic manner.
Here, we closely follow the work of \cite{ruszczynski2010risk}
which derives a recursive equation for dynamic risk from general principles.
For a thorough exploration of dynamic risk measures, see e.g. \cite{acciaio2011dynamic}.

Let $\periodspace := \{0, \ldots, \eplength\}$ denote a sequence of periods.
Consider a filtration $\Ff_{0} \subseteq \Ff_{1} \subseteq \ldots \subseteq \Ff_{\eplength} \subseteq \Ff$ on a filtered probability space $(\Omega, \Ff, \{\Ff_{\timeidx}\}_{\timeidx\in\periodspace}, \PP)$\new{, with $\Ff_{0}=\{\emptyset,\Omega\}$  being the trivial $\sigma$-algebra,} and $(\Lpspace_{\timeidx})_{\timeidx \in \periodspace}$ the spaces of bounded $p$-integrable $\Ff_{\timeidx}$-measurable random variables.
Define \new{$\Lpspace_{\timeidx_1,\timeidx_2} := \Lpspace_{\timeidx_1} \times \cdots \times \Lpspace_{\timeidx_2}$ and $\riskmeas_{\timeidx_1,\timeidx_2}: \Lpspace_{\timeidx_1,\timeidx_2} \rightarrow \Lpspace_{\timeidx_1}$}.
\begin{definition}
    \label{def:dynamic-risk-measure}
    A \emph{dynamic risk measure} \cite{ruszczynski2010risk} is a sequence $\{\riskmeas_{\timeidx,\eplength}\}_{\timeidx \in \periodspace}$, where each $\riskmeas_{\timeidx,\eplength}$ satisfies the monotonicity property $\riskmeas_{\timeidx,\eplength}(\rv) \leq \riskmeas_{\timeidx,\eplength} (\rvdum)$ for all $\rv,\rvdum \in \Lpspace_{\timeidx,\eplength}$ such that $\rv \leq \rvdum$.
\end{definition}
Here, \new{inequalities are taken almost surely, and} inequalities between sequences of costs are to be understood component-wise, i.e. $\rv_{\tau} \leq \rvdum_{\tau}$ for all $\timeidx \leq \tau \leq \eplength$.
The mappings $\riskmeas_{\timeidx,\eplength}$, often referred to as \emph{conditional risk measures}, can be interpreted as $\Ff_{\timeidx}$-measurable charges one would be willing to incur at time $\timeidx$ instead of the sequence of future costs.
A crucial property of dynamic risk measures relies on the notion of \emph{time-consistency}.
Indeed, one wishes to evaluate the risk of future outcomes, but it must not lead to inconsistencies at different points in time.
\begin{definition}
	\label{def:time-consistency}
	A dynamic risk measure $\{\riskmeas_{\timeidx,\eplength}\}_{\timeidx \in \periodspace}$ is said to be \emph{time-consistent} \cite{ruszczynski2010risk} iff for any sequence $\rv,\rvdum \in \Lpspace_{\timeidx_1,\eplength}$ and any $\timeidx_1,\timeidx_2 \in \periodspace$ such that $0 \leq \timeidx_1 < \timeidx_2 \leq \eplength$,
	\begin{equation*}
		\rv_{\tau} = \rvdum_{\tau}, \, \forall \tau = \timeidx_1, \ldots, \timeidx_2-1
		\quad \text{and} \quad
		\riskmeas_{\timeidx_2,\eplength}(\rv_{\timeidx_2}, \ldots, \rv_{\eplength}) \leq \riskmeas_{\timeidx_2,\eplength}(\rvdum_{\timeidx_2}, \ldots, \rvdum_{\eplength})
	\end{equation*}
	implies that $\riskmeas_{\timeidx_1,\eplength}(\rv_{\timeidx_1}, \ldots, \rv_{\eplength}) \leq \riskmeas_{\timeidx_1,\eplength}(\rvdum_{\timeidx_1}, \ldots, \rvdum_{\eplength})$.
\end{definition}
\cref{def:time-consistency} may be interpreted as follows: if $\rv$ will be at least as good as $\rvdum$ (in terms of the dynamic risk $\riskmeas_{\timeidx,\eplength}$) at time $\timeidx_2$ and they are identical between $\timeidx_1$ and $\timeidx_2$, then $\rv$ should not be worse than $\rvdum$ (in terms of $\riskmeas_{\timeidx,\eplength}$) at time $\timeidx_1$.
A key result to derive a recursive relationship for time-consistent dynamic risk measures is the following characterisation (see Theorem 1 of \cite{ruszczynski2010risk}).
\begin{theorem}
    \label{thm:time-consistency}
    Let $\{\riskmeas_{\timeidx,\eplength}\}_{\timeidx \in \periodspace}$ be a dynamic risk measure satisfying $\riskmeas_{\timeidx,\eplength}(\rv_{\timeidx}, \rv_{\timeidx+1}, \ldots, \rv_{\eplength}) = \rv_{\timeidx} + \riskmeas_{\timeidx,\eplength}(0, \rv_{\timeidx+1}, \ldots, \rv_{\eplength})$, and $\riskmeas_{\timeidx,\eplength} (0,\ldots,0) = 0$, for any $\rv\in\Lpspace_{\timeidx,\eplength}, \ \timeidx\in\periodspace$.
    Then $\{\riskmeas_{\timeidx,\eplength}\}_{\timeidx \in \periodspace}$ is time-consistent iff for any $0 \leq \timeidx_{1} \leq \timeidx_{2} \leq \eplength$ and $\rv \in \Lpspace_{0,\eplength}$, we have
    \begin{equation*}
    	\riskmeas_{\timeidx_1,\eplength} (\rv_{\timeidx_1}, \ldots, \rv_{\eplength}) = \riskmeas_{\timeidx_1,\timeidx_2} \Big(
    	\rv_{\timeidx_1}, \ldots, \rv_{\timeidx_2 - 1},
    	\riskmeas_{\timeidx_2, \eplength} \big(
    	\rv_{\timeidx_2}, \ldots, \rv_{\eplength}
    	\big) \Big).
    \end{equation*}
\end{theorem}
A time-consistent dynamic risk measure at time $\timeidx_1$ can be seen to include the sequence of costs up to a certain time $\timeidx_2$, and the $\Ff_{\timeidx_2}$-measurable charge for the remaining costs.
As a consequence of \cref{thm:time-consistency}, for any $\timeidx \in \periodspace$, we have the recursive relationship
\begin{equation}
	\riskmeas_{\timeidx,\eplength} (\rv_{\timeidx}, \ldots, \rv_{\eplength}) = \rv_{\timeidx} +
	\riskmeas_{\timeidx} \Bigg( \rv_{\timeidx+1} +
	\riskmeas_{\timeidx+1} \bigg( \rv_{\timeidx+2} +
	\cdots +
	\riskmeas_{\eplength-2} \Big( \rv_{\eplength-1} +
	\riskmeas_{\eplength-1} \big( \rv_{\eplength} \big) \Big) \cdots \bigg) \Bigg), \label{eq:dynamic-risk}
\end{equation}
where the \emph{one-step conditional risk measures} $\riskmeas_{\timeidx}: \Lpspace_{\timeidx+1} \rightarrow \Lpspace_{\timeidx}$ satisfy $\riskmeas_{\timeidx} (\rv) = \riskmeas_{\timeidx, \timeidx+1} (0, \rv)$ for any $\rv\in\Lpspace_{\timeidx+1}$.

Instead of assuming that the one-step conditional risk measures $\riskmeas_{\timeidx}$ are convex (see e.g. \cite{coache2021reinforcement}) or coherent (see e.g. \cite{ruszczynski2010risk,tamar2016sequential}), we impose stronger properties to focus on a narrower class of risk measures, so that we can develop more efficient learning methodologies that do not require nested simulations.
In what follows, we assume that the one-step conditional risk measures $\riskmeas_{\timeidx}$ are static spectral risk measures as in \cref{def:spectral-risk}.
\new{Furthermore, we define the one-step conditional spectral risk measure, with spectrum $\varphi$, for any $\rv\in\Lpspace_{\timeidx+1}$ as $\riskmeas_{\timeidx}^{\spectrum}(\rv | \Ff_{\timeidx})$, which outputs an $\Ff_{\timeidx}$-measurable random variable obtained by conditioning on $\Ff_{\timeidx}$.}

\subsection{Time-Consistency Issue}

\tikzstyle{states}=[shape=circle,draw=mblue,fill=mblue!10,minimum width=1.1cm]
\tikzstyle{costs}=[shape=circle,draw=mred,fill=mred!10,minimum width=1.1cm]		
\begin{wrapfigure}{r}{0.5\textwidth}
    \centering
    \vspace*{-1em}
    \begin{tikzpicture}[scale=0.65,every node/.style={transform shape}]
    \node (temp0) at (0.75,1.875) {};
    \node (temp1) at (4.25,-0.75) {};
    \node[states] (x0) at (0,0) {$\state_{0}$};
    \node[states] (x11) at (3.5,3) {$\state_{1}^{\text{up}}$};
    \node[states] (x12) at (3.5,0.75) {$\state_{1}^{\text{up}'}$};
    \node[states] (x13) at (3.5,-3) {$\state_{1}^{\text{down}}$};
    \node[costs] (x21) at (7,3) {$-2$};
    \node[costs] (x22) at (7,1.5) {$0$};
    \node[costs] (x23) at (7,0) {$-1$};
    \node[costs] (x24) at (7,-1.5) {$2$};
    \node[costs] (x25) at (7,-3) {$4$};
    \draw[-,out=90,in=180] (x0.north) to (temp0.center);
    \draw [->,out=0, in=180] (temp0.center) to node[above=0.2em,pos=0.75] {$0.9$} (x11.west);
    \draw [->,out=0, in=180] (temp0.center) to node[above=0.2em,pos=0.75] {$0.1$} (x12.west);
    \draw [->] (x0.south) to [out=-90, in=180] (x13.west);
    \draw [->] (x11.north) to [out=90, in=180] (x21.west);
    \draw [->] (x11.south) to [out=-90, in=180] (x21.west);
    \draw [->] (x12.north) to [out=90, in=180] (x22.west);
    \draw[-,out=-90,in=180] (x12.south) to (temp1.center);
    \draw [->,out=0, in=180] (temp1.center) to node[above=0.2em,pos=0.75] {$0.9$} (x23.west);
    \draw [->,out=0, in=180] (temp1.center) to node[above=0.2em,pos=0.75] {$0.1$} (x24.west);
    \draw [->] (x13.north) to [out=90, in=180] (x25.west);
    \draw [->] (x13.south) to [out=-90, in=180] (x25.west);
    \end{tikzpicture}
    \caption{Two-period optimisation problem.}
    \label{fig:tikz-time-consistency}
\end{wrapfigure}
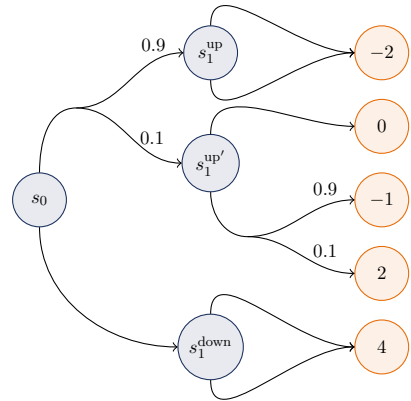
A well-known issue with static risk measures is their time-inconsistency, i.e. an optimal behaviour planned for a future state of the environment when optimising a static risk measure may not be optimal anymore once the agent visits the state.
We provide a motivational example to illustrate the consequences of using static risk measures in sequential decision making problems.

Suppose we face the two-period optimisation problem shown in \cref{fig:tikz-time-consistency} with transition probabilities indicated on each branch, where we start at $\state_{0}$, choose between moving up or down at each period, and obtain a terminal cost $\rv$.
If we aim to minimise the static \CVaR{} with threshold $\alpha=0.9$ of the terminal cost, the optimal actions are to move up and then down, which lead to
\begin{align*}
    \CVaR_{0.9} (\rv) &= \frac{1}{1-0.9} \int_{0.9}^{1} \VaR_{u}(\rv) \, \dee u \\
    &= \frac{1}{0.1} \Big( 2 \times 0.01 - 1 \times 0.09 \Big) = -0.7,
\end{align*}
while the second contender is to move up twice with a \CVaR{} of $0$.
The time-consistency issue arises if we land at $\state_{1}^{\text{up}'}$, because the \CVaR{} with threshold $\alpha=0.9$ of moving up (resp. down) is now $0$ (resp. $2$).
In this scenario, the optimal action at $\state_{1}^{\text{up}'}$ is to move up, which contradicts the initial optimal strategy starting from $\state_{0}$.
If we instead aim to minimise the dynamic \CVaR{} (i.e. one-step conditional risk measures are static \CVaR{}s at threshold $\alpha$), it is clear that the optimal actions are to move up twice; in fact, these are the optimal actions for any $\alpha \in [0.7,1)$.

To overcome this issue with static risk measures, one must solve optimisation problems for every possible state of the environment, which is highly ineffective when dealing with sequential decision making problems of several periods with a large number of states. 
This motivates the need to consider dynamic risk measures.


\section{Reinforcement Learning}
\label{sec:reinforcement-learning}

In this section, we introduce the \RL{} problems we study, and give the motivation behind our work.
We describe each problem as an \emph{agent} who tries to learn an optimal behavior, or \emph{agent's policy}, by interacting with a certain \emph{environment} in a model-agnostic manner.

Let $\statespace$ and $\actionspace$ be arbitrary state and action spaces respectively, and let $\costspace \subset \Reals$ be a cost space.
The environment is often represented as a \emph{Markov decision process} (\MDP{}) with the tuple $(\statespace, \actionspace, \costfunc, \PP)$, where $\costfunc(\state, \action, \statedum) \in \costspace$ is a cost function and $\PP$ characterises the transition probabilities $\PP(\state_{\timeidx+1} = \statedum \suchthat \state_{\timeidx} = \state, \action_{\timeidx} = \action)$.
The transition probability is assumed stationary, although time may be a component of the state.
\new{In what follows,} we assume that time is embedded in the state space\new{, which can be interpreted as an augmented state space where the first dimension corresponds to the time.}
An \emph{episode} consists of a sequence of \new{$\eplength+1$} \emph{periods}, denoted \new{$\periodspace := \{0, \ldots, \eplength\}$}, where $\eplength \in \Nats$ is known and finite.
At each period \new{$\timeidx\in\periodspace$}, the agent begins in a state $\state_{\timeidx} \in \statespace$ and takes an action $\action_{\timeidx} \in \actionspace$ according to a randomised \new{time-dependent} \emph{policy} $\policy^{\policyparams}: \statespace \rightarrow \Pp(\actionspace)$ parameterised by some $\policyparams$.
Here, $\Pp(\actionspace)$ represents the space of probability measures on a certain $\sigma$-algebra generated by the action space $\actionspace$, and therefore the agent selects the action $\action$ with probability $\policy^{\policyparams}(\action | \state_{\timeidx} = \state)$ \new{when being in state $\state$ at time $\timeidx$}.
The agent then moves to the next state $\state_{\timeidx+1} \in \statespace$, and receives a cost $\costfunc_{\timeidx} = \costfunc(\state_{\timeidx}, \action_{\timeidx}, \state_{\timeidx+1}) < \infty$.
We view the cost function as a deterministic mapping of the states and actions, but we can easily generalise to include other sources of randomness.

We consider time-consistent, Markov, dynamic risk measures $\{\riskmeas_{\timeidx,\eplength}\}_{\timeidx \in \periodspace}$ where we assume the one-step conditional risk measures $\riskmeas_{\timeidx}$ are \new{one-step conditional} spectral risk measures.
Using \cref{eq:dynamic-risk}, we aim to solve \new{$(\eplength+1)$}-period risk-sensitive \RL{} problems of the form
\begin{equation}
	\min_{\policyparams} \, \riskmeas_{0,\eplength} \Big( \{ \costfunc^{\policyparams}_{\timeidx} \}_{\timeidx \in \periodspace} \Big) = \min_{\policyparams} \, 
	\riskmeas_{0} \Bigg(\costfunc^{\policyparams}_{0} +
	\riskmeas_{1} \bigg(\costfunc^{\policyparams}_{1} +
	\cdots +
	\newmath{\riskmeas_{\eplength-1} \Big(\costfunc^{\policyparams}_{\eplength-1} +
	\riskmeas_{\eplength} \big(\costfunc^{\policyparams}_{\eplength}
	\big) \Big)} \cdots \bigg) \Bigg), \label{eq:optim-problem1} \tag{P}
\end{equation}
where $\costfunc^{\policyparams}_{\timeidx} = \costfunc(\newmath{\state_{\timeidx}^{\policyparams}}, \action^{\policyparams}_{\timeidx}, \state_{\timeidx+1}^{\policyparams})$ is a bounded $\Ff_{\timeidx+1}$-measurable random cost modulated by the policy $\policy^{\policyparams}$ -- that is why we include a $\policyparams$ index. \new{Note that the policy dependence on $\state_{\timeidx}^{\policyparams}$ holds for all periods except the initial state at $\timeidx=0$.} One might incorporate a deterministic discount factor in \cref{eq:optim-problem1}, which can be moved in and out of the risk measures $\riskmeas_{\timeidx}$ due to the positive homogeneity property.
It differs from: (i) standard risk-neutral \RL{}, which typically deals with an expectation of the cumulative (discounted) costs as the objective function; and (ii) risk-aware \RL{} with static risk measures, which provides optimal precommitment strategies.

Next, we derive dynamic programming (\DP{}) equations for the class of problems of the form \cref{eq:optim-problem1}.
We define the \emph{value function} as the running risk-to-go
\begin{equation}
    \valuefunc_{\timeidx}(\state;\policyparams) :=
    \riskmeas_{\timeidx} \bigg(\costfunc_{\timeidx}^{\policyparams} +
    \riskmeas_{\timeidx+1} \Big(\costfunc_{\timeidx+1}^{\policyparams} +
    \dots +
    \newmath{\riskmeas_{\eplength} \big(\costfunc_{\eplength}^{\policyparams} \big)} 
    \Big) \biggm| \state_{\timeidx}=\state \bigg),
    \label{eq:value-func}
\end{equation}
for all $\state \in \statespace$ and $\timeidx \in \periodspace$.
The value function outputs the dynamic spectral risk at a certain time $\timeidx$ when being in a specific state $\state$ and behaving according to the policy $\policy^{\policyparams}$; it represents the performance criterion that the agent seeks to minimise.
Using \cref{eq:value-func}, the \DP{} equations for a specific policy $\policy^{\policyparams}$ are
\begin{subequations}
\begin{align}
	\newmath{\valuefunc_{\eplength}(\state;\policyparams)} &=
    \newmath{\riskmeas_{\eplength} \Big(\costfunc_{\eplength}^{\policyparams} \Bigm|\state_{\eplength}=\state \Big)},
    \quad \text{and}
    \label{eq:value-func0-1} 
    \\
	\valuefunc_{\timeidx}(\state;\policyparams) &=
    \riskmeas_{\timeidx} \Big(\costfunc_{\timeidx}^{\policyparams} +
    \valuefunc_{\timeidx+1}(\state_{\timeidx+1}^{\policyparams};\policyparams)
    \Bigm|\state_{\timeidx}=\state \Big). \label{eq:value-func0-2}
\end{align}
\end{subequations}
The \DP{} equations for the value function allow the agent to evaluate the risk of a fixed policy $\policy^{\policyparams}$ recursively.

One common approach for solving \RL{} problems of the form \cref{eq:optim-problem1} is to optimise the value function over policies.
It then requires an accurate estimation of the value function in \cref{eq:value-func0-1,eq:value-func0-2} at every period and state of the environment, which is not straightforward at first sight.
A proposed methodology \new{in the class of actor-critic algorithms} consists of approximating the value function at a certain state by simulating several transitions, which we refer to as \emph{nested approach} (see e.g. \cite{coache2021reinforcement}).
In that work, the authors develop a \RL{} algorithm that estimates the value function by generating additional (inner) transitions for every visited state of an (outer) episode -- e.g. given a random variable, simulating many realisations and computing the average of the worst realisations gives a naive estimator of its tail expectation.
However, those nested simulations are costly, and the acquisition of new observations may not be possible when working with real datasets.
We instead wish to find an alternative methodology \new{that alleviates this curse of dimensionality, i.e.} that is \emph{improved memory-wise}, \emph{comparable in computational performance}, and \new{more importantly} \emph{using only full episodes}.

At this point, our research question is twofold:
(i) how to devise an efficient approach to estimate the value function, which is effectively a dynamic risk measure; and
(ii) how to adapt a \RL{} algorithm that avoids simulating transitions for every visited state to solve the class of \RL{} problems described above.
We use the notion of \emph{conditional elicitability} as a tool to answer the former question, and take advantage of the \emph{gradient formulae for spectral risk measures with a finite support spectrum} to solve the latter, which leads us to a computationally efficient risk-aware \RL{} algorithm.


\section{Elicitability and Consistent Scoring Functions}
\label{sec:elicitability}

In this section, we introduce the necessary theoretical background for point estimation using scoring functions, following the work from \cite{gneiting2011making}.
We keep the results as general as possible, even though we often consider $\FF$ as the set of cumulative distribution functions (\CDF{}s), and $\estimsupp, \rvsupp$ as (subset of) the reals\new{, i.e. $\estimsupp \subseteq \Reals^{k}$ and $\rvsupp \subseteq \Reals^{d}$.}

Suppose we have a $d$-dimensional \emph{random variable} $\rv$ with support on $\rvsupp$ and a \CDF{} $F := F_{\rv} \in \FF$.
We are interested in finding a $k$-dimensional \emph{point approximation} $\estim \in \estimsupp$ of a certain \emph{statistical mapping} of the random variable $\statmap(\rv)$, where $\statmap:\rvsupp \rightarrow \estimsupp$.
The main goal is to find a \emph{scoring function} $\score:\estimsupp \times \rvsupp \rightarrow \Reals$ such that when observing a realisation $y \in \rvsupp$, our current point forecast $\estim \in \estimsupp$ is penalized by $\score(\estim, y)$.
The score function can be interpreted as a loss function.

To draw parallels between this notation and our objective, recall that we aim to provide an estimation (i.e. point approximation $\estim$) of the value function $\valuefunc_{\timeidx}(\state; \policyparams)$ at every state $\state \in \statespace$ and time $\timeidx \in \periodspace$ (i.e. statistical mapping $\statmap(\rv)$) for sequences of costs $\{ \costfunc^{\policyparams}_{\timeidx} \}_{\timeidx \in \periodspace}$ induced by the agent's policy $\policy^{\policyparams}$ (i.e. random variable $\rv$).
We wish to update our estimate of the value function using a single realisation of costs from an episode and an appropriate loss function (i.e. scoring function $\score$).

Next, we define a key property of scoring functions that makes them appealing for estimating mappings $\statmap$.

\begin{definition}
    \label{def:F-consistency}
    A scoring function $\score:\estimsupp \times \rvsupp \rightarrow \Reals$ is said to be \emph{$\FF$-consistent for $\statmap$}
    iff for any $F \in \FF$ and $\estim \in \estimsupp$, we have
    \begin{equation}
        \EE_{\rv \sim F} \Big[ \score \big( \statmap(\rv), \rv \big) \Big]
        \leq
        \EE_{\rv \sim F} \Big[ \score \big( \estim, \rv \big) \Big],
        \label{eq:consistent-scoring}
    \end{equation}
    where $\EE_{\rv \sim F} [Y] = \int y \, \dee F(y)$.
    Furthermore, the scoring function $\score$ is \emph{strictly} $\FF$-consistent for $\statmap$ if the equality in \cref{eq:consistent-scoring} implies that $\estim = \statmap(\rv)$.
\end{definition}
Here, the choice of a scoring function $\score$ is closely related to the choice of a mapping $\statmap$, as demonstrated by \cite{gneiting2011making}.
For instance, one can show that if the mapping of interest is the mean, then $\score(\estim,y) = (\estim-y)^2$ is (resp. strictly) consistent for the class of \CDF{}s with finite first moments  (resp. second moments).
\cref{def:elicitable} describes all mappings $\statmap$ admitting strictly consistent scoring functions.
\begin{definition}
    \label{def:elicitable}
    A $k$-dimensional mapping $\statmap:\rvsupp \rightarrow \estimsupp$ is said to be \emph{$k$-elicitable} wrt $\FF$
    iff there exists a scoring function $\score$ that is strictly $\FF$-consistent for $\statmap$.
\end{definition}
Therefore, a mapping $\statmap$ is elicitable iff there exists a scoring function such that its estimate is the unique minimiser of the expected score, i.e.
\begin{equation*}
    \statmap(\rv) = \argmin_{\estim \in \estimsupp} \; \EE_{\rv \sim F} \Big[ \score(\estim,\rv) \Big].
\end{equation*}
Much research has been done to determine which mappings are elicitable.
Under certain regularity assumptions, all moments, quantiles, and expectiles are $1$-elicitable (see e.g. \cite{savage1971elicitation,thomson1979eliciting,saerens2000building}).
For instance, a strictly consistent scoring function $\score$ for the mean is necessarily of the form
\begin{equation*}
	\score(\estim,y) = G(y) - G(\estim) + G'(\estim)(\estim-y),
\end{equation*}
where $G:\rvsupp \rightarrow \Reals$ is strictly convex with subgradient $G'$; taking $G(x) = x^2$ leads to the squared error.
For the $\alpha$-quantile, and thus the value-at-risk $\VaR_{\alpha}$ as in \cref{def:value-at-risk}, any strictly consistent $\score$ must be of the form 
\begin{equation*}
	\score(\estim,y) = \Big( 
 {\color{blue}\Ind_{y \leq \estim}} 
 - \alpha \Big) \Big( G(\estim) - G(y) \Big),
\end{equation*}
where $G:\rvsupp \rightarrow \Reals$ is a nondecreasing function.
On the other hand, several researchers proved the non $1$-elicitability of well-known functionals, such as the variance, expected-shortfall, also known as conditional value-at-risk (\CVaR{}), and any spectral risk measure (see e.g. \cite{gneiting2011making}).

\subsection{Conditional Elicitability}

The value function we wish to estimate is a (dynamic) spectral risk measure, which is not $1$-elicitable \cite{ziegel2016coherence}.
However, \cite{emmer2015best} proved that despite not being elicitable on its own, the \CVaR{} is elicitable conditionally on the \VaR{}.
It originates from the work of \cite{osband1985providing} and the concept of \emph{higher order elicitability} or \emph{conditional elicitability}, where components of a $k$-elicitable vector-valued mapping can fail to be $1$-elicitable.
As a common example, the variance is not $1$-elicitable, but the vector composed of the mean and variance is $2$-elicitable using the so-called revelation principle.
It relies heavily on the fact that the variance can be connected to a mapping with two $1$-elicitable components, i.e. the first two moments, via a bijection.

\cite{fissler2016higher} characterises (with necessary and sufficient conditions) the class of strictly consistent scoring functions for several statistical mappings.
One of their main results (see Theorem 5.2 of \cite{fissler2016higher}) states that spectral risk measures having a spectrum $\spectrum$ with finite support can be a component of a $k$-elicitable mapping, although a spectral risk measure is not elicitable on its own.
We thus reformulate their result and adjust the notation for dynamic risk measures.
First, we add the conditioning on the $\sigma$-algebra $\Ff_{\timeidx}$, since we are working in a dynamic framework.
We also consider random variables that represent costs, and hence \CVaR{}s that are upper tail expectations.
\begin{theorem}[Conditional elicitability of spectral risk measures]
    \label{thm:elicitable-spectral}
    Let \CDF{}s of $\rv \, |_{\Ff_{\timeidx}}$ for $\rv \in \Lpspace_{\timeidx+1}$, denoted by $\FF$, have finite first moments and be supported on $\rvsupp \subseteq \Reals$.
    Let $\riskmeas^{\spectrum}$ be a spectral risk measure, where $\spectrum$ is given by
    \begin{equation*}
        \spectrum = \sum_{m=1}^{k-1} p_m \delta_{\alpha_m},
    \end{equation*}
    with $p_{m} \in (0,1]$, $\sum_{m=1}^{k-1} p_m = 1$, and $0 < \alpha_{1} < \alpha_{2} < \cdots < \alpha_{k-1} < 1$.
    Define the mapping
    \begin{equation*}
        \statmap(\rv) = \Big(
        \VaR_{\alpha_1}(\rv), \ldots, \VaR_{\alpha_{k-1}}(\rv), \,
        \riskmeas^{\spectrum} (\rv) \Big).
    \end{equation*}
    If the \CDF{}s in $\FF$ have unique $\alpha_m$-quantiles, then the mapping $\statmap$ is $k$-elicitable wrt $\FF$.
    Furthermore, define the set of estimates $\estimsupp = \{\estim \in \rvsupp^{k} \suchthat \estim_1 \leq \ldots \leq \estim_{k-1}\}$.
    Then a scoring function $\score: \estimsupp \times \rvsupp \rightarrow \Reals$ of the form
    \begin{equation}
    \begin{split}
        & \score(\estim_1, \ldots, \estim_{k-1}, \estim_k, y) \\
        &\quad = \Bigg[ \sum_{m=1}^{k-1} \Big( \Ind_{y \leq \estim_m} - \alpha_m \Big) \Big( G_m(\estim_m) - G_m(y) \Big) \Bigg] - G_k(\estim_k) + G_k(y) \\
        &\qquad\quad + G_k'(\estim_k) \Bigg[ \estim_k + \sum_{m=1}^{k-1} \frac{1}{1-\alpha_m} \bigg( \estim_m \Big( \Ind_{y > \estim_m} - (1-\alpha_m) \Big) - y \, \Ind_{y > \estim_m} \bigg) \Bigg] \label{eq:scoring-dynamic-spectral}
    \end{split}
    \end{equation}
    is (resp. strictly) $\FF$-consistent for $\statmap$ if (i) $G_k: \rvsupp \rightarrow \Reals$ is (resp. strictly) convex with subgradient $G_k'$ such that for a specific choice of mappings $G_1,\ldots,G_{k-1}: \rvsupp \rightarrow \Reals$, the functions
	\begin{equation}
	    x \mapsto G_{m}(x) - \frac{x p_m}{1-\alpha_m} G_{k}'(\estim_k), \quad m = 1, \ldots, k-1 \label{eq:scoring-dynamic-spectral-condition}
	\end{equation}
	are (resp. strictly) increasing for any $\estim_k \in \rvsupp$; and (ii) $\EE_{Y \sim F}[|G_m(\rv)|] < \infty$, $\forall \, m=1, \ldots, k$.
\end{theorem}
This result states that, under some regularity assumptions, for any spectral risk measure with a finite support spectrum of a random variable conditioned on a $\sigma$-algebra, there exists a strictly consistent scoring function for the mapping $\statmap$ composed of the spectral risk measure $\riskmeas^{\spectrum}$ and all $\alpha_{m}$-quantiles with nonzero weights in its spectrum.
One particularly interesting spectral risk measure is the \CVaR{}, where the spectrum $\spectrum$ has a weight on a single quantile.
\cref{thm:elicitable-cvar} (see Corollary 5.5 of \cite{fissler2016higher}) which follows from \cref{thm:elicitable-spectral}, gives the characterisation of (strictly) consistent scoring functions for the \CVaR{}.
\begin{corollary}[Conditional elicitability of the \CVaR{}]
    \label{thm:elicitable-cvar}
    Let \CDF{}s of $\rv \, |_{\Ff_{\timeidx}}$ for $\rv \in \Lpspace_{\timeidx+1}$, denoted by $\FF$, have finite first moments, and be supported on $\rvsupp \subseteq \Reals$.
    Define the mapping
    \begin{equation*}
        \statmap(\rv) = \Big(
        \VaR_{\alpha}(\rv), \,
        \CVaR_{\alpha}(\rv) \Big).
    \end{equation*}
    If the \CDF{}s in $\FF$ have unique $\alpha$-quantiles, then the mapping $\statmap$ is $2$-elicitable wrt $\FF$.
    Furthermore, define the set of estimates $\estimsupp = \{\estim \in \rvsupp^{2} \suchthat \estim_1 \leq \estim_2 \}$.
    Then a scoring function $\score: \estimsupp \times \rvsupp \rightarrow \Reals$ of the form
    \begin{equation}
    \begin{split}
        \score(\estim_1, \estim_2, y) &= 
        \Big( \Ind_{y \leq \estim_1} - \alpha \Big) \Big( G_1(\estim_1) - G_1(y) \Big) - G_2(\estim_2) + G_2(y) \\
        &\qquad + G_2'(\estim_2) \Bigg[ \estim_2 + \frac{1}{1-\alpha} \bigg( \estim_1 \Big( \Ind_{y > \estim_1} - (1-\alpha) \Big) - y \, \Ind_{y > \estim_1} \bigg) \Bigg] 
    \end{split} \label{eq:scoring-dynamic-cvar}
    \end{equation}
    is (resp. strictly) $\FF$-consistent for $\statmap$ if (i) $G_2: \rvsupp \rightarrow \Reals$ is (resp. strictly) convex with subgradient $G_2'$ such that for a specific choice of mapping $G_1: \rvsupp \rightarrow \Reals$, the function
	\begin{equation}
	    x \mapsto G_{1}(x) - \frac{x}{1-\alpha} G_{2}'(\estim_2)
	    \label{eq:scoring-dynamic-cvar-condition}
	\end{equation}
	is (resp. strictly) increasing for any $\estim_2 \in \rvsupp$; and (ii) $\EE_{Y \sim F}[|G_1(\rv)|], \, \EE_{Y \sim F}[|G_2(\rv)|] < \infty$.
\end{corollary}
We restrict the set of estimates to $\{\estim \in \rvsupp^{2} \suchthat \estim_1 \leq \estim_2 \}$ because $\CVaR_{\alpha}$ must be greater than $\VaR_{\alpha}$. 
\cref{thm:elicitable-cvar} is a key result as it states that one can find the \CVaR{} using the \VaR{} as an intermediate step in the optimisation problem
\begin{equation*}
    \Big( \VaR_{\alpha}(\rv), \CVaR_{\alpha}(\rv) \Big) = \argmin_{(\estim_1,\estim_2) \in \estimsupp} \; \EE_{Y \sim F} \Big[ \score(\estim_1, \estim_2, Y) \Big],
\end{equation*}
with $\score$ given in \cref{eq:scoring-dynamic-cvar}.

We emphasise here that there exist infinitely many characterisations for the scoring function $\score$, and the choice of $G_1,G_2$ strongly depends on the support $\rvsupp$.
For instance, assume that $\rv$ are bounded costs, i.e. $|\rv| < C$, for some constant $C > 0$.
Then if we use $G_1(x) = C$ and $G_2(x) = -\log(x+C)$, which leads to $G_2'(x) = \frac{-1}{x+C}$ in the strictly consistent scoring function in \cref{eq:scoring-dynamic-cvar}, we obtain that
\begin{align}
    \begin{split}
    \score(\estim_1, \estim_2, y) &= 
            \Big( \Ind_{y \leq \estim_1} - \alpha \Big) \Big( C - C \Big) + \log(\estim_2+C) - \log(y+C) \\
            &\qquad\quad - \frac{1}{\estim_2+C} \Bigg[ \estim_2 + \frac{1}{1-\alpha} \bigg( \estim_1 \Big( \Ind_{y > \estim_1} - (1-\alpha) \Big) - y \, \Ind_{y > \estim_1} \bigg) \Bigg]
    \end{split} \nonumber\\
    &= \log \left( \frac{\estim_2+C}{y+C} \right) - \frac{\estim_2}{\estim_2+C} + \frac{ \estim_1 \left( \Ind_{y \leq \estim_1} - \alpha \right) + y \, \Ind_{y > \estim_1} }{(\estim_2+C) (1-\alpha)}. \label{eq:score-experiments}
\end{align}
This choice respects the condition in \cref{eq:scoring-dynamic-cvar-condition}, as $x \mapsto C + \frac{x}{(\estim_2+C) \alpha}$ is strictly increasing wrt $x$.
This specific characterisation is a canonical choice in the literature, as the first term of \cref{eq:scoring-dynamic-cvar} vanishes.

If one is interested in nonnegative random variables, a more appropriate choice is $G_2(x) = -\log(x)$, which leads to
\begin{equation}
    S(\estim_1, \estim_2, y) = \log \left( \frac{\estim_2}{y} \right) - 1 + \frac{ \estim_1 \left( \Ind_{y \leq \estim_1} - \alpha \right) + y \, \Ind_{y > \estim_1} }{\estim_2 (1-\alpha)}.
\end{equation}

Although we do not focus here on the choice of these functions, it would be interesting from a theoretical perspective to understand if characterisations of scoring functions affect optimisation performances.
In our experiments, we use scoring functions of the form \cref{eq:score-experiments} as we work with bounded random costs.


\section{Deep Composite Model Regression}
\label{sec:deep-regression}

Assume the framework outlined in \cref{sec:elicitability}, but now the random variable $\rv \, |_{\Ff_{\timeidx}}$ is explained by $q$ \emph{covariates} or \emph{features} $x \in \Reals^{q}$.
In our \RL{} problem, recall that the sequences of costs $\{ \costfunc^{\policyparams}_{\timeidx} \}_{\timeidx \in \periodspace}$ induced by the agent's policy $\policy^{\policyparams}$ (i.e. random variable $\rv$) are explained by the states $\state \in \statespace$ and times $\timeidx \in \periodspace$ (i.e. observable features $x$).
We then take an interest in $k$-elicitable mappings of the \emph{conditional} \CDF{} $F:= F_{\rv \mid X=x} \in \FF$, such as $\statmap(\rv | X=x) = \CVaR_{\alpha}(\rv | X=x)$.
The objective is to use strictly consistent scoring functions to approximate these mappings.
In this section, we work with a $1$-elicitable mapping $\statmap$ and a corresponding strictly consistent scoring function $\score$ for readability purposes.
However, since the setting is valid for any strictly consistent scoring function, we can replace $\statmap$ with any $k$-elicitable mapping.

We wish to find a function of the features $h : \Reals^{q} \rightarrow \estimsupp$, as opposed to a point approximate $\estim \in \estimsupp$, that minimises the expectation of the strictly consistent scoring function, i.e.
\begin{equation}
    \statmap(\rv | X=x) = \argmin_{h \, : \, \Reals^{q} \rightarrow \estimsupp} \; \EE_{\rv \sim F} \Big[ \score(h(x),\rv) \Big].
\end{equation}
As the space of all functions is infinite dimensional, the usual approach to obtaining approximations for $M(Y|X=x)$ consists of assuming a certain parametric form for the function $h$, and optimising over its parameters.
It is closely related to the notion of quantile and composite model regression.

Quantile regression and composite models have shown a lot of potential in the machine learning (\ML{}) community (see e.g. \cite{meinshausen2006quantile,takeuchi2006nonparametric,zhang2019extending,rodrigues2020beyond,guillen2021joint}).
These approaches provide more desirable alternatives to the previous methods, which utilise the expectation-maximisation (\EM{}) algorithm in the modelling procedure and require specific distribution assumptions.

The original quantile regression approach in \cite{koenker1978regression} assumes there exists a prespecified link function $h:\Reals \rightarrow \Reals$ such that
\begin{equation}
    \statmap(\rv | X=x) = h^{-1} \Big( \beta_0 + \sum_{j=1}^{q} \beta_j x_j \Big).
    \label{eq:coef-quantreg}
\end{equation}
After determining the appropriate scoring function, one finds the optimal coefficients of \cref{eq:coef-quantreg} by solving for
\begin{equation}
    \beta^{*} = \argmin_{\beta \, \in \, \Reals^{q+1}} \; \EE_{Y \sim F} \Bigg[ \score \bigg(h^{-1} \Big( \beta_0 + \sum_{j=1}^{q} \beta_j X_j \Big), \rv \bigg) \Bigg].
    \label{eq:opt-coef-quantreg}
\end{equation}
As the true distribution is often unknown, the expectation in \cref{eq:opt-coef-quantreg} is replaced by the empirical mean based on observed data $(x^{(i)}, y^{(i)})$, $i=1,\ldots,n$, to obtain the estimator
\begin{equation}
    \hat{\beta} = \argmin_{\beta \, \in \, \Reals^{q+1}} \; \frac{1}{n} \sum_{i=1}^{n} \Bigg[ \score \bigg(h^{-1} \Big( \beta_0 + \sum_{j=1}^{q} \beta_j x_j^{(i)} \Big), y^{(i)} \bigg) \Bigg].
    \label{eq:est-coef-quantreg}
\end{equation}
This approach \new{with a strictly consistent scoring function} gives an M-estimator of $\statmap(\rv | X=x)$ that converges in probability to the true value under certain regularity assumptions (see \new{Theorem 1 of \cite{dimitriadis2022characterizing}} for relations between consistency and conditional model-consistency of scoring functions).

Instead of the generalised linear model assumption in \cref{eq:est-coef-quantreg}, which is quite restrictive, \cite{fissler2021deep} generalises this idea and uses a more powerful function approximator to perform deep quantile regression: neural network structures.
The authors model the mapping $\statmap$ as a linear combination of outputs from a feed-forward \ANN{}, denoted $H^{\psi}:\Reals^{q} \rightarrow \Reals^{q}$, and update the weights and biases of the \ANN{} as well as the coefficients $\beta$ to optimise the expected score. They solve
\begin{equation}
    (\hat{\beta},\hat{\psi}) = \argmin_{\beta, \psi} \; \frac{1}{n} \sum_{i=1}^{n} \Bigg[ \score \bigg(h^{-1} \Big( \beta_0 + \sum_{j=1}^{q} \beta_j H^{\psi}(x^{(i)})_{j} \Big), y^{(i)} \bigg) \Bigg],
    \label{eq:est-coef-compositereg}
\end{equation}
where they use a linear combination of the outputs of $H$, and then perform a transformation using the link function $h$.
This framework is implemented for deep composite model regression where the mapping $\statmap$ is the tuple consisting of a $\alpha$-quantile, its lower and upper tail expectations.
They illustrate the applicability of their approach on a real dataset, and show better performance than that of classical approaches.
Their works extend the paper from \cite{richman2021mind}, in which joint estimation of quantiles and the expectation is performed using neural network structures.

\cite{fissler2021deep} do not, however, study how their framework can be applied dynamically.
Also, although this deep modelling is quite powerful, their approach can be improved.
Indeed, using a linear combination with outputs of an \ANN{} in \cref{eq:est-coef-compositereg} is unnecessary and adds more parameters in the optimisation process.
\ANN{}s are known to be universal approximators, and a sufficiently large neural network structure can adequately approximate any linear or nonlinear relationship.
One could instead directly model the statistical mapping $\statmap$ with an \ANN{}, e.g. $H^{\psi}:\Reals^{q} \rightarrow \Reals$, and only optimise its weights and biases:
\begin{equation}
    \hat{\psi} = \argmin_{\psi} \; \frac{1}{n} \sum_{i=1}^{n} \bigg[ \score \Big( H^{\psi}(x^{(i)}), y^{(i)} \Big) \bigg]. \label{eq:est-coef-deepreg}
\end{equation}
The approach in \cref{eq:est-coef-deepreg} is applied by \cite{fissler2022sensitivity} to estimate (static) sensitivity measures.
We make use of this form in our subsequent analysis.


\section{Algorithm}
\label{sec:algo-framework}

We now provide some implementation details on our proposed actor-critic style \cite{konda2000actor} algorithm, shown in \cref{algo:actor-critic}.\footnote{Our Python code is publicly available in the \href{https://github.com/acoache/RL-ElicitableDynamicRisk}{Github repository RL-ElicitableDynamicRisk}.}
More specifically, we highlight the main steps of the two procedures: the \emph{critic} estimating the value function with a \emph{deep composite modelling procedure} in \cref{ssec:estimate-V}, and the \emph{actor} updating the policy via a \emph{policy gradient method} in \cref{ssec:update-policy}.
We interpret it as follows: the actor decides which actions to perform, while the critic evaluates the performance of these actions and how they should be adjusted to obtain better results.
Note here that the method is thoroughly described for dynamic \CVaR{}, but one may be interested in generalising the framework to arbitrary $k$-elicitable dynamic spectral risk measures -- we do so in \cref{sec:generalization-spectral}.

\begin{algorithm}
\caption{Risk-aware actor-critic algorithm for dynamic $\CVaR_{\alpha}$}
\label{algo:actor-critic}
{\footnotesize
\begin{algorithm2e}[H]
	\KwIn{\ANN{}s $\policy^{\policyparams},H_{1,\timeidx}^{\psi_1},H_{2,\timeidx}^{\psi_2}$, numbers of epochs $K,K_1,K_2$, mini-batch sizes $\Nbatchs_{1},\Nbatchs_{2}$}
	Initialise the environment and risk measure instances\;
	Set initial learning rates $\lrate^{\psi_1},\lrate^{\psi_2},\lrate^{\policyparams}$\;
	\For{each iteration $k = 1, \ldots, K$}{
	    \For(\tcp*[f]{Critic, \cref{ssec:estimate-V}}){each epoch $k_1 = 1, \ldots, K_1$}{
	        Zero out the gradients of $H_{1,\timeidx}^{\psi_1},H_{2,\timeidx}^{\psi_2}$\;
	        Simulate a mini-batch of $\Nbatchs_{1}$ episodes induced by $\policy^{\policyparams}$\;
	        Compute the loss $\Ll^{\valueparams}$ in \cref{eq:loss-critic}\;
	        Update $\valueparams=\{\psi_1,\psi_2\}$ by performing an Adam optimisation step\;
	        Tune the learning rates $\lrate^{\psi_1},\lrate^{\psi_2}$ with a scheduler\;
	        \uIf{$k_1 \ \mathrm{mod} \ K^{*} = 0$}{
				Update the target networks $\tilde{\valueparams}=\{\tilde{\psi}_1,\tilde{\psi}_2\}$ \;
			}
	    }
	    \For(\tcp*[f]{Actor, \cref{ssec:update-policy}}){each epoch $k_2 = 1, \ldots, K_2$}{
	        Zero out the gradient of $\policy^{\policyparams}$\;
	        Simulate a mini-batch of $\lceil \Nbatchs_{2}/(1-\alpha) \rceil$ episodes induced by $\policy^{\policyparams}$\;
	        Compute the loss $\Ll^{\policyparams}$ in \cref{eq:loss-func-gradient-reinforce}\;
	        Update $\policyparams$ by performing an Adam optimisation step\;
	        Tune the learning rate $\lrate^{\policyparams}$ with a scheduler\;
	    }
	}
	\KwOut{Optimal policy $\policy^{\policyparams}$ and its value function $\valuefunc_{\timeidx}^{\valueparams} = H_{1,\timeidx}^{\psi_1}+H_{2,\timeidx}^{\psi_2}$}
\end{algorithm2e}
}
\end{algorithm}

We consider problems of the form \cref{eq:optim-problem1}, and assume the one-step conditional risk measures are \new{one-step conditional} \CVaR{}s at threshold $\alpha$.
We \new{also recall} that time is embedded in the state space\new{, which allows time-dependent policies}.

In what follows, suppose we only simulate $\Nbatchs$ full episodes composed of $\eplength$ transitions:
\begin{equation}
	\Big( \state_{\timeidx}^{(\batchidx)}, \action_{\timeidx}^{(\batchidx)}, \state_{\timeidx+1}^{(\batchidx)}, \costfunc_{\timeidx}^{(\batchidx)} \Big), \quad \timeidx \in \periodspace, \ \batchidx \in \Nbatchs,
	\label{eq:transitions}
\end{equation}
where the transition probabilities are characterised by $\PP^{\policyparams}(\action,\statedum | \state_{\timeidx}) := \PP(\statedum | \state_{\timeidx}, \action) \policy^{\policyparams}(\action | \state_{\timeidx})$.
The value function when being in state $\state$ at the period $\timeidx$ and following the policy $\policy^{\policyparams}$ in \cref{eq:value-func0-1,eq:value-func0-2} are, respectively,
\begin{subequations}
\begin{align}
	\newmath{\valuefunc_{\eplength}(\state;\policyparams)} &=
    \newmath{\CVaR_{\alpha} \Big(\costfunc_{\eplength}^{\policyparams} \Bigm|\state_{\eplength} = \state \Big)}, \quad \text{and}
    \label{eq:lastV} 
    \\
	\valuefunc_{\timeidx}(\state;\policyparams) &=
    \CVaR_{\alpha} \Big(\costfunc_{\timeidx}^{\policyparams} +
    \newmath{\valuefunc_{\timeidx+1}}(\state_{\timeidx+1}^{\policyparams};\policyparams)
    \Bigm|\state_{\timeidx} = \state \Big). \label{eq:otherV}
\end{align}
\end{subequations}

Define the following (fully-connected, multi-layered feed forward) \ANN{} structures:
(i) a policy $\policy^{\policyparams}: \statespace \rightarrow \Pp(\actionspace)$, which outputs a distribution over the space of actions $\actionspace$ when in a certain state;
(ii) $H_{1,\timeidx}^{\psi_1}(\cdot; \policyparams): \statespace \rightarrow \Reals$, representing the estimate of
\begin{subequations}
\begin{align}
	\newmath{H_{1,{\eplength}}(\state;\policyparams)} &=
    \newmath{\VaR_{\alpha} \Big(\costfunc_{\eplength}^{\policyparams} \Bigm|\state_{\eplength}=\state \Big)}, \quad \text{and}
    \label{eq:lastH1}
    \\
	H_{1,\timeidx}(\state;\policyparams) &=
    \VaR_{\alpha} \Big(\costfunc_{\timeidx}^{\policyparams} +
    \valuefunc_{\timeidx+1}(\state_{\timeidx+1}^{\policyparams};\policyparams)
    \Bigm|\state_{\timeidx}=\state \Big); \label{eq:otherH1}    
\end{align}
\end{subequations}
(iii) $H_{2,\timeidx}^{\psi_2}(\cdot; \policyparams): \statespace \rightarrow \PosReals$, corresponding to the estimate of the difference between the value function and $H_{1,\timeidx}$, i.e.
\begin{equation}
	H_{2,\timeidx}(\state;\policyparams) = \valuefunc_{\timeidx}(\state;\policyparams) - H_{1,\timeidx}(\state;\policyparams),
    \label{eq:H2}
\end{equation}
and (iv) $\valuefunc_{\timeidx}^{\valueparams}(\state;\policyparams) = H_{1,\timeidx}^{\psi_1}(\state;\policyparams) + H_{2,\timeidx}^{\psi_2}(\state;\policyparams)$ with $\valueparams=\{\psi_1, \psi_2\}$, the estimation of the value function when being in a state $\state$ at period $\timeidx$ and following the policy $\policy^{\policyparams}$.

Note that we characterise $H_{2,\timeidx}$ as the difference between the value function and $H_{1,\timeidx}$ rather than directly the value function for stability purposes.
Indeed, as mentioned in \cref{thm:elicitable-cvar}, we must enforce $\CVaR_{\alpha}$ to be greater than $\VaR_{\alpha}$, and this characterisation allows us to have reliable approximations of the risk measures without any additional constraint on the outputs of $H_{1,\timeidx}^{\psi_1}$ and $H_{2,\timeidx}^{\psi_2}$.


\subsection{Value Function Estimation}
\label{ssec:estimate-V}

To estimate the value function in \cref{eq:lastV,eq:otherV}, the two neural network parameters $\valueparams=\{\psi_1, \psi_2\}$ must be learnt simultaneously.
As we aim to optimise the dynamic \CVaR{}, we use its conditional elicitability property discussed in \cref{sec:elicitability}, more specifically the result in \cref{thm:elicitable-cvar}.
We thus need to minimise the following objective function as in the deep composite model regression approach depicted in
\begin{subequations}
\begin{align}
    &\newmath{\argmin_{\valueparams_{\eplength}} \; \EE_{\PP^{\policyparams}(\cdot,\cdot | \state_{\eplength}=\state)} \Bigg[ \score \bigg(
	\underbrace{
	H_{1,\eplength}^{\psi_{1,\eplength}} \Big( \state; \policyparams \Big)
	}_{\VaR_{\alpha}(\cdot | \state_{\eplength})}
	; \
	\underbrace{
	\valuefunc_{\eplength}^{\valueparams_{\eplength}} \Big( \state; \policyparams \Big)
	}_{\CVaR_{\alpha}(\cdot | \state_{\eplength})}
	; \
	\underbrace{
	\vphantom{ \Big(\Big) }
	\costfunc_{\eplength}^{\policyparams}
	}_{\text{running risk-to-go}}
	\bigg)
	\Bigg]}, \quad \text{and} \label{eq:true-loss-func-dynamic-cvar}
	\\
	&\argmin_{\valueparams_{\timeidx}} \; \EE_{\PP^{\policyparams}(\cdot,\cdot | \state_{\timeidx}=\state)} \Bigg[ \score \bigg(
	\underbrace{
	H_{1,\timeidx}^{\psi_{1,\timeidx}} \Big( \state; \policyparams \Big)
	}_{\VaR_{\alpha}(\cdot | \state_{\timeidx})}
	; \
	\underbrace{
	\valuefunc_{\timeidx}^{\valueparams_{\timeidx}} \Big( \state; \policyparams \Big)
	}_{\CVaR_{\alpha}(\cdot | \state_{\timeidx})}
	; \
	\underbrace{
	\costfunc_{\timeidx}^{\policyparams} + \valuefunc_{\timeidx+1}^{\valueparams_{\timeidx+1}} \Big( \state_{\timeidx+1}^{\policyparams}; \policyparams \Big)
	}_{\text{running risk-to-go}}
	\bigg)
	\Bigg], \label{eq:true-loss-func-dynamic-cvar-2}
\end{align}
\end{subequations}
where $\score$ is given in \cref{eq:scoring-dynamic-cvar}, the characterisation of the functions $G_1$, $G_2$ for $\score$ depends on the support of the costs $\costfunc^{\policyparams}_{\timeidx} \in \costspace$, and the expectation is taken wrt the transition probabilities $\PP^{\policyparams}(\action_{\timeidx},\state_{\timeidx+1} | \state_{\timeidx})$.
We also assume that the \CDF{}s of the costs have finite first moments and unique $\alpha$-quantiles for each period $\timeidx \in \periodspace$, which is not too restrictive in practice.
Here, we view $H_{1,\timeidx}^{\psi_1},H_{2,\timeidx}^{\psi_2},\valuefunc_{\timeidx}^{\psi}$ as ensembles of \ANN{}s, denoted $\{H_{1,\timeidx}^{\psi_{1,\timeidx}}\}_{\timeidx\in\periodspace},\{H_{2,\timeidx}^{\psi_{2,\timeidx}}\}_{\timeidx\in\periodspace},\{\valuefunc_{\timeidx}^{\psi_{\timeidx}}\}_{\timeidx\in\periodspace}$, and approximate them using single \ANN{}s with an augmented input for the period.

To overcome stability issues often arising in actor-critic approaches, we replace the parameters $\valueparams$ in the last term of the scoring function with $\tilde{\valueparams} = \{ \tilde{\psi}_1, \tilde{\psi}_2 \}$, which is a slowly-updated parametrisation of the \ANN{} approximation -- this technique is usually referred to in the \ML{} literature as using \emph{target networks} (see e.g. \cite{van2016deep}).
As the true distribution is unknown, we also replace the expectation in \cref{eq:true-loss-func-dynamic-cvar,eq:true-loss-func-dynamic-cvar-2} with the empirical mean based on the observed transitions in \cref{eq:transitions}.
Combining everything together, we aim to minimise over ANN parameters $\psi$ the loss function
\begin{equation}
\begin{split}
    \Ll^{\valueparams} &= \newmath{\sum_{\timeidx \in \periodspace \setminus \{\eplength\}}} \sum_{\batchidx=1}^{\Nbatchs} \Bigg[
	\score \bigg(
	H_{1,\timeidx}^{\psi_1} \Big( \state_{\timeidx}^{(\batchidx)}; \policyparams \Big)
	; \
	\valuefunc_{\timeidx}^{\valueparams} \Big( \state_{\timeidx}^{(\batchidx)}; \policyparams \Big)
	; \
	\costfunc_{\timeidx}^{(\batchidx)}
	+ \valuefunc_{\timeidx+1}^{\tilde{\valueparams}} \Big( \state_{\timeidx+1}^{(\batchidx)}; \policyparams \Big)
	\bigg)
	\Bigg] \\
	&\qquad\quad + \newmath{\sum_{\batchidx=1}^{\Nbatchs} \Bigg[
	\score \bigg(
	H_{1,\eplength}^{\psi_1} \Big( \state_{\eplength}^{(\batchidx)}; \policyparams \Big)
	; \
	\valuefunc_{\eplength}^{\valueparams} \Big( \state_{\eplength}^{(\batchidx)}; \policyparams \Big)
	; \
	\costfunc_{\eplength}^{(\batchidx)}
	\bigg)
	\Bigg]},
\end{split} \label{eq:loss-critic} \tag{L1}
\end{equation}
where $\tilde{\valueparams}$ is updated every, say, $K^{*}$ epochs.

The critic procedure to estimate the value function is summarised as follows.
For each epoch of the training loop, we initially set to zero the accumulated gradients in the \ANN{}s $H_{1,\timeidx}^{\psi_1}$ and $H_{2,\timeidx}^{\psi_2}$.
We then simulate a mini-batch of full episodes induced by the policy $\policy^{\policyparams}$, compute the loss function $\Ll^{\valueparams}$ given in \cref{eq:loss-critic}, update the parameters $\valueparams$ using an optimisation rule, such as the Adam optimiser \cite{kingma2014adam}, and modify the learning rates for $H_{1,\timeidx}^{\psi_1}$, $H_{2,\timeidx}^{\psi_2}$ according to a learning rate scheduler, such as an exponential decay or a cyclical learning rate rule \cite{smith2017cyclical}.
After the whole training, we obtain approximations for $H_{1,\timeidx}$ and the value function $\valuefunc_{\timeidx}$.

The following results state that we can approximate the value function to an arbitrary accuracy using our framework devised here.
\new{\begin{theorem}
    Suppose $\policy^{\policyparams}$ is a fixed policy, its value function $\valuefunc_{\timeidx}(\state;\policyparams)$ is given in \cref{eq:lastV,eq:otherV}, and the decomposition in terms of $H_{1,\timeidx}(\state;\policyparams)$ and $H_{1,\timeidx}(\state;\policyparams)$ is given in \cref{eq:lastH1,eq:otherH1,eq:H2}.
    Then for any $\varepsilon^{*}_1,\varepsilon^{*}_2 > 0$, there exist two \ANN{}s, denoted $H_{1,\timeidx}^{\psi_1}$ and $H_{2,\timeidx}^{\psi_2}$, such that for any $\timeidx \in \periodspace$, we have
    \begin{subequations}
    \begin{align}
        \esssup_{\state \in \statespace} \; \Big\Vert H_{1,\timeidx}(\state;\policyparams) - H_{1,\timeidx}^{\psi_{1}}(\state;\policyparams)\Big\Vert &< \varepsilon^{*}_{1}, \quad \text{and} \\
        \esssup_{\state \in \statespace} \; \Big\Vert H_{2,\timeidx}(\state;\policyparams) - H_{2,\timeidx}^{\psi_{2}}(\state;\policyparams)\Big\Vert &< \varepsilon^{*}_{2}.
    \end{align}
    \end{subequations}
    \label{thm:univ-approx-H}
\end{theorem}}

\begin{proof}
We start by recalling the universal approximation theorem (see e.g. \cite{cybenko1989approximation,pinkus1999approximation}), a powerful result proving the existence of \new{a single-layer fully-connected feed forward} \ANN{} approximating any absolutely continuous function to an arbitrary accuracy $\varepsilon>0$ \new{on a compact set}, given a sufficiently large neural network structure and activation functions that are \new{not polynomial}.

We also note that any bounded one-step conditional risk measure $\riskmeas_{\timeidx}$ satisfying the monotonicity and translation invariance properties is absolutely continuous a.s..
Indeed, using the two inequalities $\rv \leq \rvdum + \Vert \rv - \rvdum \Vert$ and $\rvdum \leq \rv + \Vert \rv - \rvdum \Vert$, where $\rv,\rvdum \in \Lpspace_{\timeidx+1}$ and the norm is to be understood as the $\infty$-norm, we have 
\begin{equation}
    \esssup \Big\Vert \riskmeas_{\timeidx}(\rv) - \riskmeas_{\timeidx}(\rvdum)\Big\Vert \, \leq \, \esssup \Big\Vert \rv-\rvdum \Big\Vert \, \leq \, \Big\Vert \rv-\rvdum \Big\Vert \, \leq \infty.
    \label{eq:absolutely-continuous}
\end{equation}
Thus, $\riskmeas_{\timeidx}$ is Lipschitz continuous a.s. wrt the essential supremum norm, and hence it is absolutely continuous a.s..
Furthermore, any linear combination of monotone and translation invariant risk measures remains absolutely continuous a.s..
This implies that we can use the universal approximation theorem on $\VaR_{\alpha}$, $\CVaR_{\alpha}$, and the difference between $\CVaR_{\alpha}$ and $\VaR_{\alpha}$ for any threshold $\alpha \in (0,1)$ due to their almost sure absolute continuity.

We prove that \ANN{}s $H_{1,\timeidx}^{\psi_1}$, $H_{2,\timeidx}^{\psi_2}$ approximate the mappings $H_{1,\timeidx}$, $H_{2,\timeidx}$ by induction.
Throughout the proof, we view $H_{1,\timeidx}^{\psi_1},H_{2,\timeidx}^{\psi_2}$ as ensembles of \ANN{}s for each period $\timeidx \in \periodspace$.
For the last period \new{$\eplength$}, using the universal approximation theorem on \cref{eq:lastH1,eq:H2}, we have that for any $\newmath{\varepsilon'_{\eplength}, \varepsilon''_{\eplength}} > 0$, there exist \ANN{}s \new{$H_{1,\eplength}^{\psi_{1,\eplength}}$ and $H_{2,\eplength}^{\psi_{2,\eplength}}$} such that
\begin{subequations}
\begin{align}
	\esssup_{\state \in \statespace} \; \Big\Vert \newmath{H_{1,\eplength}(\state;\policyparams) - H_{1,\eplength}^{\psi_{1,\eplength}}(\state;\policyparams)}\Big\Vert &< \newmath{\varepsilon'_{\eplength} =: \varepsilon_{1,\eplength}}, \quad \text{and} \label{eq:proof-approx-lastH1} 
    \\
	\esssup_{\state \in \statespace} \; \Big\Vert \newmath{H_{2,\eplength}(\state;\policyparams) - H_{2,\eplength}^{\psi_{2,\eplength}}(\state;\policyparams)} \Big\Vert &< \newmath{\varepsilon''_{\eplength} =: \varepsilon_{2,\eplength}}. \label{eq:proof-approx-lastH2}
\end{align}
\end{subequations}

This proves the base case of our proof by induction.
For the induction step, we show that $H_{1,\timeidx}^{\psi_1},H_{2,\timeidx}^{\psi_2}$ approximate the mappings $H_{1,\timeidx},H_{2,\timeidx}$ at the period $\timeidx$ with an arbitrary accuracy as long as they adequately approximate them at the periods $\timeidx+1$ up to \new{$\eplength$} inclusively.
Assume that for any $\varepsilon_{1,\tau},\varepsilon_{2,\tau} > 0$, there exist \ANN{}s $H_{1,\tau}^{\psi_{1,\tau}}$ and $H_{2,\tau}^{\psi_{2,\tau}}$ such that
\begin{subequations}
\begin{align}
	\esssup_{\state \in \statespace} \; \Big\Vert H_{1,\tau}(\state;\policyparams) - H_{1,\tau}^{\psi_{1,\tau}}(\state;\policyparams)\Big\Vert &< \varepsilon_{1,\tau}, \quad \text{and} \label{eq:proof-approx-lastH1-induction} 
    \\
	\esssup_{\state \in \statespace} \; \Big\Vert H_{2,\tau}(\state;\policyparams) - H_{2,\tau}^{\psi_{2,\tau}}(\state;\policyparams) \Big\Vert &< \varepsilon_{2,\tau}, \label{eq:proof-approx-lastH2-induction}
\end{align}
for all \new{$\tau\in\{\timeidx+1,\ldots,\eplength\}$}.
\end{subequations}
For the sake of brevity in this proof, we define the following $\Ff_{\timeidx+1}$-measurable random variables:
\begin{subequations}
\begin{align}
    \rvdum_{\timeidx} &:= \costfunc_{\timeidx}^{\policyparams} +
        H_{1,\timeidx+1} \Big( \state_{\timeidx+1}^{\policyparams};\policyparams \Big) + H_{2,\timeidx+1} \Big(\state_{\timeidx+1}^{\policyparams};\policyparams \Big), \quad \text{and}
        \label{eq:proof-temp-rv-truth} \\
    \rvdum^{\psi}_{\timeidx} &:= \costfunc_{\timeidx}^{\policyparams} +
        H_{1,\timeidx+1}^{\psi_{1,\timeidx+1}} \Big( \state_{\timeidx+1}^{\policyparams};\policyparams \Big) + H_{2,\timeidx+1}^{\psi_{2,\timeidx+1}} \Big(\state_{\timeidx+1}^{\policyparams};\policyparams \Big). \label{eq:proof-temp-rv-ANN}
\end{align}
\end{subequations}
\new{Using the triangle inequality (denoted \triangleineq{}),} we then obtain for $H_{1,\timeidx}$ that
\begin{align}
    &\esssup_{\state \in \statespace} \; \Big\Vert H_{1,\timeidx}(\state;\policyparams) - H_{1,\timeidx}^{\psi_{1,\timeidx}}(\state;\policyparams)\Big\Vert 
    \nonumber\\
    \text{\scriptsize{[\cref{eq:otherH1}]}}&\quad = \esssup_{\state \in \statespace} \; \Big\Vert \VaR_{\alpha} \Big(\costfunc_{\timeidx}^{\policyparams} +
    \valuefunc_{\timeidx+1}(\state_{\timeidx+1}^{\policyparams};\policyparams)
    \Bigm|\state_{\timeidx}=\state \Big) - H_{1,\timeidx}^{\psi_{1,\timeidx}}(\state;\policyparams) \Big\Vert 
    \nonumber\\
    \text{\scriptsize{[\cref{eq:H2}]}}&\quad = \esssup_{\state \in \statespace} \; \Big\Vert \VaR_{\alpha} \Big(\costfunc_{\timeidx}^{\policyparams} +
    H_{1,\timeidx+1}(\state_{\timeidx+1}^{\policyparams};\policyparams) + H_{2,\timeidx+1}(\state_{\timeidx+1}^{\policyparams};\policyparams)
    \Bigm|\state_{\timeidx}=\state \Big) \nonumber\\
    &\qquad\qquad\qquad - H_{1,\timeidx}^{\psi_{1,\timeidx}}(\state;\policyparams) \Big\Vert 
    \nonumber\\
    \text{\scriptsize{[\cref{eq:proof-temp-rv-truth}]}}&\quad = \esssup_{\state \in \statespace} \; \Big\Vert \VaR_{\alpha} \Big( \rvdum_{\timeidx} \Bigm|\state_{\timeidx}=\state \Big) - H_{1,\timeidx}^{\psi_{1,\timeidx}}(\state;\policyparams) \Big\Vert 
    \nonumber\\
    \text{\scriptsize{[\cref{eq:proof-temp-rv-ANN}]}}&\quad = \esssup_{\state \in \statespace} \; \Big\Vert \VaR_{\alpha} \Big( \rvdum_{\timeidx} \Bigm|\state_{\timeidx}=\state \Big) -\VaR_{\alpha} \Big(\rvdum^{\psi}_{\timeidx}
    \Bigm|\state_{\timeidx}=\state \Big) \nonumber\\
    &\qquad\qquad\qquad + \VaR_{\alpha} \Big(\rvdum^{\psi}_{\timeidx}
    \Bigm|\state_{\timeidx}=\state \Big) - H_{1,\timeidx}^{\psi_{1,\timeidx}}(\state;\policyparams) \Big\Vert 
    \nonumber\\
    \text{\scriptsize{[\triangleineq{}]}}&\quad 
        \leq \esssup_{\state \in \statespace} \; \Big\Vert \VaR_{\alpha} \Big(\rvdum^{\psi}_{\timeidx}
        \Bigm|\state_{\timeidx}=\state \Big) - H_{1,\timeidx}^{\psi_{1,\timeidx}}(\state;\policyparams) \Big\Vert 
        \nonumber\\
        &\qquad\quad + \esssup_{\state \in \statespace} \; \Big\Vert \VaR_{\alpha} \Big(\rvdum_{\timeidx}
        \Bigm|\state_{\timeidx}=\state \Big) - \VaR_{\alpha} \Big(\rvdum^{\psi}_{\timeidx}
        \Bigm|\state_{\timeidx}=\state \Big) \Big\Vert 
        \nonumber\\
    \text{\scriptsize{[\cref{eq:absolutely-continuous}]}}&\quad 
        \leq \esssup_{\state \in \statespace} \; \Big\Vert \VaR_{\alpha} \Big(\rvdum^{\psi}_{\timeidx}
        \Bigm|\state_{\timeidx}=\state \Big) - H_{1,\timeidx}^{\psi_{1,\timeidx}}(\state;\policyparams) \Big\Vert + \esssup_{\state \in \statespace} \; \Big\Vert \rvdum_{\timeidx} - \rvdum^{\psi}_{\timeidx} \Big\Vert 
        \nonumber\\
    \text{\scriptsize{[\triangleineq{}]}}&\quad 
        \leq \esssup_{\state \in \statespace} \; \Big\Vert \VaR_{\alpha} \Big(\rvdum^{\psi}_{\timeidx}
        \Bigm|\state_{\timeidx}=\state \Big) - H_{1,\timeidx}^{\psi_{1,\timeidx}}(\state;\policyparams) \Big\Vert 
        \nonumber\\
        &\qquad\quad + \esssup_{\state \in \statespace} \; \Big\Vert H_{1,\timeidx+1}(\state;\policyparams) - H_{1,\timeidx+1}^{\psi_{1,\timeidx+1}}(\state;\policyparams)
        \Big\Vert 
        \nonumber\\
        &\qquad\quad + \esssup_{\state \in \statespace} \; \Big\Vert H_{2,\timeidx+1}(\state;\policyparams) - H_{2,\timeidx+1}^{\psi_{2,\timeidx+1}}(\state;\policyparams) \Big\Vert 
        \nonumber\\
    \begin{split}
    \text{\scriptsize{[\cref{eq:proof-approx-lastH1-induction}, \cref{eq:proof-approx-lastH2-induction}]}}&\quad 
        < \esssup_{\state \in \statespace} \; \Big\Vert \VaR_{\alpha} \Big(\rvdum^{\psi}_{\timeidx}
        \Bigm|\state_{\timeidx}=\state \Big) - H_{1,\timeidx}^{\psi_{1,\timeidx}}(\state;\policyparams) \Big\Vert + \varepsilon_{1,\timeidx+1} + \varepsilon_{2,\timeidx+1}. 
        \label{eq:proof-cvar-induction-H1}
    \end{split}
\end{align}
Similarly, for $H_{2,\timeidx}$, we have the following:
\begin{align}
    &\esssup_{\state \in \statespace} \; \Big\Vert H_{2,\timeidx}(\state;\policyparams) - H_{2,\timeidx}^{\psi_{2,\timeidx}}(\state;\policyparams)\Big\Vert 
    \nonumber\\
    \text{\scriptsize{[\cref{eq:H2}]}}&\quad = \esssup_{\state \in \statespace} \; \Big\Vert\valuefunc_{\timeidx}(\state;\policyparams) - H_{1,\timeidx}(\state;\policyparams) - H_{2,\timeidx}^{\psi_{2,\timeidx}}(\state;\policyparams) \Big\Vert \nonumber\\
    \text{\scriptsize{[\cref{eq:otherV}, \cref{eq:otherH1}, \cref{eq:H2}]}}&\quad 
        = \esssup_{\state \in \statespace} \; \Big\Vert \CVaR_{\alpha} \Big(\rvdum_{\timeidx} \Bigm|\state_{\timeidx}=\state \Big)  - \VaR_{\alpha} \Big(\rvdum_{\timeidx} \Bigm|\state_{\timeidx}=\state \Big) - H_{2,\timeidx}^{\psi_{2,\timeidx}}(\state;\policyparams) \Big\Vert 
        \nonumber\\
    \text{\scriptsize{[\triangleineq{}]}}&\quad 
        \leq \esssup_{\state \in \statespace} \; \Big\Vert \CVaR_{\alpha} \Big(\rvdum^{\psi}_{\timeidx} \Bigm|\state_{\timeidx}=\state \Big) - \VaR_{\alpha} \Big(\rvdum^{\psi}_{\timeidx} \Bigm|\state_{\timeidx}=\state \Big) - H_{2,\timeidx}^{\psi_{2,\timeidx}}(\state;\policyparams) \Big\Vert 
        \nonumber\\
        &\qquad\quad + \esssup_{\state \in \statespace} \; \Big\Vert \CVaR_{\alpha} \Big(\rvdum_{\timeidx} \Bigm|\state_{\timeidx}=\state \Big) - \CVaR_{\alpha} \Big(\rvdum^{\psi}_{\timeidx} \Bigm|\state_{\timeidx}=\state \Big) \Big\Vert 
        \nonumber\\
        &\qquad\quad + \esssup_{\state \in \statespace} \; \Big\Vert \VaR_{\alpha} \Big(\rvdum_{\timeidx}
        \Bigm|\state_{\timeidx}=\state \Big) - \VaR_{\alpha} \Big(\rvdum^{\psi}_{\timeidx} \Bigm|\state_{\timeidx}=\state \Big) \Big\Vert 
        \nonumber\\
    \text{\scriptsize{[\cref{eq:absolutely-continuous}, \triangleineq{}]}}&\quad 
        \leq \esssup_{\state \in \statespace} \; \Big\Vert \CVaR_{\alpha} \Big(\rvdum^{\psi}_{\timeidx} \Bigm|\state_{\timeidx}=\state \Big) - \VaR_{\alpha} \Big(\rvdum^{\psi}_{\timeidx} \Bigm|\state_{\timeidx}=\state \Big) - H_{2,\timeidx}^{\psi_{2,\timeidx}}(\state;\policyparams) \Big\Vert
        \nonumber\\
        &\qquad\quad  + 2\esssup_{\state \in \statespace} \; \Big\Vert H_{1,\timeidx+1}(\state;\policyparams) - H_{1,\timeidx+1}^{\psi_{1,\timeidx+1}}(\state;\policyparams) \Big\Vert 
        \nonumber\\
        &\qquad\quad  + 2\esssup_{\state \in \statespace} \; \Big\Vert H_{2,\timeidx+1}(\state;\policyparams) - H_{2,\timeidx+1}^{\psi_{2,\timeidx+1}}(\state;\policyparams) \Big\Vert 
        \nonumber\\
    \begin{split}
    \text{\scriptsize{[\cref{eq:proof-approx-lastH1-induction}, \cref{eq:proof-approx-lastH2-induction}]}}&\quad 
        < \esssup_{\state \in \statespace} \; \Big\Vert \CVaR_{\alpha} \Big(\rvdum^{\psi}_{\timeidx} \Bigm|\state_{\timeidx}=\state \Big) - \VaR_{\alpha} \Big(\rvdum^{\psi}_{\timeidx} \Bigm|\state_{\timeidx}=\state \Big) - H_{2,\timeidx}^{\psi_{2,\timeidx}}(\state;\policyparams) \Big\Vert 
        \\
        &\qquad\quad  + 2 \Big(\varepsilon_{1,\timeidx+1} + \varepsilon_{2,\timeidx+1} \Big).
        \label{eq:proof-cvar-induction-H2}
    \end{split}
\end{align}
Using the universal approximation theorem on \cref{eq:proof-cvar-induction-H1,eq:proof-cvar-induction-H2}, we get that for any $\varepsilon'_{\timeidx},\varepsilon''_{\timeidx} > 0$, there exist \ANN{}s $H_{1,\timeidx}^{\psi_{1,\timeidx}},H_{2,\timeidx}^{\psi_{2,\timeidx}}$ such that
\begin{subequations}
\begin{align}
	\esssup_{\state \in \statespace} \; \Big\Vert H_{1,\timeidx}(\state;\policyparams) - H_{1,\timeidx}^{\psi_{1,\timeidx}}(\state;\policyparams)\Big\Vert &< \varepsilon'_{\timeidx}+ \varepsilon_{1,\timeidx+1} + \varepsilon_{2,\timeidx+1} =: \varepsilon_{1,\timeidx}, \quad \text{and} 
    \\
	\esssup_{\state \in \statespace} \; \Big\Vert H_{2,\timeidx}(\state;\policyparams) - H_{2,\timeidx}^{\psi_{2,\timeidx}}(\state;\policyparams) \Big\Vert &< \varepsilon''_{\timeidx} + 2 \Big(\varepsilon_{1,\timeidx+1} + \varepsilon_{2,\timeidx+1} \Big) =: \varepsilon_{2,\timeidx}.
\end{align}
\end{subequations}
This completes the proof by induction, and it is valid for any $\varepsilon'_{\timeidx},\varepsilon''_{\timeidx}>0$, $\timeidx \in \periodspace$.
Given \new{global errors $\varepsilon^{*}_{1},\varepsilon^{*}_{2}$} and sufficiently large neural network structures in terms of depth and number of nodes per layer, we can train the \ANN{}s to construct two sequences $\{\varepsilon_{1,\timeidx}\}_{\timeidx \in \periodspace}$ and $\{\varepsilon_{2,\timeidx}\}_{\timeidx \in \periodspace}$ such that \new{$\varepsilon_{1,0} < \varepsilon^{*}_1$ and $\varepsilon_{2,0} < \varepsilon^{*}_2$}.
\end{proof}

\new{\begin{corollary}
    Suppose $\policy^{\policyparams}$ is a fixed policy, and its value function $\valuefunc_{\timeidx}(\state;\policyparams)$ is given in \cref{eq:lastV,eq:otherV}.
    Then for any $\varepsilon^{*} > 0$, there exist two \ANN{}s, denoted $H_{1,\timeidx}^{\psi_1}$ and $H_{2,\timeidx}^{\psi_2}$, such that for any $\timeidx \in \periodspace$, we have
    \begin{equation}
        \esssup_{\state \in \statespace} \; \Big\Vert \valuefunc_{\timeidx}(\state; \policyparams) - \left( H_{1,\timeidx}^{\psi_1}(\state;\policyparams) + H_{2,\timeidx}^{\psi_2}(\state;\policyparams) \right) \Big\Vert < \varepsilon^{*}.
    \end{equation}
    \label{thm:univ-approx-V}
\end{corollary}}

\new{\begin{proof}
    As a consequence of \cref{thm:univ-approx-H}, for any $\varepsilon^{*}_1,\varepsilon^{*}_2 > 0$, there exist $H_{1,\timeidx}^{\valueparams_1}$ and $H_{2,\timeidx}^{\valueparams_1}$ such that
    \begin{subequations}
    \begin{align}
        \esssup_{\state \in \statespace} \; \Big\Vert H_{1,\timeidx}(\state;\policyparams) - H_{1,\timeidx}^{\psi_{1}}(\state;\policyparams)\Big\Vert &< \varepsilon^{*}_{1}, \quad \text{and} \label{eq:proof-approx-lastH1-end} \\
        \esssup_{\state \in \statespace} \; \Big\Vert H_{2,\timeidx}(\state;\policyparams) - H_{2,\timeidx}^{\psi_{2}}(\state;\policyparams)\Big\Vert &< \varepsilon^{*}_{2}, \label{eq:proof-approx-lastH2-end}
    \end{align}
    \end{subequations}
    for any $\timeidx\in\periodspace$. Therefore, for the value function at any period $\timeidx \in \periodspace$, we have
    \begin{align*}
        &\esssup_{\state \in \statespace} \; \Big\Vert \valuefunc_{\timeidx}(\state;\policyparams) - \left( H_{1,\timeidx}^{\psi_{1}}(\state;\policyparams) + H_{2,\timeidx}^{\psi_{2}}(\state;\policyparams) \right) \Big\Vert 
        \\
        \text{\scriptsize{[\cref{eq:H2}]}}&\quad = \esssup_{\state \in \statespace} \; \Big\Vert H_{1,\timeidx}(\state;\policyparams) + H_{2,\timeidx}(\state;\policyparams) - H_{1,\timeidx}^{\psi_{1}}(\state;\policyparams) - H_{2,\timeidx}^{\psi_{2}}(\state;\policyparams) \Big\Vert 
        \\
        \text{\scriptsize{[\triangleineq{}]}}&\quad \leq \esssup_{\state \in \statespace} \; \Big\Vert H_{1,\timeidx}(\state;\policyparams) - H_{1,\timeidx}^{\psi_{1}}(\state;\policyparams) \Big\Vert + \esssup_{\state \in \statespace} \; \Big\Vert H_{2,\timeidx}(\state;\policyparams) - H_{2,\timeidx}^{\psi_{2}}(\state;\policyparams) \Big\Vert
        \\
        \text{\scriptsize{[\cref{eq:proof-approx-lastH1-end}, \cref{eq:proof-approx-lastH2-end}]}}&\quad < \varepsilon^{*}_{1} + \varepsilon^{*}_{2}.
    \end{align*}
    Since $\varepsilon^{*}_1,\varepsilon^{*}_2$ are arbitrary, this concludes the proof.
\end{proof}}

In the proof of \cref{thm:univ-approx-H,thm:univ-approx-V}, we derive the result for \ANN{}s $H_{1,\timeidx}^{\psi_1},H_{2,\timeidx}^{\psi_2}$ that are ensembles of \ANN{}s for all periods $\timeidx \in \periodspace$.
In our implementation, however, we instead use single \ANN{}s with time as an additional dimension of the state space -- it requires less memory than storing ensembles of \ANN{}s, speeds up the critic procedure, and provides an appropriate approximation of the value function at any period. In fact, using the universal approximation theorem and a polynomial interpolation argument, one may prove that any finite ensemble of \ANN{}s can be approximated by a single \ANN{} with an augmented input (see \cite{coache2021reinforcement}).

\new{We also remark that the universal approximation theorem used here is valid for an \ANN{} with fixed depth (i.e. number of layers) but arbitrary width (i.e. number of nodes). Since the number of nodes may be exponentially large, the exact \ANN{} structure in terms of layers, nodes, and activation functions depends on the specific application. In practice, it is common that deep narrow layers outperform shallow but wide ones.}
\new{Finally,} given the recursion form of the value function in \cref{eq:lastV,eq:otherV}, and the proof of \cref{thm:univ-approx-V}, we understand that the accuracy of the estimation of the value function at a period $\timeidx$ strongly depends on the accuracy of its estimation at the subsequent periods.
Therefore, from a practical viewpoint, one may need to increase the number of epochs for the critic procedure in \cref{algo:actor-critic} if the number of periods is large.

\subsection{Policy Update}
\label{ssec:update-policy}

Next, we provide details regarding the gradient of the value function which is used to update the policy.
\begin{theorem}
    Suppose the logarithm of transition probabilities $\log \PP^{\policyparams}(\action,\statedum | \state)$ is a differentiable function in $\policyparams$ when $\PP^{\policyparams}(\action,\statedum | \state) \neq 0$, and its gradient wrt $\policyparams$ is bounded for any $(\action, \state) \in \actionspace \times \statespace$.
    Then, for any state $\state \in \statespace$, the gradient of the value function at period \new{$\eplength$} given in \cref{eq:lastV} is
    \begin{subequations}
    \begin{equation}
    	\newmath{\grad{\policyparams} \valuefunc_{\eplength} (\state; \policyparams) = \frac{1}{1-\alpha} \EE_{\PP^{\policyparams}(\cdot,\cdot | \state_{\eplength}=\state)} \bigg[
    	\Big( \costfunc_{\eplength}^{\policyparams} - \lambda^{*} \Big)_{+}
    	\bigg( \grad{\policyparams} \log \policy^{\policyparams} (\action | \state_{\eplength})\Big\rvert_{\action=\action_{\eplength}^{\policyparams}} \bigg)
    	\bigg]},
    	\label{eq:score-gradient-reinforce}
    \end{equation}
    and the gradient of the value function at periods \new{$\timeidx \in \periodspace \setminus \{\eplength\}$} given in \cref{eq:otherV} is
    \begin{equation}
    \begin{split}
    	\grad{\policyparams} \valuefunc_{\timeidx} (\state; \policyparams) &= \frac{1}{1-\alpha} \EE_{\PP^{\policyparams}(\cdot,\cdot | \state_{\timeidx}=\state)} \bigg[
    	\Big( \costfunc_{\timeidx}^{\policyparams} + \valuefunc_{\timeidx+1}(\state_{\timeidx+1}^{\policyparams}; \policyparams) - \lambda^{*} \Big)_{+}
    	\bigg( \grad{\policyparams} \log \policy^{\policyparams} (\action | \state_{\timeidx})\Big\rvert_{\action=\action_{\timeidx}^{\policyparams}} \bigg)
    	\bigg] \\
    	&\qquad\qquad\qquad +
    	\EE_{\PP^{\policyparams}(\cdot,\cdot | \state_{\timeidx}=\state)} \bigg[
    	\bigg( \grad{\policyparams} \valuefunc_{\timeidx+1}(\statedum; \policyparams)\Big\rvert_{\statedum=\state_{\timeidx+1}^{\policyparams}} \bigg) \; \weight^{*}(\action_{\timeidx}^{\policyparams}, \state_{\timeidx+1}^{\policyparams}) \bigg],
    \end{split} \label{eq:score-gradient-reinforce-2}
    \end{equation}
    \end{subequations}
    where $(\weight^{*}, \lambda^{*})$ is any saddle-point of the Lagrangian function of \cref{eq:lastV,eq:otherV} respectively.
    \label{thm:gradient-V}
\end{theorem}

\begin{proof}
We provide a proof for the gradient at the last period -- the gradient for the others periods is obtained in an analogous manner.
First, recall that using the representation theorem for coherent risk measures (see Theorem 6.4 of \cite{shapiro2014lectures}), the value function at period \new{$\timeidx=\eplength$} in \cref{eq:lastV} can be written as
\begin{equation}
    \newmath{\valuefunc_{\eplength} (\state;\policyparams) =
    \sup_{\weight \in \Uu(\PP^{\policyparams}(\cdot,\cdot|\state_{\eplength}=\state))} \EE_{\PP^{\policyparams}(\cdot,\cdot|\state_{\eplength}=\state)} \Big[
    \costfunc_{\eplength}^{\policyparams} \, \weight
    \Big]},
	\label{eq:proof-valuefunc}
\end{equation}
where the set $\Uu$ is the (convex) risk envelope
\begin{equation}
	\label{eq:risk-envelope-CVaR}
	\Uu(\PP^{\policyparams}(\cdot,\cdot|\state)) := \left\{ \weight : \EE_{\PP^{\policyparams}(\cdot,\cdot|\state)} \Big[ \weight \Big] = 1, \ \weight \in \left[ 0, \frac{1}{1-\alpha} \right] \right\}.
\end{equation}
For any state $\state \in \statespace$, the Lagrangian associated with the optimisation problem in \cref{eq:proof-valuefunc} is
\new{\begin{align}
	\lagrangian^{\policyparams}(\weight, \lambda) &= \EE_{\PP^{\policyparams}(\cdot,\cdot|\state_{\eplength}=\state)} \Big[
    \costfunc_{\eplength}^{\policyparams} \, \weight
    \Big] - \lambda  \left( \EE_{\PP^{\policyparams}(\cdot,\cdot|\state_{\eplength}=\state)} \Big[ \weight \Big] - 1 \right) \nonumber\\
    \begin{split}
	&=
	\sum_{(\action,\statedum)} \costfunc(\state_{\eplength},\action, \statedum) \,
	\PP^{\policyparams}(\action,\statedum | \state_{\eplength}=\state) \,
	\weight(\action,\statedum) \\
	&\qquad - \lambda \left( \sum_{(\action,\statedum)} \PP^{\policyparams}(\action,\statedum |  \state_{\eplength}=\state) \, \weight(\action,\statedum) - 1 \right).
	\label{eq:lagrangian-last-period}
	\end{split}
\end{align}}
One can show that there exists at least one saddle-point of the Lagrangian function in \cref{eq:lagrangian-last-period} due to the convexity of the \CVaR{} and its risk envelope $\Uu$.
Furthermore, any saddle-point $(\weight^{*}, \lambda^{*})$ satisfies $\weight^{*}(\action,\statedum) = \frac{1}{1-\alpha}$ if $\newmath{\costfunc_{\eplength}^{\policyparams}} > \lambda^{*}$ and $\weight^{*}(\action,\statedum) = 0$ otherwise, where $\lambda^{*}$ is any $\alpha$-quantile of $\newmath{\costfunc_{\eplength}^{\policyparams}}$.
Using Slater's condition and the convexity of the problem, strong duality holds, i.e.
\begin{equation}
    \newmath{\valuefunc_{\eplength}(\state;\policyparams)} = \max_{\weight \geq 0} \, \min_{\lambda} \, \lagrangian^{\policyparams}(\weight, \lambda). \label{eq:strong-duality}
\end{equation}

Next, we apply the Envelope theorem for saddle-point problems (see Theorem 4 in \cite{milgrom2002envelope}) on \cref{eq:strong-duality}, which gives us that the gradient of the maximisation problem equals to the gradient of the Lagrangian evaluated at one of its saddle-points.
It relies on the equidifferentiability of the objective function and the absolute continuity of its gradient -- these properties are satisfied as we work with a spectral (and thus coherent and convex) risk measure.

\new{We remark that} the transition probabilities $\PP^{\policyparams}(\action,\statedum | \state)$ depends on $\policyparams$ only through the policy $\policy^{\policyparams}(\action | \state)$.
\new{By the assumptions in the Theorem statement, $\PP^{\policyparams}(\action,\statedum | \state)$ is differentiable and bounded, hence policy gradient can be applied. This assumption is standard among policy gradient methods (see e.g. \cite{sutton2000policy,konda2000actor,marbach2001simulation}).}
Using the Envelope theorem on \cref{eq:strong-duality} and the log derivative identity
\begin{equation}
    \newmath{p(x; \policyparams) \; \grad{\policyparams} \log \left( p(x; \policyparams) \right) = \grad{\policyparams} \, p(x;\policyparams)},
    \label{eq:likelihood-trick}
\end{equation}
also known as the ``likelihood-ratio'' trick in the \ML{} literature, we obtain
\begin{align}
    &\grad{\policyparams} \newmath{\valuefunc_{\eplength} (\state; \policyparams)} \nonumber \\
    \text{\scriptsize{[\cref{eq:strong-duality}]}}&\quad= \grad{\policyparams} \max_{\weight \geq 0} \, \min_{\lambda} \, \lagrangian^{\policyparams}(\weight, \lambda) \nonumber\\
    \text{\scriptsize{[Envelope]}}&\quad= \grad{\policyparams} \lagrangian^{\policyparams}(\weight, \lambda)
    \Big|_{\weight^{*}, \lambda^{*}} \nonumber\\
    \text{\scriptsize{[\cref{eq:lagrangian-last-period}]}}&\quad=
    \newmath{\sum_{(\action,\statedum)}
    \Big( \costfunc(\state_{\eplength},\action, \statedum) - \lambda^{*} \Big)
	\grad{\policyparams} \PP^{\policyparams}(\action,\statedum | \state_{\eplength}=\state)
	\; \weight^{*}(\action,\statedum)}
    \nonumber\\
    \text{\scriptsize{[\cref{eq:likelihood-trick}]}}&\quad=
    \newmath{\sum_{(\action,\statedum)}
    \Big( \costfunc(\state_{\eplength},\action, \statedum) - \lambda^{*} \Big)
    \PP^{\policyparams}(\action,\statedum | \state_{\eplength}=\state)
	\grad{\policyparams} \log \PP^{\policyparams}(\action,\statedum | \state_{\eplength}=\state)
	\; \weight^{*}(\action,\statedum)}
    \nonumber\\
    &\quad=
    \newmath{\EE_{\PP^{\policyparams}(\cdot,\cdot | \state_{\eplength}=\state)} \bigg[
	\Big( \costfunc_{\eplength}^{\policyparams} - \lambda^{*} \Big)
    \bigg( \grad{\policyparams} \log \PP^{\policyparams} (\action,\state | \state_{\eplength})\Big\rvert_{(\action,\state)=(\action_{\eplength}^{\policyparams},\state_{\eplength+1}^{\policyparams})} \bigg)
	\; \weight^{*}(\action_{\eplength}^{\policyparams}, \state_{\eplength}^{\policyparams}) \bigg]}
	\nonumber\\
	\text{\scriptsize{[form of $\weight^{*}$]}}&\quad=
    \newmath{\frac{1}{1-\alpha} \EE_{\PP^{\policyparams}(\cdot,\cdot | \state_{\eplength}=\state)} \bigg[
	\Big( \costfunc_{\eplength}^{\policyparams} - \lambda^{*} \Big)_{+}
    \bigg( \grad{\policyparams} \log \PP^{\policyparams} (\action,\state | \state_{\eplength-1})\Big\rvert_{(\action,\state)=(\action_{\eplength}^{\policyparams},\state_{\eplength+1}^{\policyparams})} \bigg) \bigg]}
	\nonumber\\
    \text{\scriptsize{[$\policyparams$-depend.]}}&\quad=
    \newmath{\frac{1}{1-\alpha} \EE_{\PP^{\policyparams}(\cdot,\cdot | \state_{\eplength}=\state)} \bigg[
	\Big( \costfunc_{\eplength}^{\policyparams} - \lambda^{*} \Big)_{+}
	\bigg( \grad{\policyparams} \log \policy^{\policyparams} (\action | \state_{\eplength})\Big\rvert_{\action=\action_{\eplength}^{\policyparams}} \bigg) \bigg]}.
	\nonumber
\end{align}
\end{proof}

In what follows, we assume that we can obtain samples from the policy using the \emph{reparameterization trick}.
This technique is widely used in generative modelling, because it removes randomness from the stochastic computation graph to allow the back-propagation of derivatives for all operations in the \ANN{}s.
As a common example, if $\policy^{\policyparams}(\cdot | \state) \sim \normaldist(\mu^{\policyparams},\sigma^{\policyparams})$, then one would simulate from $\mu^{\policyparams} + Z \sigma^{\policyparams}$ with $Z$ distributed as $\normaldist(0,1)$ instead of simulating directly from $\normaldist(\mu^{\policyparams},\sigma^{\policyparams})$.
More generally, any probability distribution in the location-scale family or with a tractable inverse \CDF{} is amenable to the reparametrization trick.

We aim to optimise the value function $\valuefunc_{\timeidx}(\cdot; \policyparams)$ over policies $\policy^{\policyparams}$ using a \emph{policy gradient approach} \cite{sutton2000policy}, which proposes to update parameters of the policy using the optimisation rule $\policyparams \leftarrow \policyparams - \lrate \grad{\policyparams} \valuefunc_{\timeidx}(\cdot; \policyparams)$.
Our proposed framework with two neural networks is convenient when dealing with the dynamic \CVaR{}, as we may take advantage of the gradient formulae in \cref{eq:score-gradient-reinforce,eq:score-gradient-reinforce-2} and the form of the saddle-points $(\weight^{*}, \lambda^{*})$.
More specifically, as $\lambda^{*}$ is the $\alpha$-quantile (and thus the value-at-risk with threshold $\alpha$) of the running risk-to-go, we have
\begin{equation}
    \lambda^{*} = H_{1,\timeidx}(\state; \policyparams)
    \quad \text{and} \quad
    \weight^{*}(\action_{\timeidx}^{\policyparams}, \state_{\timeidx+1}^{\policyparams}) = \frac{1}{1-\alpha} \, \Ind_{\costfunc_{\timeidx}^{\policyparams} + \valuefunc_{\timeidx+1}(\state_{\timeidx+1}^{\policyparams}; \policyparams) > \lambda^{*}}.
    \label{eq:saddle-points}
\end{equation}
We then reuse the estimates $H_{1,\timeidx}^{\psi_1},H_{2,\timeidx}^{\psi_2},\valuefunc_{\timeidx}^{\valueparams}$ already computed by the critic in \cref{ssec:estimate-V} to get an estimation of these saddle-points.
The conditional elicitability of the dynamic \CVaR{} serves our purpose for both the estimation of the value function and the policy gradient method.

As the true transition probability distributions are unknown to the agent, we instead take the empirical mean based on the transitions in \cref{eq:transitions}.
Combining this with the gradient formulae in \cref{thm:gradient-V} and the estimation of the saddle-points in \cref{eq:saddle-points}, we obtain the loss function
{\small{
\begin{equation}
\begin{split}
    \Ll^{\policyparams} &= \frac{1}{1-\alpha} \newmath{\sum_{\timeidx \in \periodspace \setminus \{\eplength\}}} \sum_{\batchidx=1}^{\Nbatchs} \Bigg[
	\Big( \costfunc_{\timeidx}^{(\batchidx)} + \valuefunc_{\timeidx+1}^{\valueparams}(\state_{\timeidx+1}^{(\batchidx)}; \policyparams) - H_{1,\timeidx}^{\psi_1}(\state_{\timeidx}^{(\batchidx)}; \policyparams) \Big)_{+}
	\bigg( \grad{\policyparams} \log \policy^{\policyparams} (\action | \state_{\timeidx}^{(\batchidx)})\Big\rvert_{\action=\action_{\timeidx}^{(\batchidx)}} \bigg)
	\Bigg] \\
	&\qquad\quad + \frac{1}{1-\alpha}
	\sum_{\batchidx=1}^{\Nbatchs} \newmath{\Bigg[
	\Big( \costfunc_{\eplength}^{(\batchidx)} - H_{1,\eplength}^{\psi_1}(\state_{\eplength}^{(\batchidx)}; \policyparams) \Big)_{+}
	\bigg( \grad{\policyparams} \log \policy^{\policyparams} (\action | \state_{\eplength}^{(\batchidx)})\Big\rvert_{\action=\action_{\eplength}^{(\batchidx)}} \bigg)
	\Bigg]},
\end{split} \label{eq:loss-func-gradient-reinforce} \tag{L2}
\end{equation}
}}where the expectation of the gradient of $\valuefunc$ in \cref{eq:score-gradient-reinforce-2} is omitted as we fix the value function of the actor part, and its \ANN{} does not depend explicitly on $\policyparams$; a common approach in the literature, see e.g. \cite{degris2012off}.

We note that if most of the observations $\costfunc_{\timeidx}^{\policyparams} + \valuefunc_{\timeidx+1}^{\valueparams}(\state_{\timeidx+1}^{\policyparams}; \policyparams)$ are smaller than $\lambda^{*}$, then the gradient would be very close to zero, and the policy $\policy^{\policyparams}$ would very slowly learn due to its almost null loss.
Indeed, for a mini-batch of $N$ episodes of $\eplength$ periods where one optimises the dynamic \CVaR{} at level $\alpha$, one expects that $(1-\alpha) N \eplength$ transitions contributes to the gradient.
To overcome this potential issue, we suggest to weight the mini-batch size $N$ by $\frac{1}{1-\alpha}$.
There are other methods one could use to assist with this issue, such as \emph{importance sampling} or \emph{state distribution correction factors} similarly to \cite{huang2021convergence}. In this work, we must note that the authors derive the correction factors from the stationary state distribution, and mention that their (not-implemented) algorithm, which ``can serve as a basis for future work'', requires an additional \ANN{} when dealing with large state spaces.

We summarise the actor procedure to update the policy as follows.
For each epoch of the training loop, we initially set to zero the accumulated gradients in the \ANN{} $\policy^{\policyparams}$.
We then simulate a $\frac{1}{1-\alpha}$-weighted mini-batch of full episodes induced by the policy $\policy^{\policyparams}$, compute the loss function $\Ll^{\policyparams}$ given in \cref{eq:loss-func-gradient-reinforce}, update the parameters $\policyparams$ using an optimisation rule, and decay the learning rate for $\policy^{\policyparams}$ based on a certain learning rate scheduler.


\section{Experiments}
\label{sec:experiments}

In this section, we validate our proposed framework on two benchmark applications.
We apply our actor-critic algorithm on a statistical arbitrage example in \cref{ssec:statistical-arbitrage} and recover results from \cite{coache2021reinforcement}.
We also explore a portfolio allocation problem and solve it using our model-agnostic approach in \cref{ssec:portfolio-allocation}.
For readability purposes, we refer the reader to \cref{sec:appendix-hyperparams} for an exhaustive description of the hyperparameters used in these two examples.
Finally, we formulate settings for other \RL{} problems in \cref{ssec:other-examples} to illustrate the wide range of applications for which our approach may be used.


\subsection{Statistical Arbitrage Example}
\label{ssec:statistical-arbitrage}

Suppose an agent begins each episode with zero inventory, and at each period the agent wishes to trade quantities of an asset, whose price fluctuates according to some data-generating processes.
There are quadratic costs associated with each transaction, and the market imposes a terminal penalty encouraging the agent to liquidate their inventory by the end of the episode.
For each period $\timeidx \in \periodspace$, the agent observes the asset's price  $\price_{\timeidx} \in \pricespace$ and their inventory $q_{\timeidx} \in (q_{\min}, q_{\max})$, performs a trade $\action_{\timeidx}^{\policyparams} \in (\action_{\min}, \action_{\max})$, resulting in wealth $y_{\timeidx} \in \Reals$ according to
\begin{equation*}
\label{eq:wealth}
\begin{cases}
\begin{aligned}
	y_0 &= 0,
	\\
	y_{\timeidx} &= y_{\timeidx-1}
	- \action_{\timeidx-1}^{\policyparams} \price_{\timeidx-1}
	- (\action_{\timeidx-1}^{\policyparams})^2 \phi_1, \qquad \newmath{\timeidx=1, \ldots, \eplength}
	\\
	\newmath{y_{\eplength+1}} &= \newmath{y_{\eplength}
	- \action_{\eplength}^{\policyparams} \price_{\eplength}
	- (\action_{\eplength}^{\policyparams})^2 \phi_1
	+ q_{\eplength+1} \price_{\eplength+1}
	- q_{\eplength+1}^2 \phi_2},
\end{aligned},
\end{cases}
\end{equation*}
with coefficients $\phi_1=0.005$ and $\phi_2=0.5$ for the transaction costs and terminal penalty, respectively.
We suppose that \new{$\eplength=4$}, $q_{\max} = -q_{\min} = 5$, $\action_{\max} = -\action_{\min} = 2$, and the asset price follows an Ornstein-Uhlenbeck process given by the stochastic differential equation (\SDE{})
\begin{equation*}
	\dee \price_{\timeidx} = \kappa (\mu - \price_{\timeidx}) \dee \timeidx + \sigma \dee W_{\timeidx},
\end{equation*}
where $\kappa = 2$, $\mu = 1$, $\sigma = 0.2$ and $W_{\timeidx}$ is a standard Brownian motion.
The statistical arbitrage consists in taking positions in the stock to take advantage of temporary price deviations from the mean-reverting value $\mu$.
The risk-aware agent tries to optimise the \RL{} problem stated in \cref{eq:optim-problem1}, where for every period, the actions are determined by the trades $\action_{\timeidx}$, the costs by the differences in wealth $\costfunc_{\timeidx} = y_{\timeidx} - y_{\timeidx+1}$, and the states by the tuples $(\timeidx, \price_{\timeidx}, q_{\timeidx})$.

We compare the optimal policies learnt with both the elicitable and nested simulation (proposed by \cite{coache2021reinforcement}) approaches in \cref{fig:stat-arbitrage}.
We first observe that our proposed approach gives a similar approximation to the time-dependent optimal solution.
We remark that the estimation differs for extreme prices -- we believe it is due to the fact that when randomising the initial states of the episodes, most of them do not have major price changes as time progresses.
Also, our algorithm is more efficient in terms of memory and time.
Indeed, the approach does not require additional transitions for every visited state, and the training procedure takes approximately half of the time compared to the nested simulation approach.
We suspect that this time improvement is mostly due to the actor procedure, where the nested approach must repeatedly simulate outer episodes and inner transitions.

We acknowledge that this set of experiments is not an exhaustive comparison between the two approaches, as there are many components to consider: i.e. the code implementation, number of inner transitions, dimensions of the state space, initialisation of the \ANN{}s, etc.
The key message of this illustrative example is that one can retrieve an \emph{accurate estimation} of the dynamic risk measure, and optimise over policies in \emph{a fast and efficient manner} using exclusively \emph{full episodes} and \emph{an additional \ANN{}}.

\begin{figure}[htbp]
    \centering
        \begin{minipage}[c]{0.0525\textwidth}
            \textbf{\footnotesize{elic.}}
        \end{minipage}
    	\hfill
        \begin{minipage}[c]{0.3025\textwidth}
    		\centering
    		\includegraphics[width=0.90\textwidth]{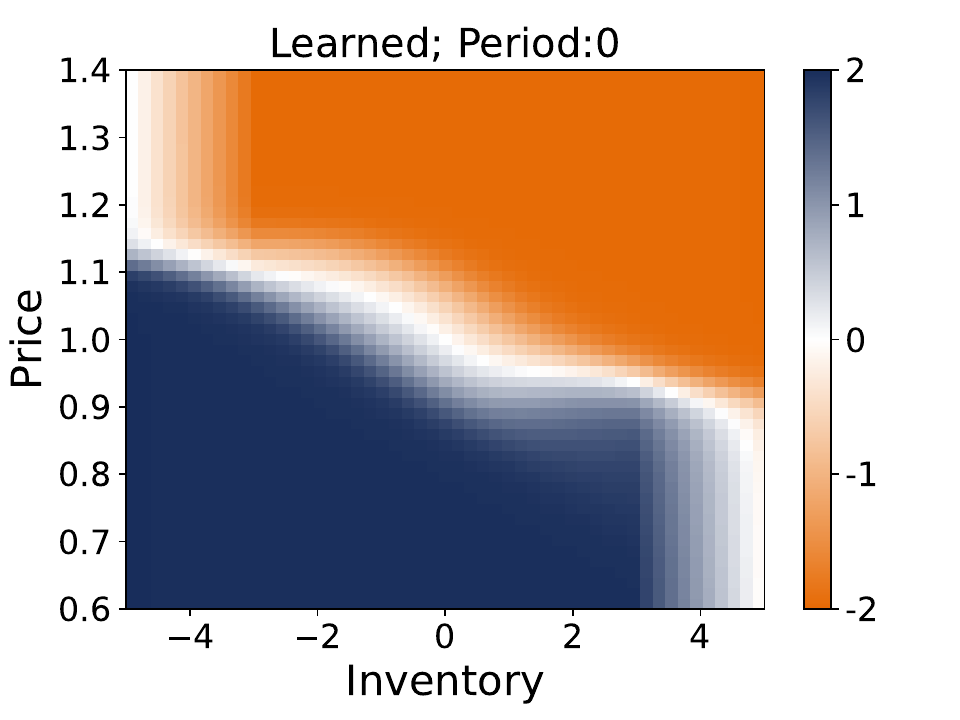}
    	\end{minipage}
    	\hfill
    	\begin{minipage}[c]{0.3025\textwidth}
    		\centering
    		\includegraphics[width=0.90\textwidth]{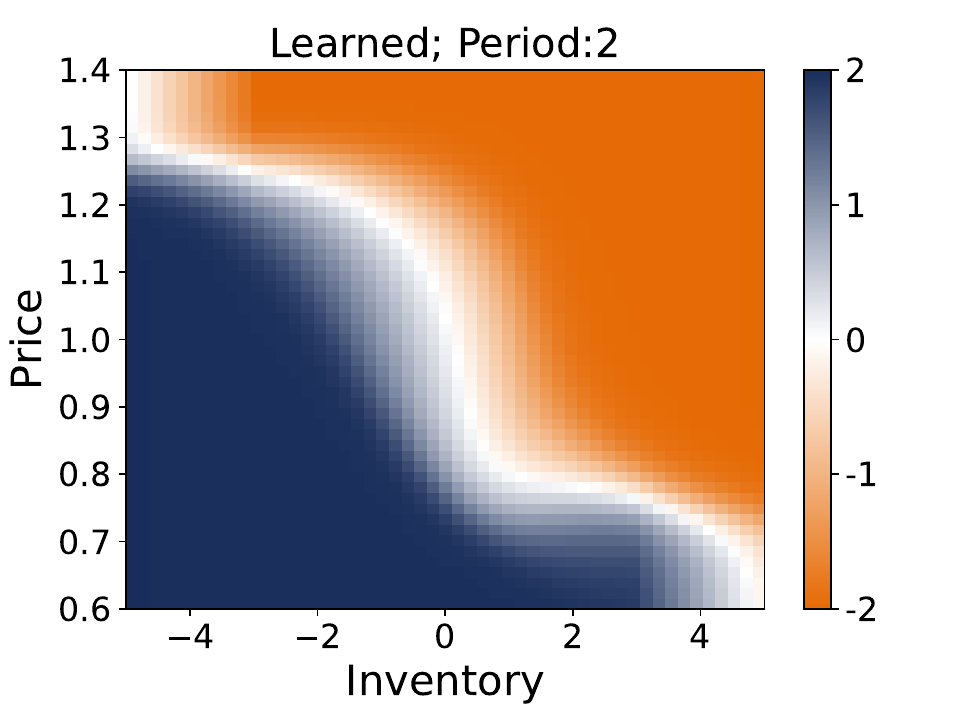}
    	\end{minipage}
    	\hfill
    	\begin{minipage}[c]{0.3025\textwidth}
    		\centering
    		\includegraphics[width=0.90\textwidth]{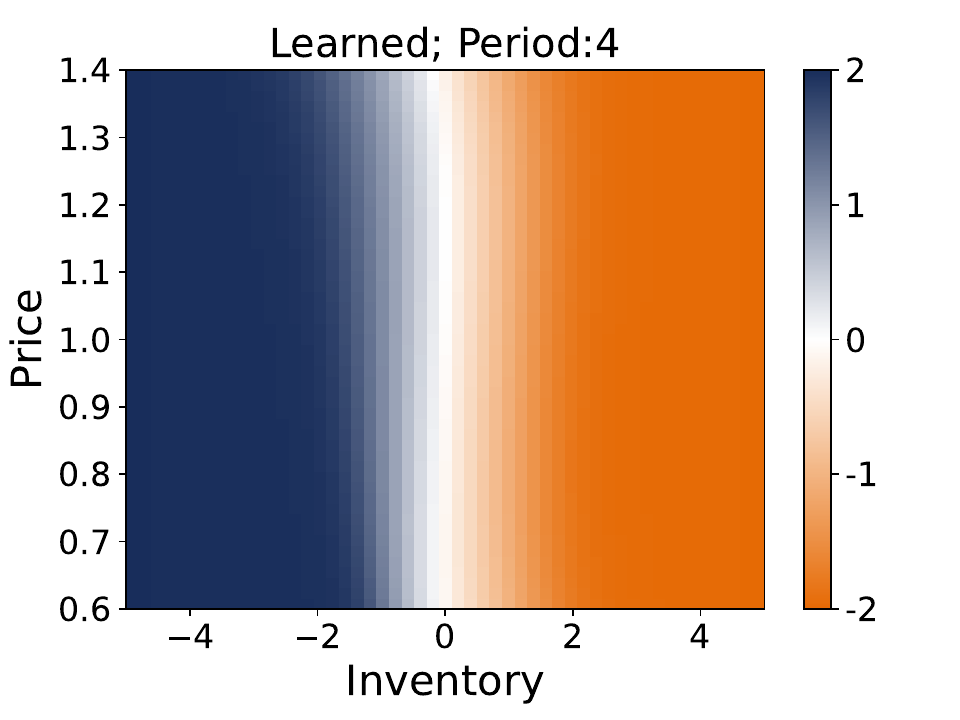}
    	\end{minipage}
    	\begin{minipage}[c]{0.0525\textwidth}
            \textbf{\footnotesize{nest.}}
        \end{minipage}
    	\hfill
    	\begin{minipage}[c]{0.3025\textwidth}
    		\centering
    		\includegraphics[width=0.90\textwidth]{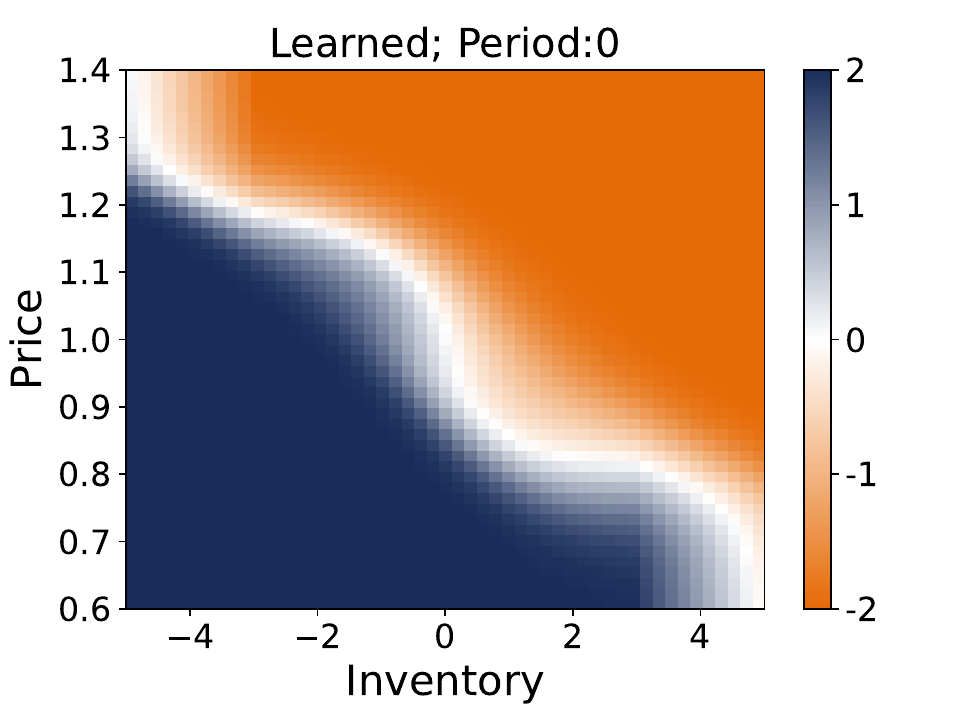}
    	\end{minipage}
    	\hfill
    	\begin{minipage}[c]{0.3025\textwidth}
    		\centering
    		\includegraphics[width=0.90\textwidth]{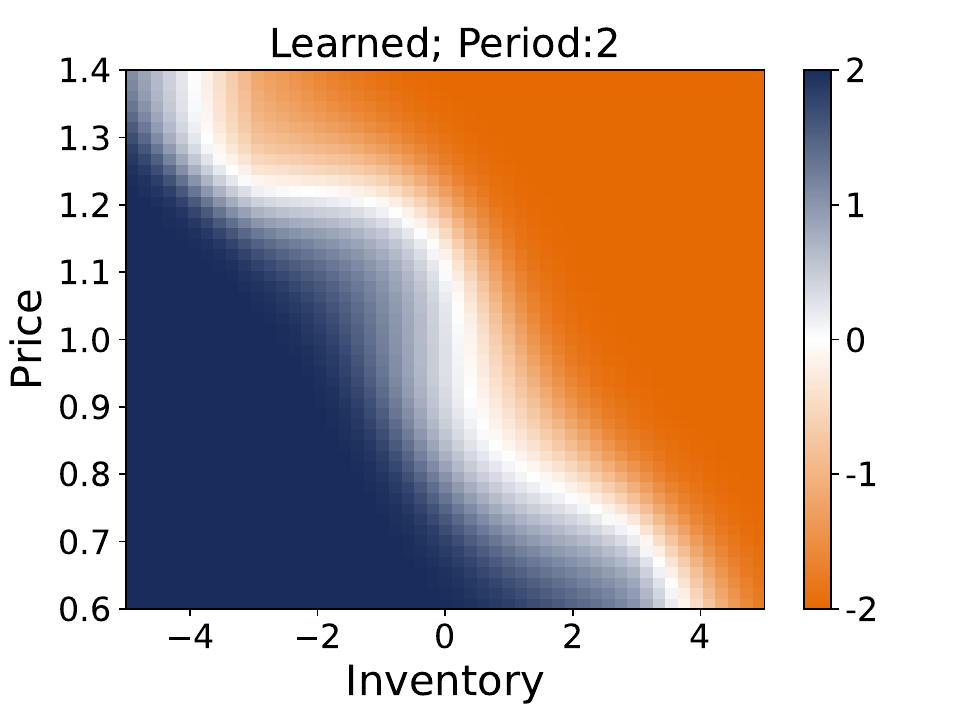}
    	\end{minipage}
    	\hfill
    	\begin{minipage}[c]{0.3025\textwidth}
    		\centering
    		\includegraphics[width=0.90\textwidth]{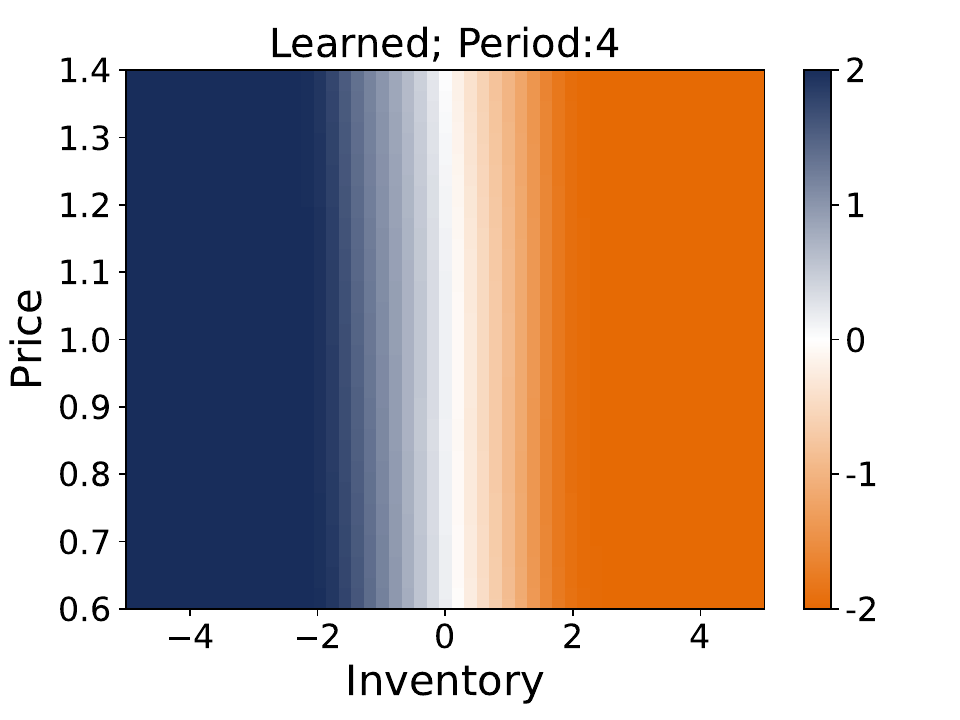}
    	\end{minipage}
    \subfloat[$\alpha = 0.5$]{
        \hphantom{\begin{minipage}[c]{0.96\textwidth} \centering \includegraphics[width=0.90\textwidth]{figure-files/figures-stat-arbitrage/CVaR0.5/ANN_actions_CVaR0.5_4.pdf}
    	\end{minipage}}
    }\\
        \begin{minipage}[c]{0.0525\textwidth}
            \textbf{\footnotesize{elic.}}
        \end{minipage}
    	\hfill
        \begin{minipage}[c]{0.3025\textwidth}
    		\centering
    		\includegraphics[width=0.90\textwidth]{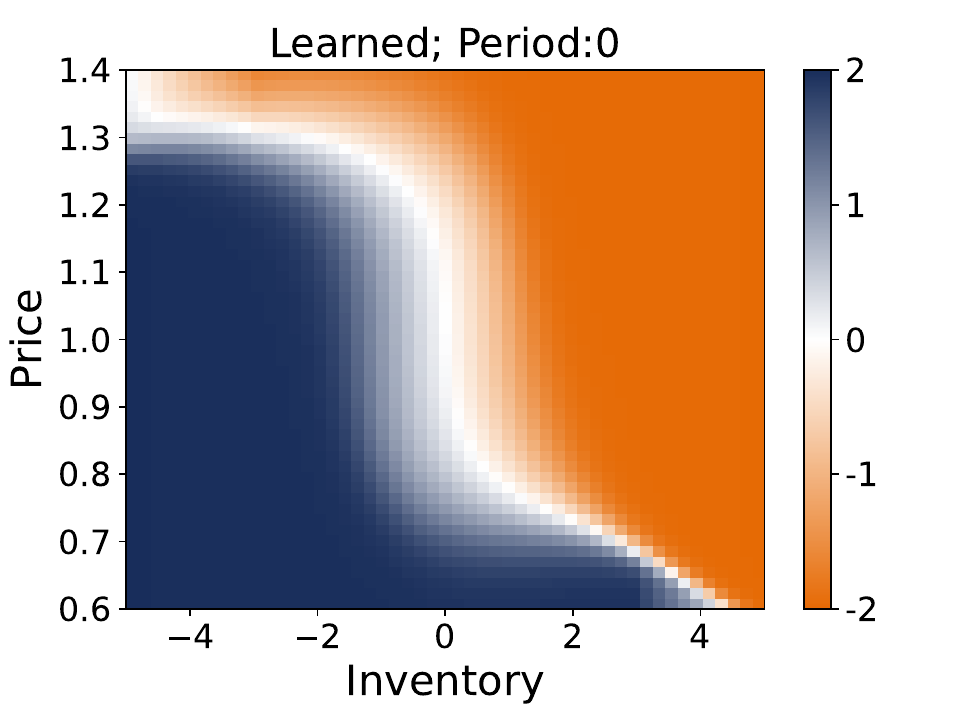}
    	\end{minipage}
    	\hfill
    	\begin{minipage}[c]{0.3025\textwidth}
    		\centering
    		\includegraphics[width=0.90\textwidth]{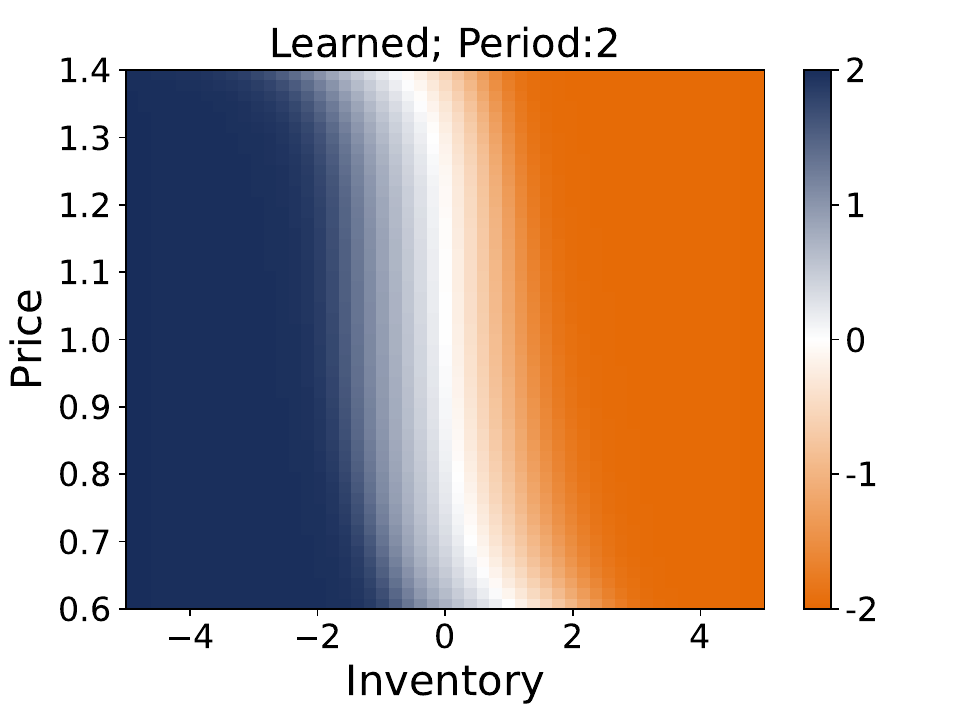}
    	\end{minipage}
    	\hfill
    	\begin{minipage}[c]{0.3025\textwidth}
    		\centering
    		\includegraphics[width=0.90\textwidth]{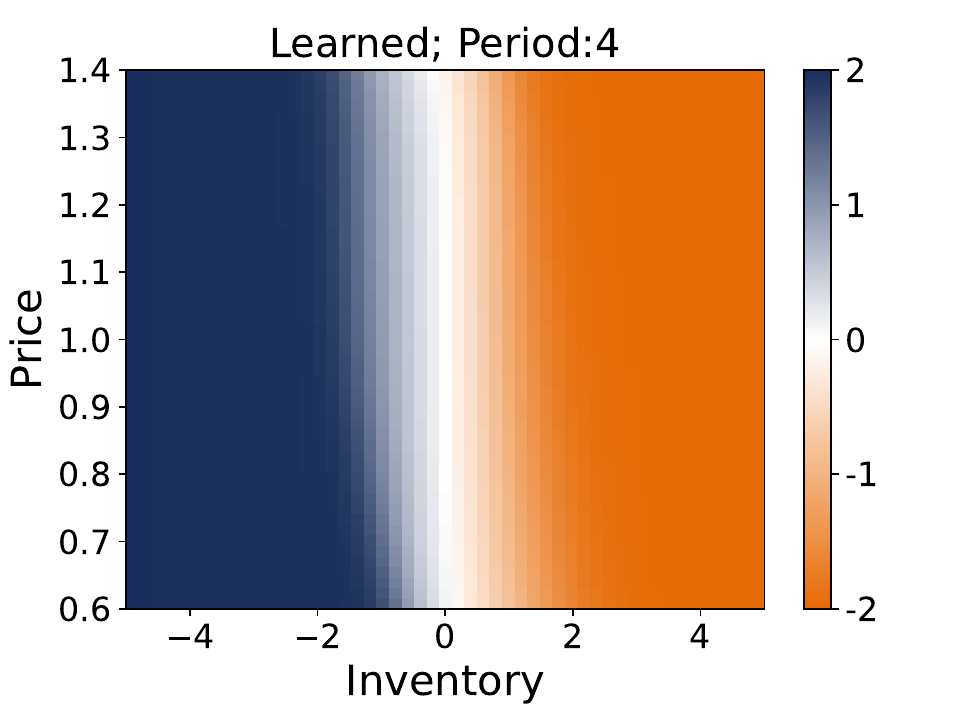}
    	\end{minipage}
    	\begin{minipage}[c]{0.0525\textwidth}
            \textbf{\footnotesize{nest.}}
        \end{minipage}
    	\hfill
    	\begin{minipage}[c]{0.3025\textwidth}
    		\centering
    		\includegraphics[width=0.90\textwidth]{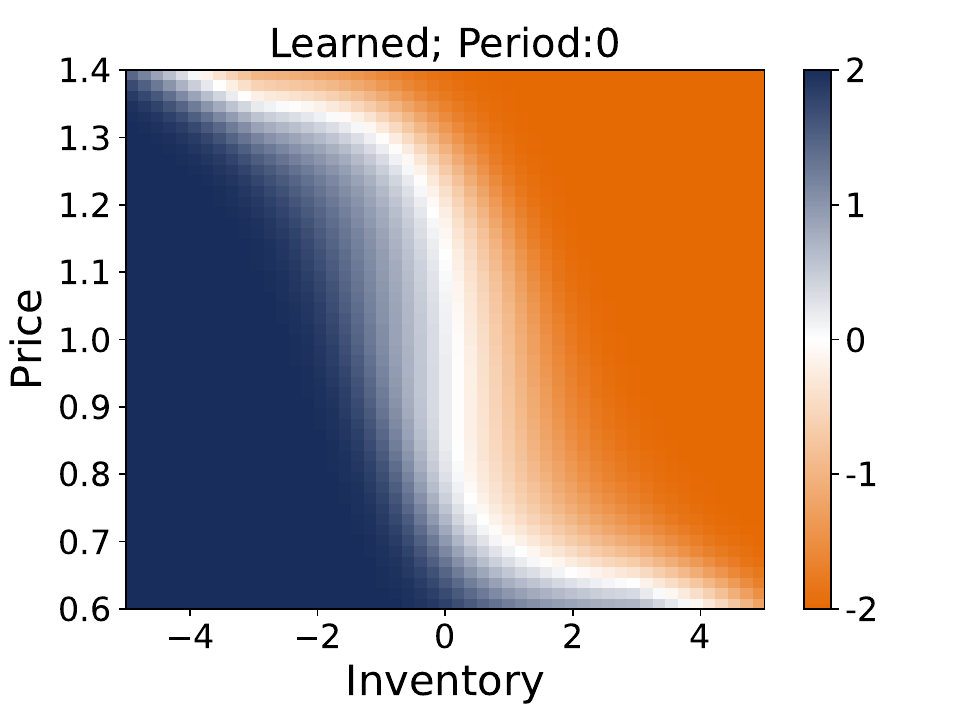}
    	\end{minipage}
    	\hfill
    	\begin{minipage}[c]{0.3025\textwidth}
    		\centering
    		\includegraphics[width=0.90\textwidth]{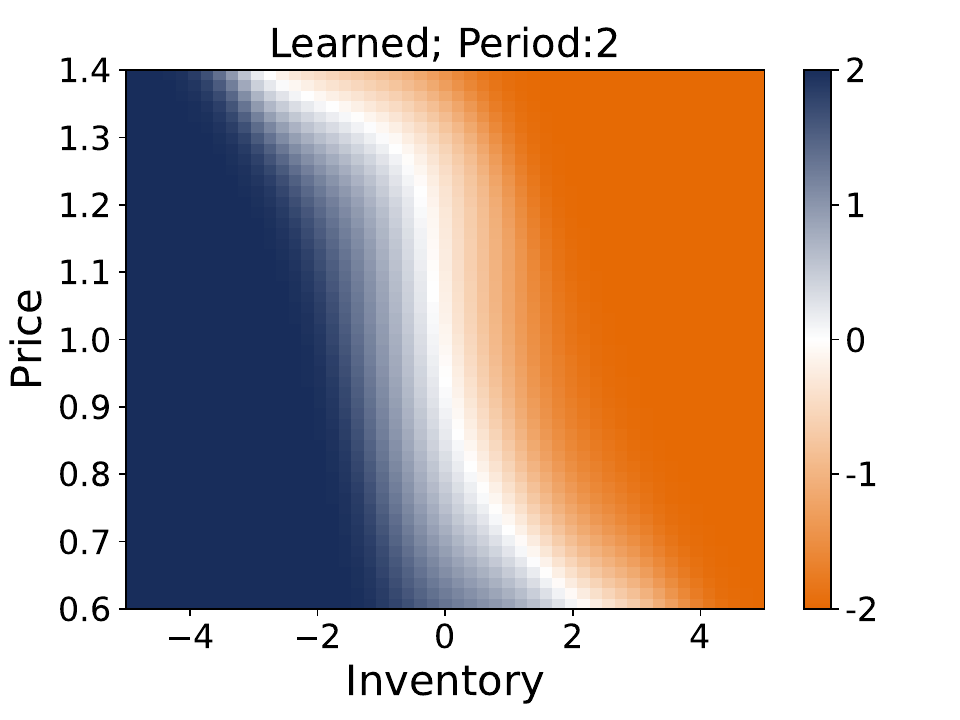}
    	\end{minipage}
    	\hfill
    	\begin{minipage}[c]{0.3025\textwidth}
    		\centering
    		\includegraphics[width=0.90\textwidth]{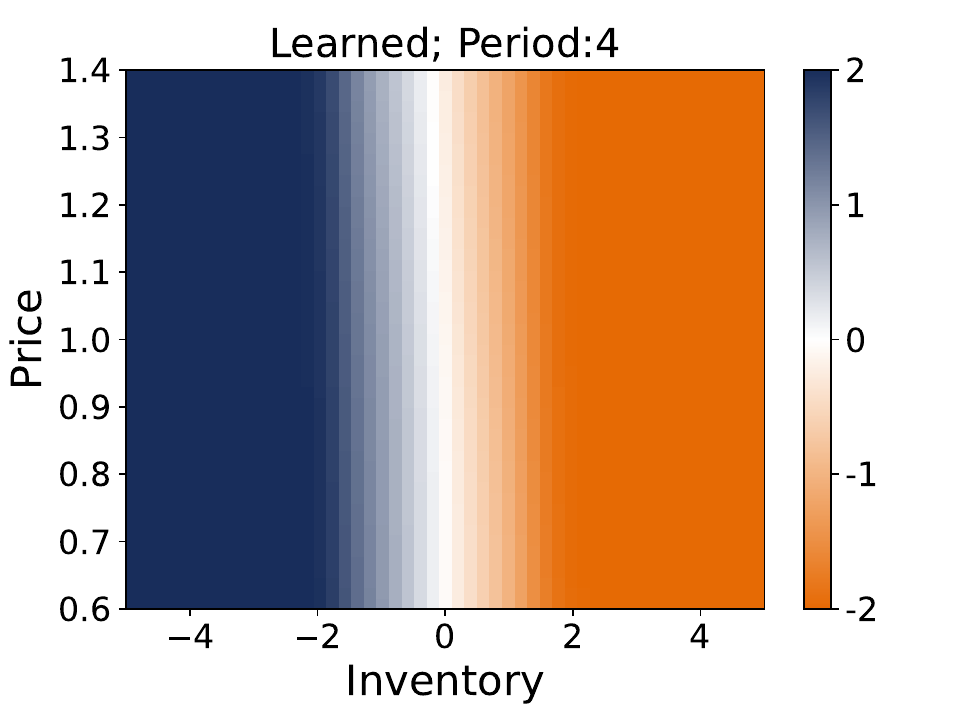}
    	\end{minipage}
    \subfloat[$\alpha = 0.8$]{
    	\hphantom{\begin{minipage}[c]{0.96\textwidth} \centering \includegraphics[width=0.90\textwidth]{figure-files/figures-stat-arbitrage/CVaR0.5/ANN_actions_CVaR0.5_4.pdf}
    	\end{minipage}}
    }
	\caption{Comparison of the learnt policy by the actor-critic algorithms with the elicitable (top row) and the nested simulation (bottom row) approaches as a function of time (from left to right) when optimising the dynamic $\CVaR_{\alpha}$ in the statistical arbitrage example.}
	\label{fig:stat-arbitrage}
\end{figure}
\newpage

\subsection{Portfolio Allocation Example}
\label{ssec:portfolio-allocation}

In this set of experiments, we highlight the flexibility of our novel approach on a portfolio allocation problem setting.
Suppose an agent faces a market with $I$ risky assets, and can allocate its wealth between these financial instruments during $\eplength=12$ periods over a one year horizon. First, we conduct two sets of experiments: for every risky asset denoted $i=1,\ldots,I$, the price dynamics $(\price_{\timeidx}^{(i)})_{\timeidx}$ are given by either one of the \SDE{}s
\begin{subequations}
\begin{align}
	\dee \price_{\timeidx}^{(i)} &= \mu^{(i)} \price_{\timeidx}^{(i)} \dee \timeidx + \sigma^{(i)} \price_{\timeidx}^{(i)} \dee W_{\timeidx}^{(i)}, \quad \text{or}
	\label{eq:price-GBM} \\
	\dee X_{\timeidx}^{(i)} &= - \kappa X_{\timeidx}^{(i)} \dee \timeidx + \sigma^{(i)} \dee W_{\timeidx}^{(i)} \quad \text{with} \quad \price_{\timeidx}^{(i)} = e^{X_{\timeidx}^{(i)} + \vartheta^{(i)}_{\timeidx}},
	\label{eq:price-meanrev}
\end{align}
\end{subequations}
where $\mu^{(i)}$ are the drifts, $\sigma^{(i)}$ the volatility scale parameters,
$\vartheta^{(i)}_{\timeidx}=\mu^{(i)} \timeidx - (\sigma^{(i)})^2 (1 - e^{-2\kappa \timeidx})/4\kappa$ the (time-dependent) mean-reversion levels, and $(W_{\timeidx}^{(i)})_{\timeidx}$ are $\PP$-Brownian motions.
One can show that the expectation of the price dynamics is identical with both \SDE{}s, i.e. $\EE[\price_{\timeidx}^{(i)}] = e^{\mu^{(i)} \timeidx}$ for all $i = 1,\ldots,I$, but asset prices mean-revert when using \cref{eq:price-meanrev}.
In order to make the market more realistic, we suppose that the financial instruments are correlated with each other: $(W_{\timeidx}^{(i)})_{\timeidx}$, $i=1,\ldots,I$ are $\PP$-Brownian motions with correlation $\rho$, i.e. $\dee [W^{(i)}, W^{(j)}]_{\timeidx} = \rho_{i,j} \dee \timeidx$.
In our experiments, we set the correlation between assets to $\rho_{i,j} = 0.2$.

The agent decides on the proportions of its wealth $\policy_{\timeidx}^{(i)}$, $i=1,\ldots,I$, to invest in the different assets at each period, which gives its portfolio allocation.
The agent's wealth is determined with the following \SDE{}:
\begin{equation*}
	\dee y_{\timeidx} = y_{\timeidx} \Bigg( \sum_{i=1}^{I} \policy_{\timeidx}^{(i)} \frac{\dee \price_{\timeidx}^{(i)}}{\price_{\timeidx}^{(i)}} \Bigg), \quad y_{0} = 1.
\end{equation*}

In our \RL{} problem context, the risk-aware agent wants to minimise the dynamic \CVaR{} of the profit and loss (\PnL{}) $\costfunc_{\timeidx} = y_{\timeidx} - y_{\timeidx+1}$, where actions are the investment proportions $\{\policy_{\timeidx}^{(i)}\}_{i=1,\ldots,I}$ at each period $\timeidx \in \periodspace$.
The states represent the information available to the agent before making its investment decisions.
One may assume that the agent observes several features, such as the asset price history, its current wealth, (estimated) market volatilities, etc.
Here, we focus on a few features: states are given by the time $\timeidx$ and the prices of the risky assets $\{\price_{\timeidx}^{(i)}\}_{i=1,\ldots,I}$\new{, where $I=3$}.

The policy is characterised by a $I$-dimensional multivariate Gaussian distribution.
More specifically, the outputs of an \ANN{} $\policy^{\policyparams}$ give the mean of the Gaussian distribution, and we perform a softmax mapping $f_{i}(x) = e^{x_i} / \sum_j e^{x_j}$ on realisations from the resulting distribution to enforce the constraint $\sum_i \policy_{\timeidx}^{(i)} = 1$.
We thus have a self-financing portfolio where we do not permit leveraging nor short-selling.
If short-selling is allowed in the market, one can, for instance, obtain the investment proportions for all assets except one (e.g. with a hyperbolic tangent output activation function to control the maximum leverage allowed) and then impose the $\sum_i \policy_{\timeidx}^{(i)} = 1$ constraint.

\begin{figure}[htbp]
    \centering
    \begin{minipage}[b]{0.60\textwidth}
        \subfloat[Learnt policies]{\label{fig:portfolio-GBM1-policy} \includegraphics[width=0.95\textwidth]{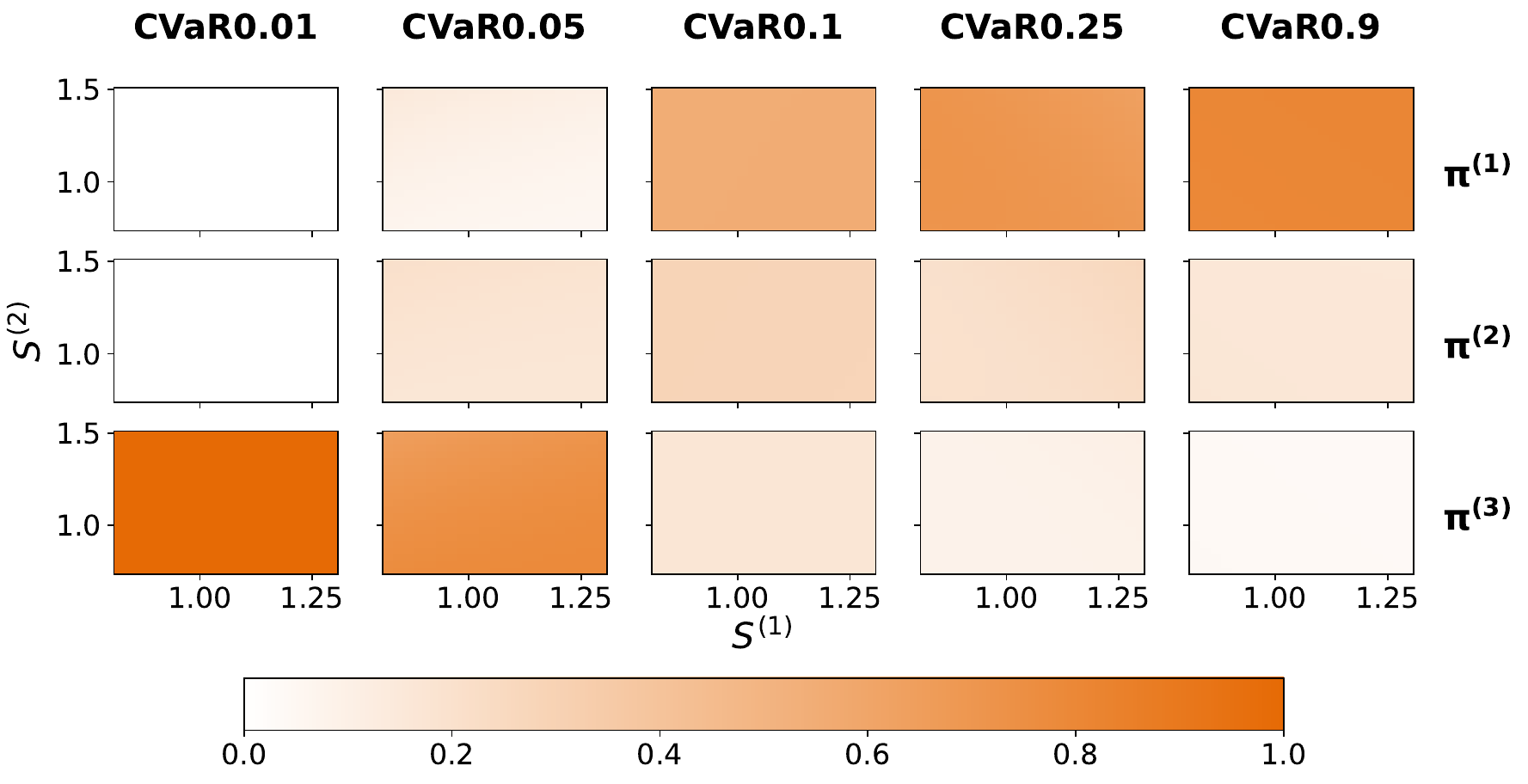}}
    \end{minipage}
    \hfill
    \begin{minipage}[b]{0.36\textwidth}
        \subfloat[Distributions of terminal \PnL{}]{\label{fig:portfolio-GBM1-PnL} \includegraphics[width=0.98\textwidth]{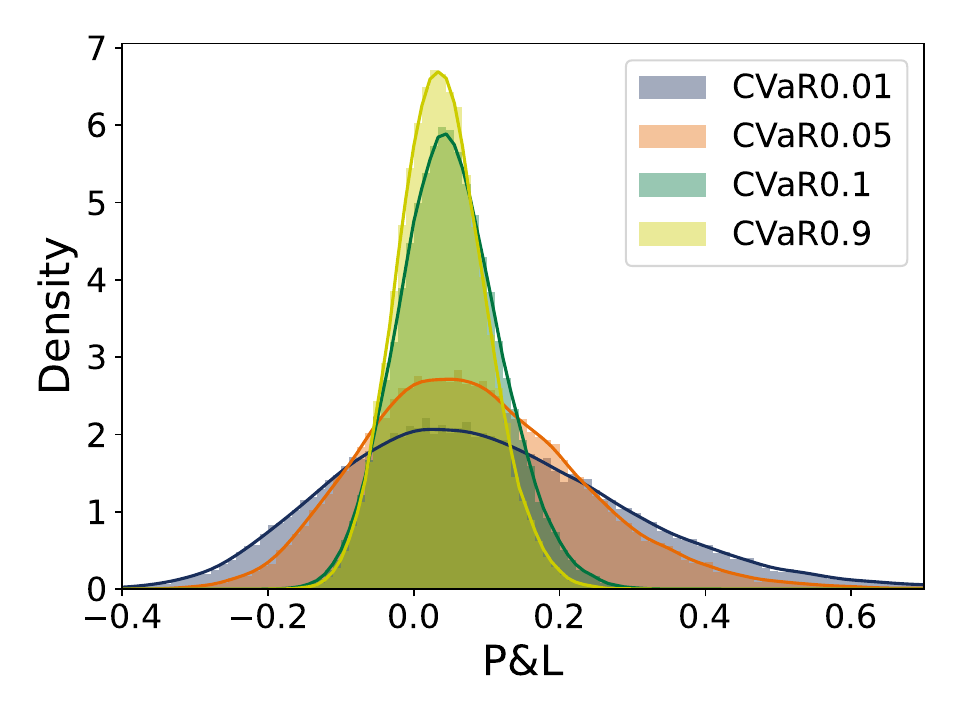}}
    \end{minipage}
	\caption{\PnL{} distributions when following learnt policies wrt dynamic $\CVaR_{\alpha}$, where price dynamics satisfy the \SDE{}s in \cref{eq:price-GBM}. Drifts and volatilities are respectively $\mu=[0.03; 0.06; 0.09]$ and $\sigma=[0.06; 0.12; 0.18]$.}
	\label{fig:portfolio-GBM1}
\end{figure}

When the price dynamics satisfy the \SDE{}s in \cref{eq:price-GBM} and all assets have identical Sharpe ratios, the learnt policy suggests to invest in the three different assets with an emphasis on the asset with the smallest volatility as we increase the risk-awareness of the agent, as illustrated in \cref{fig:portfolio-GBM1-policy}.
In this figure, the investment proportions $\policy^{(i)}$ are shown in each row, where darker colors correspond to more investment in that asset.
As the threshold of the dynamic \CVaR{} decreases, we recover the optimal risk-neutral strategy that fully invests in the asset with the best expected return.
We note that even when accounting for nominal amounts of risk, the agent pulls away from the riskier assets, which results in \PnL{} distributions in \cref{fig:portfolio-GBM1-PnL} with smaller variances.
\new{We report approximations of the dynamic \CVaR{}s at different thresholds under these optimal policies in \cref{sec:appendix-tables}.}
The \PnL{} distributions are estimated over 30,000 full episodes.
Also, when varying the Sharpe ratio of the assets in our experiments (not shown here for brevity), the risk-aware agent learns to allocate more wealth in risky assets with better Sharpe ratios -- it makes sense from an investor's perspective to diversify its portfolio with the best available assets.
We point out that for the specific dynamics described in \cref{eq:price-GBM}, returns are independent of the price and time, and thus we observe neither price nor time dependency in the learnt policies.

Next, consider the case where price dynamics satisfy the \SDE{}s in \cref{eq:price-meanrev}.
We expect our trained agent to change its behaviour according to asset prices due to the mean-reversion patterns.
Indeed, the learnt optimal policies suggest to allocate more wealth to assets for which the price is lower than its mean-reversion level, as illustrated in \cref{fig:portfolio-OU1} where the policy is a function of $\price^{(1)}$ (the $x$-axis), $\price^{(2)}$ (the $y$-axis) and $\price^{(3)}$ (the columns).
As noted in the previous scenario, increasing risk-awareness leads to investing more in financial instruments with lower volatilities, and diversifying the portfolio when prices are not favourable.
We also observe in \cref{fig:portfolio-OU1-PnL-noriskfree} that the agent's risk-sensitivity is translated to narrower terminal \PnL{} distributions and less variance in the \PnL{} during the episode.
Once again, it shows that our approach not only produces \emph{policies that reflect observed investor behaviours}, but also gives \emph{sensible \PnL{} distributions that agents can tune} to their risk tolerances.

\begin{figure}[htbp]
    \centering
    \begin{minipage}[b]{0.32\textwidth}
		\subfloat[$\alpha=0.05$]{\label{fig:portfolio-OU1-CVaR0.05}
		\includegraphics[width=0.97\textwidth]{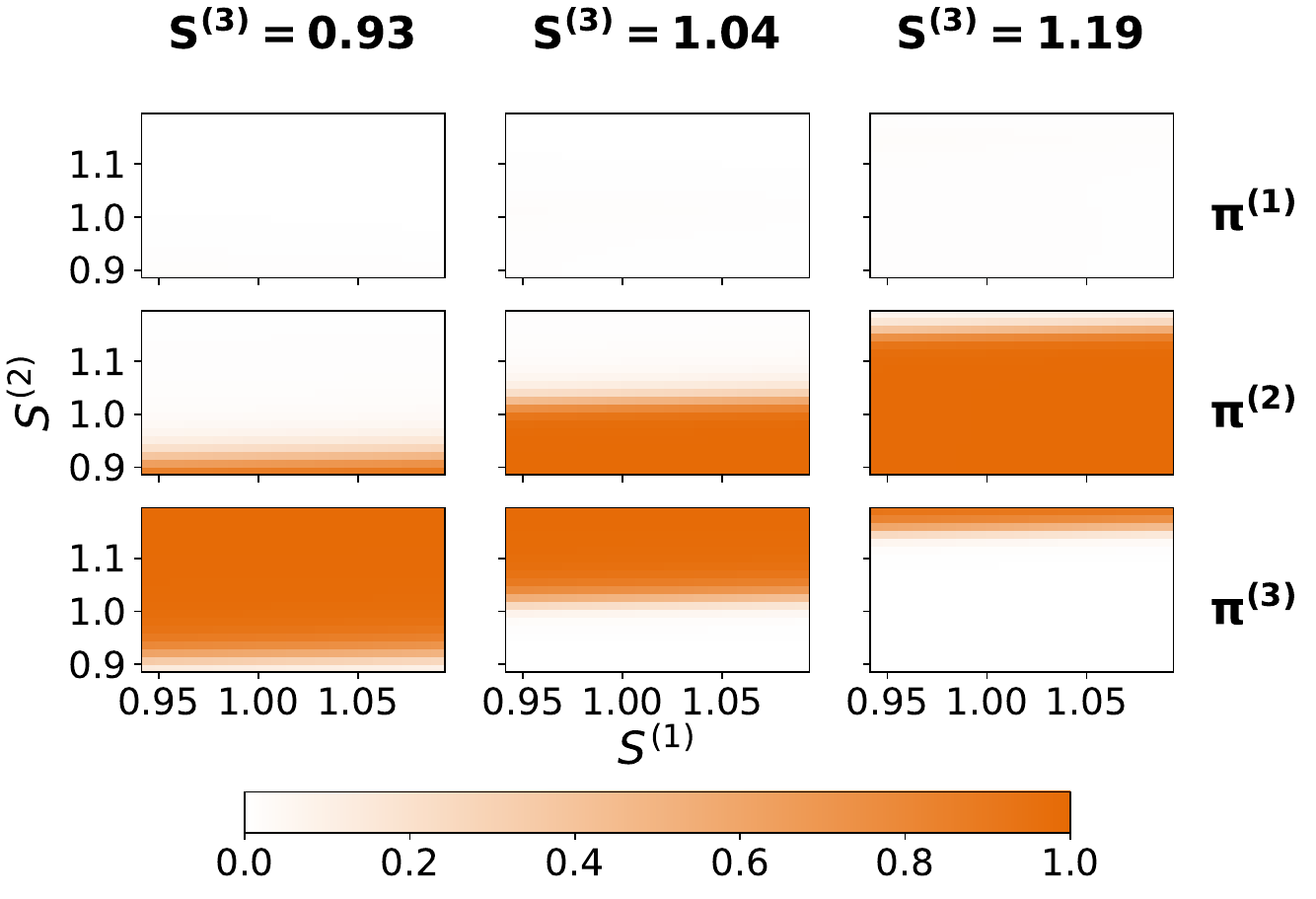}}
	\end{minipage}
	\begin{minipage}[b]{0.32\textwidth}
	    \subfloat[$\alpha=0.4$]{\label{fig:portfolio-OU1-CVaR0.4}
		\includegraphics[width=0.97\textwidth]{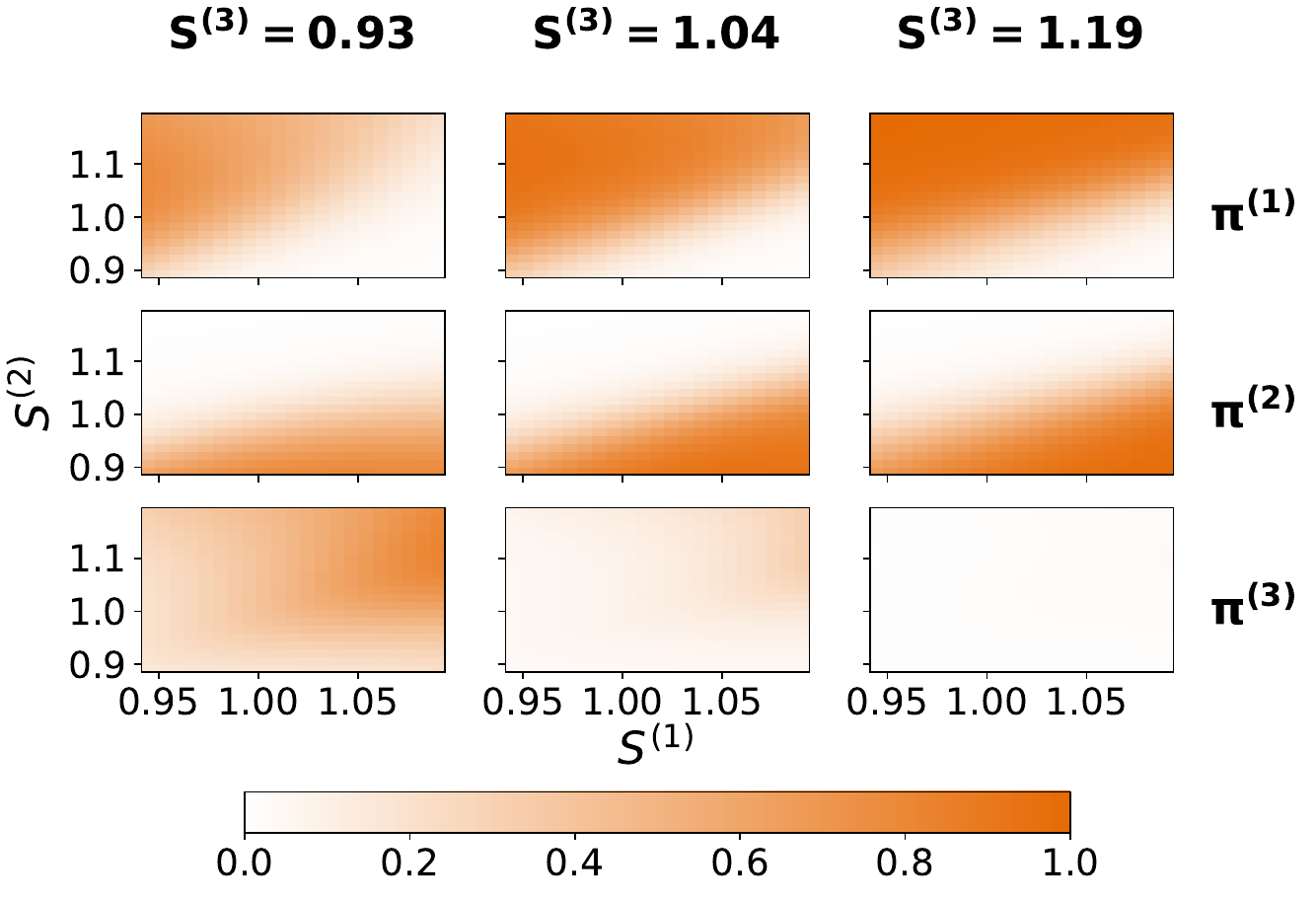}}
	\end{minipage}
	\begin{minipage}[b]{0.32\textwidth}
		\subfloat[$\alpha=0.7$]{\label{fig:portfolio-OU1-CVaR0.7}
		\includegraphics[width=0.97\textwidth]{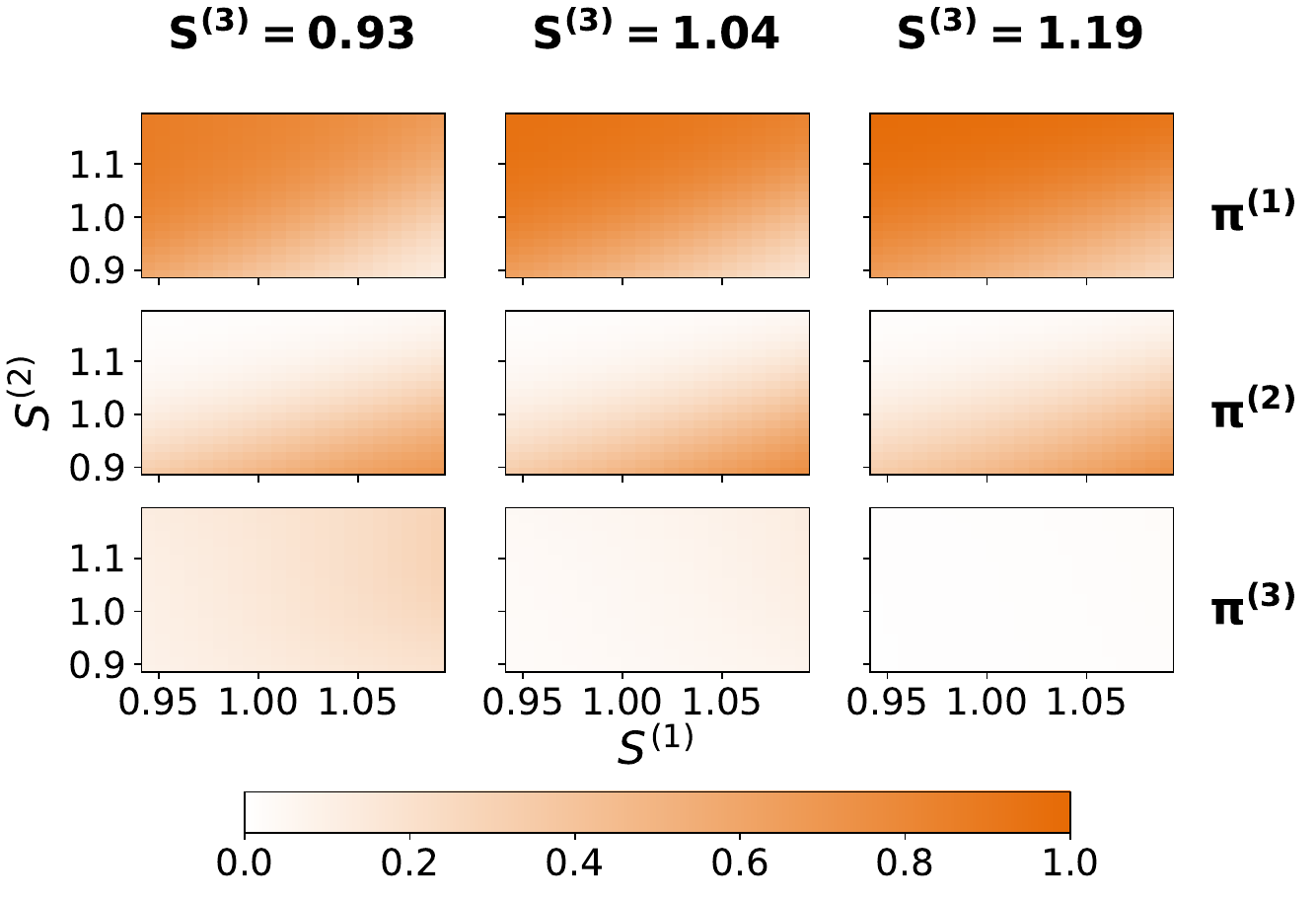}}
	\end{minipage}
	\caption{Learnt optimal policies at period $\timeidx=0$ wrt dynamic $\CVaR_{\alpha}$, where price dynamics satisfy the \SDE{}s in \cref{eq:price-meanrev} with $\kappa=2$. Drifts and volatilities are respectively $\mu=[0.03; 0.06; 0.09]$ and $\sigma=[0.06; 0.12; 0.18]$.}
	\label{fig:portfolio-OU1}
\end{figure}
\begin{figure}[htbp]
    \centering
	\begin{minipage}[b]{0.48\textwidth}
		\subfloat[Absence of a risk-free asset]{\label{fig:portfolio-OU1-PnL-noriskfree}
		\includegraphics[width=0.98\textwidth]{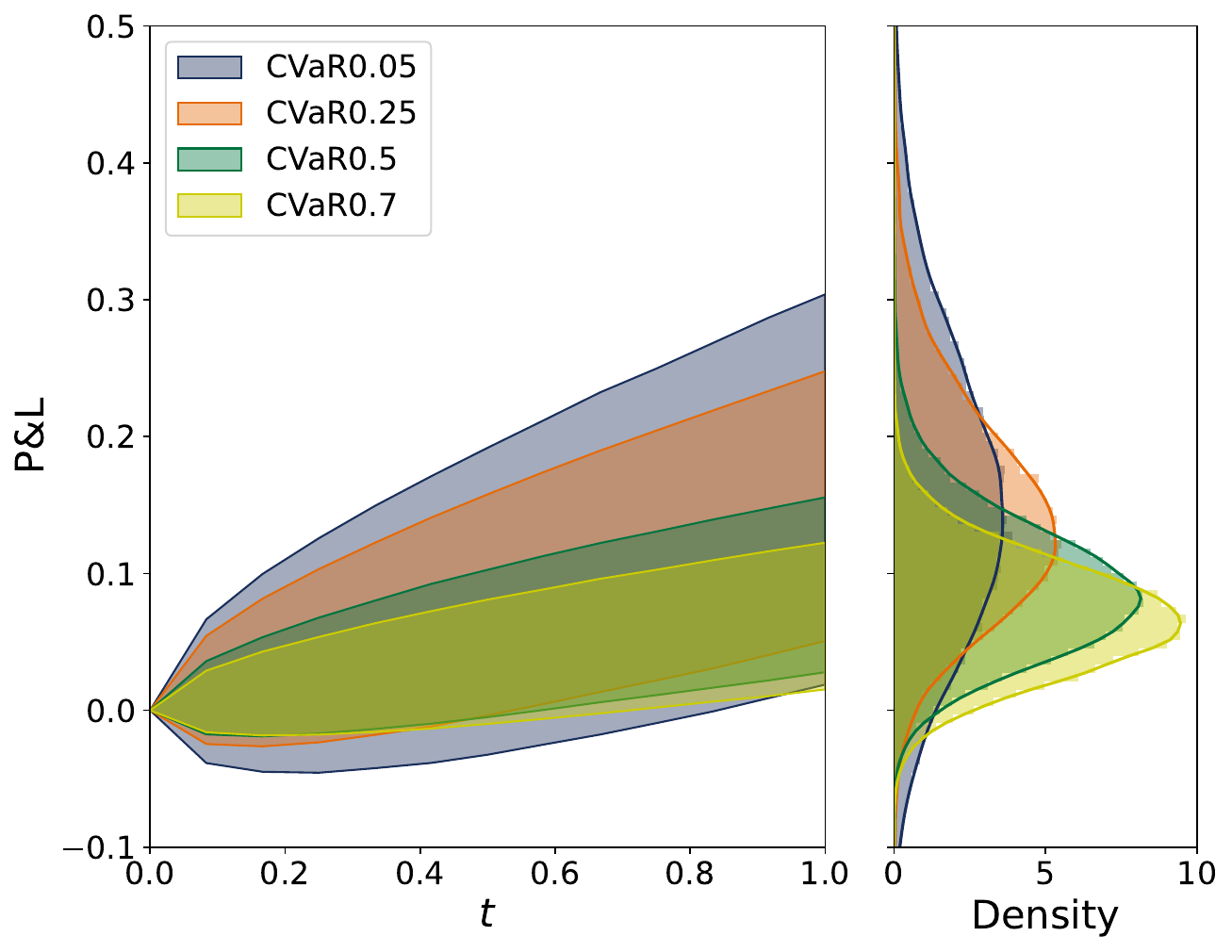}}
	\end{minipage}
	\hfill
	\begin{minipage}[b]{0.48\textwidth}
		\subfloat[Presence of a risk-free asset]{\label{fig:portfolio-OU1-PnL-riskfree}
		\includegraphics[width=0.98\textwidth]{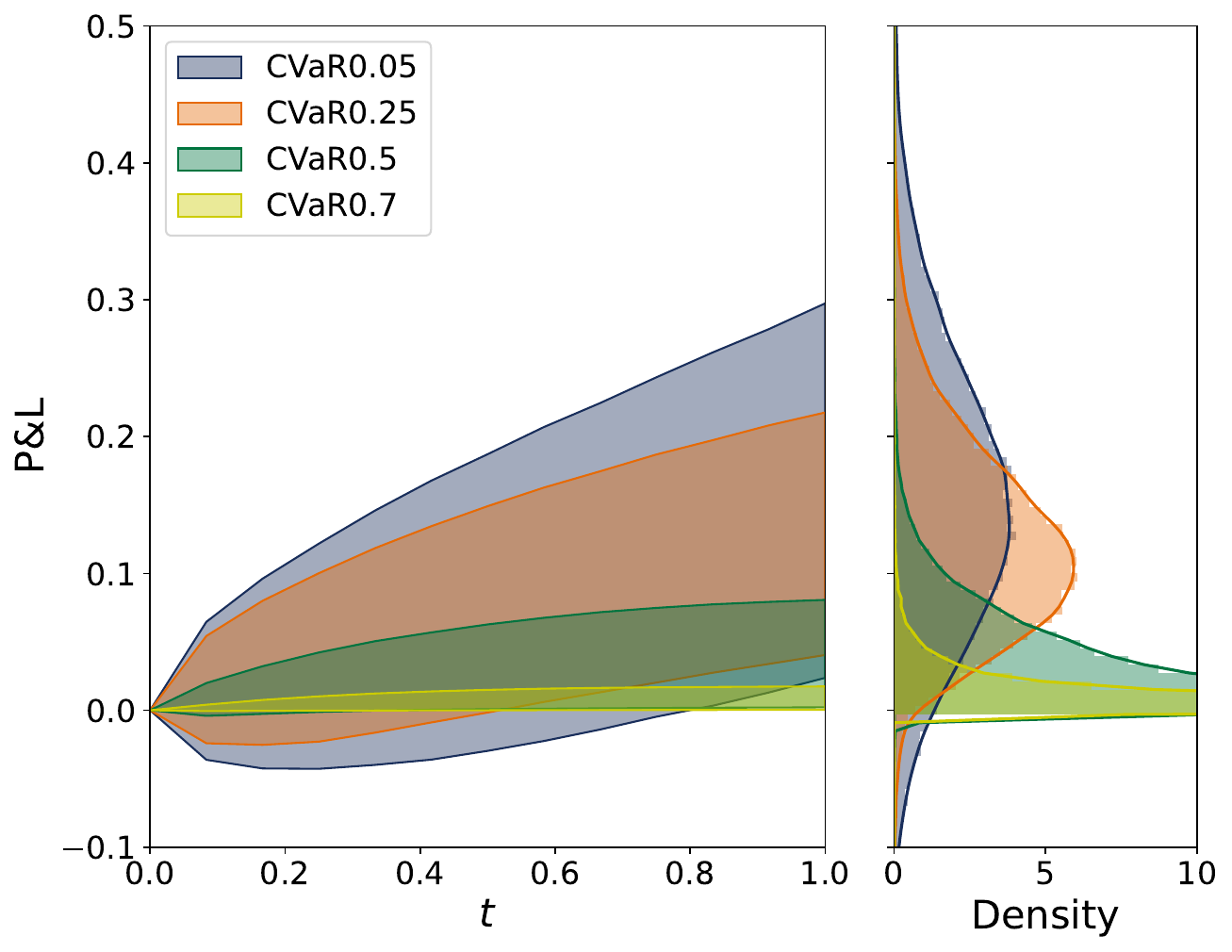}}
	\end{minipage}
	\caption{\PnL{} distributions when following optimal policies wrt dynamic $\CVaR_{\alpha}$, where price dynamics satisfy the \SDE{}s in \cref{eq:price-meanrev} with $\kappa=2$. Drifts and volatilities are respectively $\mu=[0.03; 0.06; 0.09]$ and $\sigma=[0.06; 0.12; 0.18]$.}
	\label{fig:portfolio-OU1-PnL}
\end{figure}

If one takes a closer look at \cref{fig:portfolio-OU1-PnL-noriskfree}, it seems as if increasing the risk-sensitivity does not mitigate the lower tail of the terminal \PnL{} distribution.
Indeed, this is due to the formulation of the problem -- the agent must allocate its entire wealth in a risky asset at each period, even when all prices are higher than their mean-reversion levels, for instance.
Suppose that we include in the market a risk-free asset, and, without loss of generality, that the interest rate is zero.
We train the risk-aware agents with this new market and illustrate the evolution of the \PnL{} throughout the episode in \cref{fig:portfolio-OU1-PnL-riskfree} for comparison purposes.
We observe that increasing the threshold of the dynamic \CVaR{} then indeed reduces both tails of the \PnL{} distributions, and similar $\alpha$'s lead to significantly different \PnL{} distributions when including a risk-free asset.
This scenario warns the reader that care must be taken in practice when formalising the problem and choosing the parameters of the dynamic risk.

Finally, we modify the price dynamics by considering a co-integrated system, as such models correspond more adequately to dynamics one would observe in practice.
To do so, we estimate a vector error correction model (VECM), essentially a co-integration model, using daily data from \new{$I=8$} different stocks listed on the NASDAQ exchange\footnote{We consider the following assets: AAL, AMZN, CCL, FB, IBM, INTC, LYFT, OXY.} between September 31, 2020 and December 31, 2021 inclusively.
The resulting estimated model, with two co-integration factors and no lag differences (both selected using the BIC criterion), is used as a simulation engine to generate price paths $\{ \price_{\timeidx}^{(i)} \}_{\timeidx}$, $i=1, \ldots, I$ during $\eplength=24$ periods over a one year horizon -- we illustrate some simulated sample paths in \cref{fig:portfolio-VECM-paths}.
We give a brief explanation of VECMs and the parameter estimates for this dataset in \cref{sec:appendix-vecm}.

\begin{figure}[htbp]
    \centering
    \begin{minipage}[b]{0.80\textwidth}
		\centering
		\includegraphics[width=0.98\textwidth]{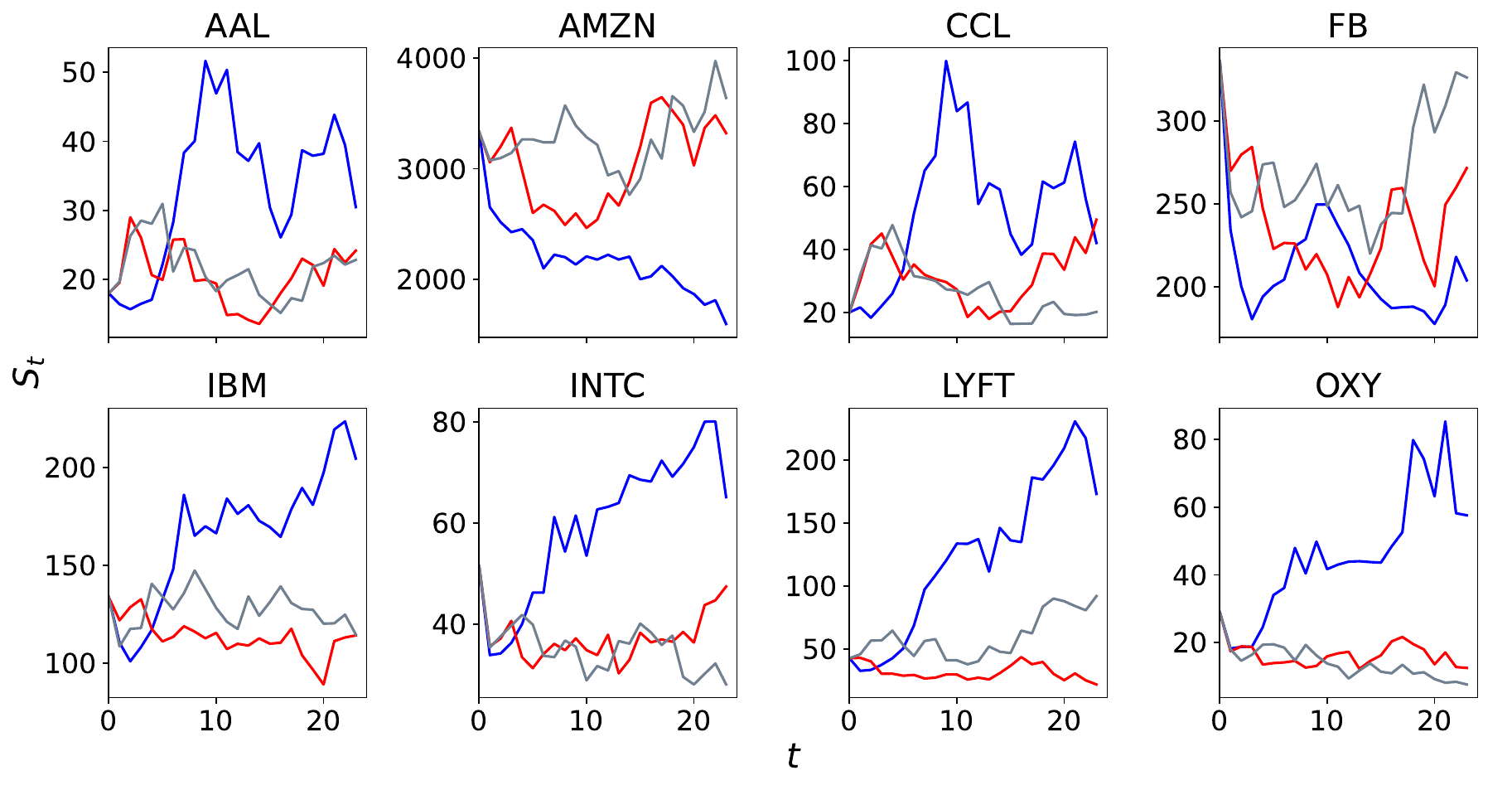}
	\end{minipage}
	\caption{Three simulated sample paths from the estimated VECM.}
	\label{fig:portfolio-VECM-paths}
\end{figure}

Similarly to the other experiments, \cref{fig:portfolio-VECM-PnL} shows the evolution of \PnL{} through time when following the learnt optimal policies.
With such a high-dimensional state space, it becomes difficult to visualise the learnt policy and fully understand the optimal strategy at every possible state of the environment.
Nonetheless, we compute the estimated investment proportions (over 50,000 episodes) in the different assets on average at every time period across all simulations used to generate the \PnL{} distributions, and show the proportions in \cref{fig:portfolio-VECM-CVaR_combined}.
When the threshold of the dynamic \CVaR{} is low, our agent seems to prefer more volatile assets that provide greater returns on average, i.e. AAL, CCL, LYFT and OXY based on the parameter estimates of the VECM for our dataset.
On the other hand, when increasing its risk-sensitivity, the agent allocates its wealth in assets that are more stable over that time frame, e.g. AMZN, FB, IBM and INTC.
In all the tested scenarios, we observe diverse portfolios with multiple nonzero investments, which emphasise that diversification is beneficial.

\begin{figure}[htbp]
    \centering
    \begin{minipage}[b]{0.85\textwidth}
		\subfloat[Learnt policies]{\label{fig:portfolio-VECM-CVaR_combined}
		\includegraphics[width=0.97\textwidth]{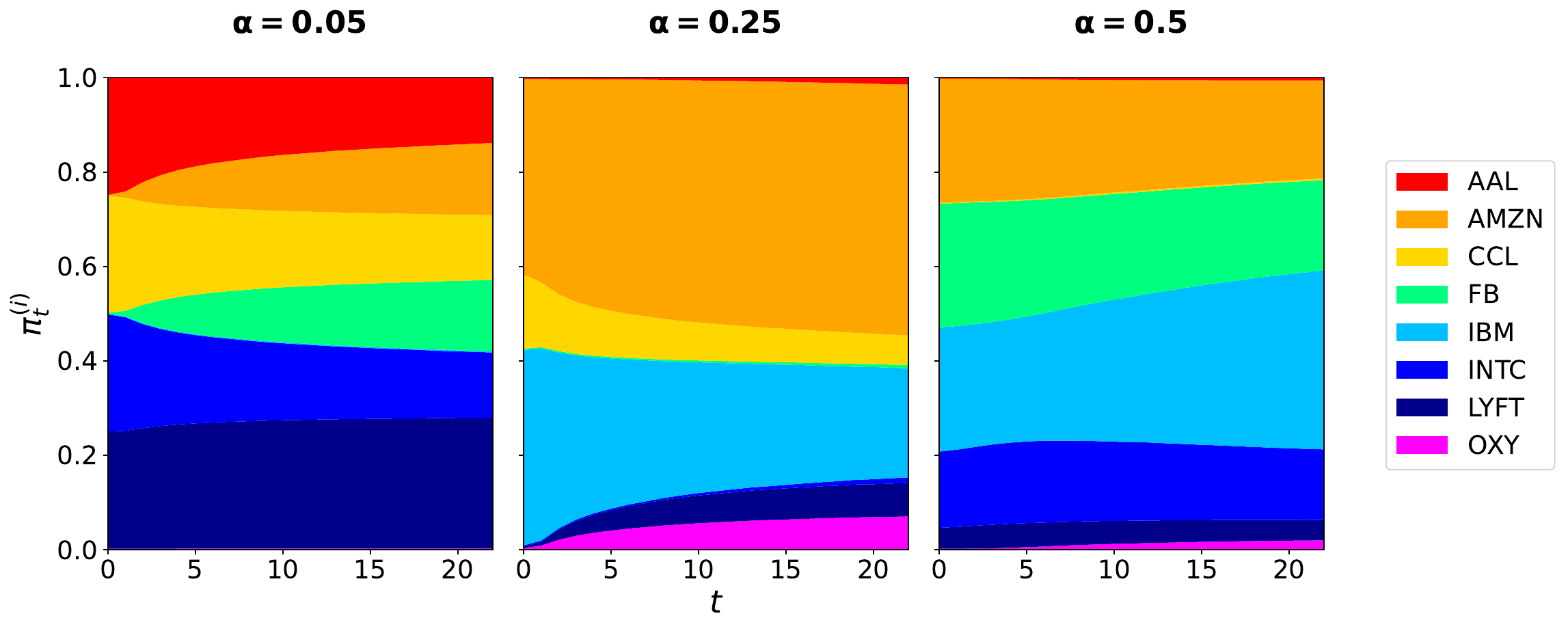}}
	\end{minipage}
	\hfill
	\begin{minipage}[b]{0.60\textwidth}
		\subfloat[Distribution of \PnL{}]{\label{fig:portfolio-VECM-PnL}
		\includegraphics[width=0.97\textwidth]{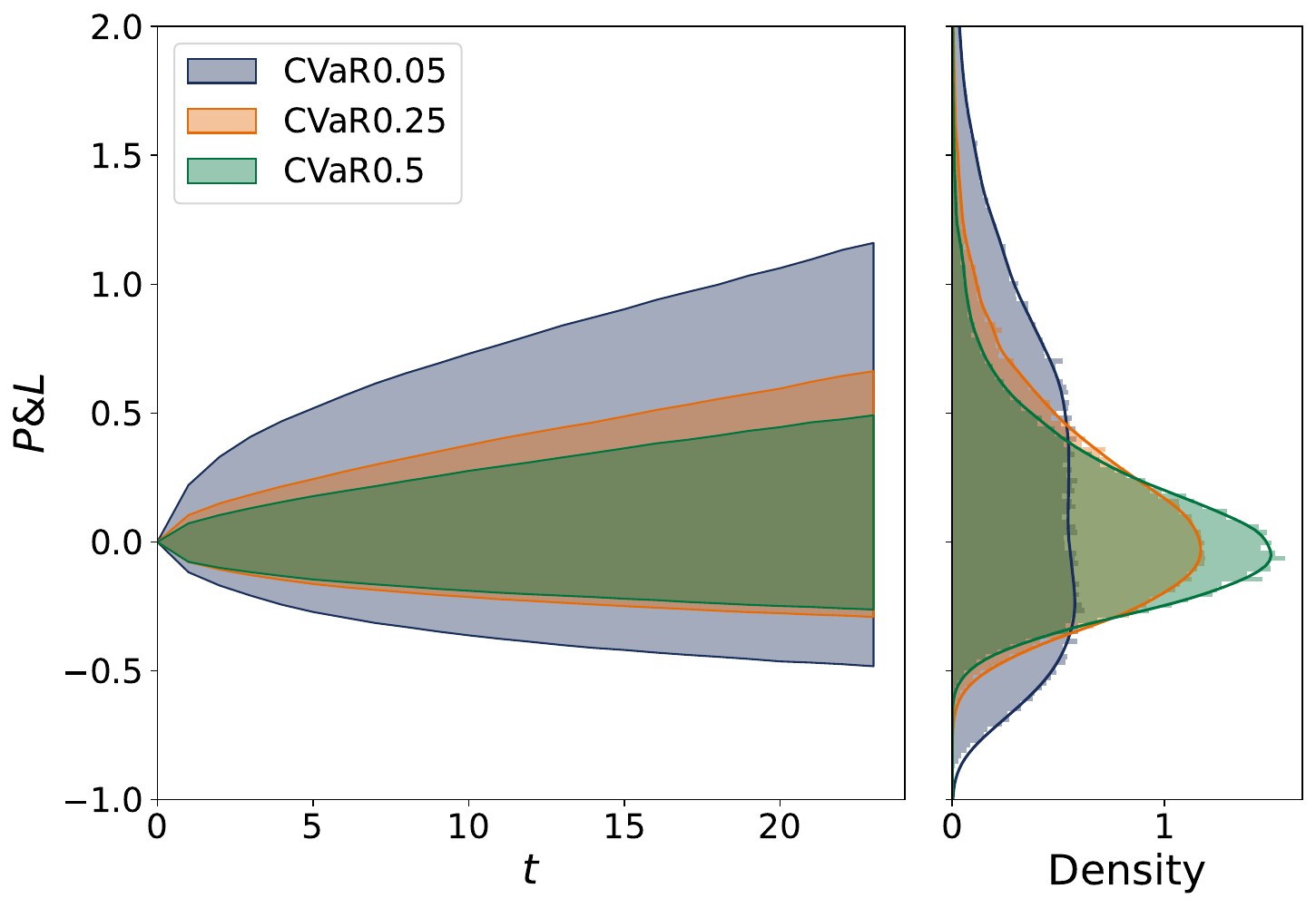}}
	\end{minipage}
	\caption{\PnL{} distributions when following optimal policies wrt dynamic $\CVaR_{\alpha}$, where price dynamics are simulated from the co-integrated system.}
	\label{fig:portfolio-VECM}
\end{figure}

\subsection{Other Examples}
\label{ssec:other-examples}

In this paper, we apply our proposed methodology on statistical arbitrage and portfolio optimisation problems.
Risk-awareness plays a central role in various mathematical finance applications, which clearly motivates the need to account for uncertainty in those situations.
Indeed, one may apply our algorithm in other contexts, to solve sequential decision making problems where the objective involves a dynamic risk measure of the costs induced by a certain policy -- e.g., in credit risk, visual recognition, video games, robot control and trading problems, among others.

Next, we describe some settings for which a risk-aware \RL{} algorithm is relevant or more appropriate than the usual risk-neutral approaches.
This is by no means exhaustive, but rather an illustration of \RL{} problems that could benefit from risk-sensitive behaviours.

An important problem in mathematical finance consists of hedging and pricing contingent claims.
Suppose we have a financial market of multiple risky securities and risk-free bonds following some asset dynamics, and all assets can be traded over a time horizon with trading frictions.
The agent aims to determine either: (i) the price it should charge at each period for selling a certain (path-dependent) contingent claim; or (ii) the strategy it should adopt for hedging a certain (path-dependent) contingent claim.
The observable features that dictates the agent's decisions may include, but are not restricted to, the market volatility, asset price history, or inventories of the different assets held by the agent.
It is obvious why applying risk-aware \RL{} methods on this class of problems is sensible -- indeed, the agent wishes to protect itself from high-cost outcomes in the market.
Our proposed approach also allows the agent to tweak the threshold of the dynamic \CVaR{} to fulfil its own tolerance to risk.
We note that the setup described above is applicable to other sequential decision making problems in finance, e.g. in market making, entry-exit problems, goal-based portfolio management, and placement of cryptocurrency limit orders.

Uncertainty in the environment also affects performances of agents in domains outside of finance, for instance in robot control.
Consider an autonomous rover exploring an unknown terrain.
Its objective is to reach a final destination while avoiding different obstacles along the way that could cause substantial damages.
At each period, the rover evaluates its position and environment conditions and chooses to move in a certain direction, but its movements are altered by meteorological/topographical events, e.g. facing strong winds, slipping on sand, or crossing shallow water.
The randomness may be modeled through the randomised policy with an atypical distribution and/or directly within the transition probability distributions.
The costs observed by the rover may be influenced by the travelled distance, the consumption of energy and the severity of any impairments.
In this scenario, a risk-neutral rover would typically take the most direct route to reach the destination as quickly as possible, ignoring the randomness in its movements and potentially causing expensive damages.
Instead, a more conservative behaviour would account for uncertainty by keeping distance from obstacles to avoid inadvertent collisions, even if it increases the costs in terms of travelled distance.
A similar setting may be used in video games applications, where the agent tries to accomplish different missions.

\begingroup
\allowdisplaybreaks

\section{Generalisation to Spectral Risk Measures}
\label{sec:generalization-spectral}

Our algorithm is described in detail for the dynamic \CVaR{}; however, it is interesting and important to consider generalising the framework to dynamic risk measures which are $k$-elicitable.
This section presents the modifications to our approach to incorporate any dynamic spectral risk measure having a spectrum with finite support.

In what follows, we consider problems of the form \cref{eq:optim-problem1}, and assume that the one-step conditional risk measures are \new{one-step conditional} spectral risk measures $\riskmeas^{\spectrum}$, where $\spectrum$ is given by
\begin{equation*}
    \spectrum = \sum_{m=1}^{k-1} p_m \delta_{\alpha_m},
\end{equation*}
with $p_{m} \in (0,1]$, $\sum_{m=1}^{k-1} p_m = 1$, and $0 < \alpha_{1} < \alpha_{2} < \cdots < \alpha_{k-1} < 1$.
In this setting, the value function in \cref{eq:value-func0-1,eq:value-func0-2} is written as 
\begin{subequations}
\begin{align}
	\newmath{\valuefunc_{\eplength}(\state;\policyparams)} &=
    \sum_{m=1}^{k-1} p_m \; \newmath{\CVaR_{\alpha_m} \Big(\costfunc_{\eplength}^{\policyparams} \Bigm|\state_{\eplength}=\state \Big)}, \quad \text{and}
    \label{eq:lastV-spectral} 
    \\
	\valuefunc_{\timeidx}(\state;\policyparams) &=
    \sum_{m=1}^{k-1} p_m \; \CVaR_{\alpha_m} \Big(\costfunc_{\timeidx}^{\policyparams} +
    \valuefunc_{\timeidx+1}(\state_{\timeidx+1}^{\policyparams};\policyparams)
    \Bigm|\state_{\timeidx}=\state \Big). \label{eq:otherV-spectral}
\end{align}
\end{subequations}

Similarly to the \CVaR{} case, we define an \ANN{} $\policy^{\policyparams}: \statespace \rightarrow \Pp(\actionspace)$ for the agent's policy.
We decompose the value function when in state $\state\in \statespace$ at period $\timeidx \in \periodspace$ and behaving according to $\policy^{\policyparams}$ as
\begin{equation}
    \valuefunc_{\timeidx}(\state;\policyparams) = H_{k,\timeidx}(\state;\policyparams) + \sum_{m=1}^{k-1} p_m \sum_{l=1}^{m} H_{l,\timeidx}(\state;\policyparams), \label{eq:V-spectral}
\end{equation}
where we define for $m=1,\ldots,k-1$
\begin{subequations}
\begin{align}
    \begin{split}
    H_{m,\timeidx}(\state;\policyparams) &:=
    \VaR_{\alpha_m} \Big(\costfunc_{\timeidx}^{\policyparams} +
    \valuefunc_{\timeidx+1}(\state_{\timeidx+1}^{\policyparams};\policyparams)
    \Bigm|\state_{\timeidx}=\state \Big) - \sum_{l=1}^{m-1} H_{l,\timeidx}(\state;\policyparams), \quad \text{and}
    \label{eq:Hm-spectral}
    \end{split}
    \\
    \begin{split}
    H_{k,\timeidx}(\state;\policyparams) &:= \riskmeas^{\spectrum} \Big(\costfunc_{\timeidx}^{\policyparams} +
    \valuefunc_{\timeidx+1}(\state_{\timeidx+1}^{\policyparams};\policyparams)
    \Bigm|\state_{\timeidx}=\state \Big) - \sum_{m=1}^{k-1} p_m \sum_{l=1}^{m} H_{l,\timeidx}(\state;\policyparams).
    \label{eq:Hk-spectral}
    \end{split}
\end{align}
\end{subequations}
We further denote the estimation of the value function in \cref{eq:V-spectral} by
\begin{equation}
    \valuefunc_{\timeidx}^{\valueparams}(\state;\policyparams) = H_{k,\timeidx}^{\psi_k}(\state;\policyparams) + \sum_{m=1}^{k-1} p_m \sum_{l=1}^{m} H_{l,\timeidx}^{\psi_l}(\state;\policyparams),
\end{equation}
where $\valueparams=\{ \psi_1,\ldots,\psi_k\}$ and $H_{1,\timeidx}^{\psi_1}(\cdot; \policyparams): \statespace \rightarrow \Reals$, $H_{2,\timeidx}^{\psi_2}(\cdot; \policyparams),\ldots,H_{k,\timeidx}^{\psi_k}(\cdot; \policyparams): \statespace \rightarrow \PosReals$ estimate respectively $H_{1,\timeidx},H_{2,\timeidx},\ldots,H_{k,\timeidx}$.
Again, note that we characterise the difference between \VaR{}s to ensure $\VaR_{\alpha_{m}}$ is greater than $\VaR_{\alpha_{m-1}}$ for every $m=2,\ldots,k-1$, and we constrain the \ANN{} $H_{k,\timeidx}^{\psi_k}$ to nonnegative outputs due to the inequality
\begin{equation*}
    \riskmeas^{\spectrum}(\rv) = \sum_{m=1}^{k-1} p_m \CVaR_{\alpha_m}(\rv) \geq \sum_{m=1}^{k-1} p_m \VaR_{\alpha_m}(\rv).
\end{equation*}

To fix ideas, suppose we are interested in the \emph{range value-at-risk} \cite{cont2010robustness}, where the spectrum is given by
\begin{equation*}
    \spectrum = \frac{1-\alpha}{\beta-\alpha} \delta_{\alpha} + \frac{\beta-1}{\beta-\alpha} \delta_{\beta}, \quad 0 \leq \alpha < \beta \leq 1.
\end{equation*}
We then have $H_{1,\timeidx}^{\psi_1}$ to approximate $\VaR_{\alpha}$, $H_{2,\timeidx}^{\psi_2}$ for the difference between $\VaR_{\beta}$ and $\VaR_{\alpha}$, and $H_{3,\timeidx}^{\psi_3}$ for the difference between the value function and $(1-\alpha)/(\beta-\alpha) \VaR_{\alpha} + (\beta-1)/(\beta-\alpha) \VaR_{\beta}$.

\subsection{Value Function Estimation}

The estimation of the value function described in \cref{ssec:estimate-V} can easily be extended to any $k$-elicitable spectral risk measure by utilising \cref{thm:elicitable-spectral} instead of \cref{thm:elicitable-cvar}.
Similarly to \cref{eq:loss-critic}, the loss function $\Ll^{\valueparams}$ we minimise is
{\small{
\begin{equation}
\begin{split}
    & \newmath{\sum_{\timeidx \in \periodspace \setminus \{\eplength\}}} \sum_{\batchidx=1}^{\Nbatchs} \Bigg[ \score \bigg(
	\underbrace{
	\vphantom{\sum_{m=1}^{k}} H_{1,\timeidx}^{\psi_1} \Big( \state_{\timeidx}^{(\batchidx)}; \policyparams \Big)
	}_{\VaR_{\alpha_{1}}(\cdot | \state_{\timeidx}^{(\batchidx)})}
	; \cdots; \
	\underbrace{
	\sum_{m=1}^{k-1} H_{m,\timeidx}^{\psi_m} \Big( \state_{\timeidx}^{(\batchidx)}; \policyparams \Big)
	}_{\VaR_{\alpha_{k-1}}(\cdot | \state_{\timeidx}^{(\batchidx)})}
	; \
	\underbrace{
	\vphantom{ \sum_{m=1}^{k} }
	\valuefunc_{\timeidx}^{\valueparams} \Big( \state_{\timeidx}^{(\batchidx)}; \policyparams \Big)
	}_{\riskmeas^{\spectrum}(\cdot | \state_{\timeidx}^{(\batchidx)})}
	; \
	\underbrace{
	\vphantom{ \sum_{m=1}^{k} }
	\costfunc_{\timeidx}^{(\batchidx)} + 
	\valuefunc_{\timeidx+1}^{\tilde{\valueparams}} \Big( \state_{\timeidx+1}^{(\batchidx)}; \policyparams \Big)
	}_{\text{running risk-to-go}}
	\bigg)
	\Bigg] \\
	&\qquad + \newmath{\sum_{\batchidx=1}^{\Nbatchs} \Bigg[ \score \bigg(
	\underbrace{
	\vphantom{\sum_{m=1}^{k}} H_{1,\eplength}^{\psi_1} \Big( \state_{\eplength}^{(\batchidx)}; \policyparams \Big)
	}_{\VaR_{\alpha_{1}}(\cdot | \state_{\eplength}^{(\batchidx)})}
	; \cdots; \
	\underbrace{
	\sum_{m=1}^{k-1} H_{m,\eplength}^{\psi_m} \Big( \state_{\eplength}^{(\batchidx)}; \policyparams \Big)
	}_{\VaR_{\alpha_{k-1}}(\cdot | \state_{\eplength}^{(\batchidx)})}
	; \
	\underbrace{
	\vphantom{ \sum_{m=1}^{k} }
	\valuefunc_{\eplength}^{\valueparams} \Big( \state_{\eplength}^{(\batchidx)}; \policyparams \Big)
	}_{\riskmeas^{\spectrum}(\cdot | \state_{\eplength}^{(\batchidx)})}
	; \
	\underbrace{
	\vphantom{ \sum_{m=1}^{k} }
	\costfunc_{\eplength}^{(\batchidx)}
	}_{\text{risk-to-go}}
	\bigg)
	\Bigg]},
\end{split} \tag{L3} \label{eq:loss-func-dynamic-spectral}
\end{equation}
}}where $\score$ is given in \cref{eq:scoring-dynamic-spectral} with some characterisation $G_1,\ldots,G_{k-1}$ satisfying \cref{eq:scoring-dynamic-spectral-condition}.
For specific choices of strictly consistent scoring functions for the range \VaR{}, we refer the reader to \cite{fissler2021elicitability}.

The following result states that we can approximate the value function to an arbitrary accuracy using our framework devised here.
\new{\begin{theorem}
    Suppose $\policy^{\policyparams}$ is a fixed policy, its value function $\valuefunc_{\timeidx}(\state;\policyparams)$ is given in \cref{eq:lastV-spectral,eq:otherV-spectral}, and the decomposition in terms of $H_{1,\timeidx}^{\psi_1},\ldots,H_{k,\timeidx}^{\psi_k}$ is given in \cref{eq:Hm-spectral,eq:Hk-spectral}.
    Then for any $\varepsilon^{*}_1,\ldots,\varepsilon^{*}_{k} > 0$, there exist \ANN{}s denoted $H_{1,\timeidx}^{\psi_1},\ldots,H_{k,\timeidx}^{\psi_k}$ such that for any $\timeidx \in \periodspace$, we have
    \begin{equation}
        \esssup_{\state \in \statespace} \; \Big\Vert H_{m,\eplength}(\state;\policyparams) - H_{m,\eplength}^{\psi_{m}}(\state;\policyparams) \Big\Vert < \varepsilon^{*}_{m}, \quad  \forall m = 1,\ldots,k.
    \end{equation}
    \label{thm:univ-approx-H-spectral}
\end{theorem}}

\begin{proof}
The proof is similar to the proof of \cref{thm:univ-approx-H}.
First, note that spectral risk measures satisfy the monotonicity and translation invariance properties.
Therefore, using the same argument as the one in the proof of \cref{thm:univ-approx-H}, any linear combination of $\VaR_{\alpha_m}$ and $\riskmeas^{\spectrum}$ is absolutely continuous a.s. for any threshold $\alpha_m \in (0,1)$, which allows us to use the universal approximation theorem result.

We show that \ANN{}s $H_{1,\timeidx}^{\psi_1},\ldots,H_{k,\timeidx}^{\psi_k}$ approximate the mappings $H_{1,\timeidx},\ldots,H_{k,\timeidx}$ by induction.
For the last period \new{$\eplength$}, using the universal approximation theorem on \cref{eq:Hm-spectral,eq:Hk-spectral}, we have that for any $\newmath{\varepsilon_{m,\eplength}} > 0$, there exists \ANN{}s \new{$H_{m,\eplength}^{\psi_{m,\eplength}}$} such that
\begin{align}
    \esssup_{\state \in \statespace} \; \Big\Vert \newmath{ H_{m,\eplength}(\state;\policyparams) - H_{m,\eplength}^{\psi_{m,\eplength}}(\state;\policyparams) } \Big\Vert &< \newmath{\varepsilon_{m,\eplength}}, \quad  \forall m = 1,\ldots,k. \label{eq:proof-approx-lastHm}
\end{align}

For the induction step, we show that $H_{1,\timeidx}^{\psi_1},\ldots,H_{k,\timeidx}^{\psi_k}$ approximate $H_{1,\timeidx},\ldots,H_{k,\timeidx}$ at the period $\timeidx$ as long as they adequately approximate them at the periods $\timeidx+1$ up to \new{$\eplength$} inclusively.
Assume that for any $\varepsilon_{1,\tau},\ldots,\varepsilon_{k,\tau} > 0$, there exist \ANN{}s $H_{1,\tau}^{\psi_{1,\tau}},\ldots,H_{k,\tau}^{\psi_{k,\tau}}$ such that
\begin{align}
    \esssup_{\state \in \statespace} \; \Big\Vert H_{m,\tau}(\state;\policyparams) - H_{m,\tau}^{\psi_{m,\tau}}(\state;\policyparams) \Big\Vert &< \varepsilon_{m,\tau}, \quad  \forall m = 1,\ldots,k, \label{eq:proof-approx-lastHm-induction}
\end{align}
for all \new{$\tau\in\{\timeidx+1,\ldots,\eplength\}$}.
To avoid a cumbersome notation, we define the following random variables as
\begin{subequations}
\begin{align*}
    \rvdum_{\timeidx} &:= \costfunc_{\timeidx}^{\policyparams} +
    H_{k,\timeidx+1}(\state_{\timeidx+1}^{\policyparams};\policyparams) + \sum_{m'=1}^{k-1} p_{m'} \sum_{l=1}^{m'} H_{l,\timeidx+1}(\state_{\timeidx+1}^{\policyparams};\policyparams), \quad \text{and}
    \\
    \rvdum^{\psi}_{\timeidx} &:= \costfunc_{\timeidx}^{\policyparams} +
    H_{k,\timeidx+1}^{\psi_{k,\timeidx+1}}(\state_{\timeidx+1}^{\policyparams};\policyparams) + \sum_{m'=1}^{k-1} p_{m'} \sum_{l=1}^{m'} H_{l,\timeidx+1}^{\psi_{l,\timeidx+1}}(\state_{\timeidx+1}^{\policyparams};\policyparams).
\end{align*}
\end{subequations}
We then obtain for any $m = 1, \ldots, k-1$ that
\begin{align}
    &\esssup_{\state \in \statespace} \; \Big\Vert H_{m,\timeidx}(\state;\policyparams) - H_{m,\timeidx}^{\psi_{m,\timeidx}}(\state;\policyparams) \Big\Vert \nonumber\\
    \text{\scriptsize{[\cref{eq:Hm-spectral}]}}&\quad = \esssup_{\state \in \statespace} \; \Bigg\Vert \VaR_{\alpha_m} \Big(\costfunc_{\timeidx}^{\policyparams} +
    \valuefunc_{\timeidx+1}(\state_{\timeidx+1}^{\policyparams};\policyparams) \Bigm|\state_{\timeidx}=\state \Big)
    - \sum_{l=1}^{m-1} H_{l,\timeidx}(\state;\policyparams)
    - H_{m,\timeidx}^{\psi_{m,\timeidx}}(\state;\policyparams)\Bigg\Vert \nonumber\\
    \text{\scriptsize{[\cref{eq:Hm-spectral}]}}&\quad = \esssup_{\state \in \statespace} \; \Bigg\Vert \VaR_{\alpha_m} \Big(\costfunc_{\timeidx}^{\policyparams} +
    \valuefunc_{\timeidx+1}(\state_{\timeidx+1}^{\policyparams};\policyparams) \Bigm|\state_{\timeidx}=\state \Big) \nonumber\\
    &\qquad\qquad\qquad - \VaR_{\alpha_{m-1}} \Big(\costfunc_{\timeidx}^{\policyparams} +
    \valuefunc_{\timeidx+1}(\state_{\timeidx+1}^{\policyparams};\policyparams) \Bigm|\state_{\timeidx}=\state \Big) \nonumber\\
    &\qquad\qquad\qquad - H_{m,\timeidx}^{\psi_{m,\timeidx}}(\state;\policyparams)\Bigg\Vert \nonumber\\
    \text{\scriptsize{[\cref{eq:V-spectral}]}}&\quad = \esssup_{\state \in \statespace} \; \Big\Vert \VaR_{\alpha_m} \Big(\rvdum_{\timeidx} \Bigm|\state_{\timeidx}=\state \Big) - \VaR_{\alpha_{m-1}} \Big(\rvdum_{\timeidx} \Bigm|\state_{\timeidx}=\state \Big) - H_{m,\timeidx}^{\psi_{m,\timeidx}}(\state;\policyparams) \Big\Vert \nonumber\\
    \text{\scriptsize{[\triangleineq{}]}}&\quad \leq \esssup_{\state \in \statespace} \; \Big\Vert \VaR_{\alpha_m} \Big(\rvdum^{\psi}_{\timeidx} \Bigm|\state_{\timeidx}=\state \Big) - \VaR_{\alpha_{m-1}} \Big(\rvdum^{\psi}_{\timeidx} \Bigm|\state_{\timeidx}=\state \Big) - H_{m,\timeidx}^{\psi_{m,\timeidx}}(\state;\policyparams) \Big\Vert \nonumber\\
    &\qquad\quad + \esssup_{\state \in \statespace} \; \Big\Vert \VaR_{\alpha_m} \Big(\rvdum_{\timeidx} \Bigm|\state_{\timeidx}=\state \Big) - \VaR_{\alpha_{m}} \Big(\rvdum^{\psi}_{\timeidx} \Bigm|\state_{\timeidx}=\state \Big) \Big\Vert \nonumber\\
    &\qquad\quad + \esssup_{\state \in \statespace} \; \Big\Vert \VaR_{\alpha_{m-1}} \Big(\rvdum_{\timeidx} \Bigm|\state_{\timeidx}=\state \Big) - \VaR_{\alpha_{m-1}} \Big(\rvdum^{\psi}_{\timeidx} \Bigm|\state_{\timeidx}=\state \Big) \Big\Vert \nonumber\\
    \text{\scriptsize{[\cref{eq:absolutely-continuous}]}}&\quad \leq \esssup_{\state \in \statespace} \; \Big\Vert \VaR_{\alpha_m} \Big(\rvdum^{\psi}_{\timeidx} \Bigm|\state_{\timeidx}=\state \Big) - \VaR_{\alpha_{m-1}} \Big(\rvdum^{\psi}_{\timeidx} \Bigm|\state_{\timeidx}=\state \Big) - H_{m,\timeidx}^{\psi_{m,\timeidx}}(\state;\policyparams) \Big\Vert \nonumber\\
    &\qquad\quad + 2 \esssup_{\state \in \statespace} \; \Big\Vert \rvdum_{\timeidx} - \rvdum^{\psi}_{\timeidx} \Big\Vert \nonumber\\
    \text{\scriptsize{[\triangleineq{}]}}&\quad \leq \esssup_{\state \in \statespace} \; \Big\Vert \VaR_{\alpha_m} \Big(\rvdum^{\psi}_{\timeidx} \Bigm|\state_{\timeidx}=\state \Big) - \VaR_{\alpha_{m-1}} \Big(\rvdum^{\psi}_{\timeidx} \Bigm|\state_{\timeidx}=\state \Big) - H_{m,\timeidx}^{\psi_{m,\timeidx}}(\state;\policyparams) \Big\Vert \nonumber\\
    &\qquad\quad + 2 \esssup_{\state \in \statespace} \; \Big\Vert H_{k,\timeidx+1}(\state;\policyparams) - H_{k,\timeidx+1}^{\psi_{k,\timeidx+1}}(\state;\policyparams) \Big\Vert \nonumber\\
    &\qquad\quad + 2 \sum_{m'=1}^{k-1} p_{m'} \sum_{l=1}^{m'} \esssup_{\state \in \statespace} \; \Big\Vert H_{l,\timeidx+1}(\state;\policyparams) - H_{l,\timeidx+1}^{\psi_{l,\timeidx+1}}(\state;\policyparams)\Big\Vert \nonumber\\
    \begin{split}
    \text{\scriptsize{[\cref{eq:proof-approx-lastHm-induction}]}}&\quad < \esssup_{\state \in \statespace} \; \Big\Vert \VaR_{\alpha_m} \Big(\rvdum^{\psi}_{\timeidx} \Bigm|\state_{\timeidx}=\state \Big) - \VaR_{\alpha_{m-1}} \Big(\rvdum^{\psi}_{\timeidx} \Bigm|\state_{\timeidx}=\state \Big) - H_{m,\timeidx}^{\psi_{m,\timeidx}}(\state;\policyparams) \Big\Vert \\
    &\qquad\quad + 2 \left( \varepsilon_{k,\timeidx+1} + \sum_{m'=1}^{k-1} p_{m'} \sum_{l=1}^{m'} \varepsilon_{l,\timeidx+1} \right).
    \label{eq:proof-spectral-induction-Hm}
    \end{split}
\end{align}
Note that in \cref{eq:proof-spectral-induction-Hm}, given the recursion form of \cref{eq:Hm-spectral}, the last accuracy term appears only once with $m=1$.
Similarly, we show for the function $H_{k,\timeidx}$ that
\begin{align}
    &\esssup_{\state \in \statespace} \; \Big\Vert H_{k,\timeidx}(\state;\policyparams) - H_{k,\timeidx}^{\psi_{k,\timeidx}}(\state;\policyparams) \Big\Vert \nonumber\\
    \text{\scriptsize{[\cref{eq:Hk-spectral}]}}&\quad = \esssup_{\state \in \statespace} \; \Bigg\Vert \riskmeas^{\spectrum} \Big(\costfunc_{\timeidx}^{\policyparams} + \valuefunc_{\timeidx+1}(\state_{\timeidx+1}^{\policyparams};\policyparams)
    \Bigm|\state_{\timeidx}=\state \Big) \nonumber\\
    &\qquad\qquad\qquad - \sum_{m=1}^{k-1} p_m \sum_{l=1}^{m} H_{l,\timeidx}(\state;\policyparams)
    - H_{k,\timeidx}^{\psi_{k,\timeidx}}(\state;\policyparams) \Bigg\Vert \nonumber\\
    \text{\scriptsize{[\cref{eq:V-spectral}, \cref{eq:Hm-spectral}]}}&\quad = \esssup_{\state \in \statespace} \; \Bigg\Vert \riskmeas^{\spectrum} \Big(\rvdum_{\timeidx} \Bigm|\state_{\timeidx}=\state \Big) - \sum_{m=1}^{k-1} p_m \; \VaR_{\alpha_{m}} \Big(\rvdum_{\timeidx} \Bigm|\state_{\timeidx}=\state \Big) - H_{k,\timeidx}^{\psi_{k,\timeidx}}(\state;\policyparams) \Bigg\Vert \nonumber\\
    \text{\scriptsize{[\triangleineq{}, \cref{eq:absolutely-continuous}]}}&\quad \leq \esssup_{\state \in \statespace} \; \Bigg\Vert \riskmeas^{\spectrum} \Big(\rvdum^{\psi}_{\timeidx} \Bigm|\state_{\timeidx}=\state \Big) - \sum_{m=1}^{k-1} p_m \; \VaR_{\alpha_{m}} \Big(\rvdum^{\psi}_{\timeidx} \Bigm|\state_{\timeidx}=\state \Big) - H_{k,\timeidx}^{\psi_{k,\timeidx}}(\state;\policyparams) \Bigg\Vert \nonumber\\
    &\qquad\quad + \left( 1 + \sum_{m=1}^{k-1} p_m \right) \esssup_{\state \in \statespace} \; \Big\Vert \rvdum_{\timeidx} - \rvdum^{\psi}_{\timeidx} \Big\Vert \nonumber\\
    \text{\scriptsize{[sum of $p_m$'s]}}&\quad \leq \esssup_{\state \in \statespace} \; \Bigg\Vert \riskmeas^{\spectrum} \Big(\rvdum^{\psi}_{\timeidx} \Bigm|\state_{\timeidx}=\state \Big) - \sum_{m=1}^{k-1} p_m \; \VaR_{\alpha_{m}} \Big(\rvdum^{\psi}_{\timeidx} \Bigm|\state_{\timeidx}=\state \Big) - H_{k,\timeidx}^{\psi_{k,\timeidx}}(\state;\policyparams) \Bigg\Vert \nonumber\\
    &\qquad\quad + 2 \esssup_{\state \in \statespace} \; \Big\Vert \rvdum_{\timeidx} - \rvdum^{\psi}_{\timeidx} \Big\Vert \nonumber\\
    \text{\scriptsize{[\triangleineq{}]}}&\quad \leq \esssup_{\state \in \statespace} \; \Bigg\Vert \riskmeas^{\spectrum} \Big(\rvdum^{\psi}_{\timeidx} \Bigm|\state_{\timeidx}=\state \Big) - \sum_{m=1}^{k-1} p_m \; \VaR_{\alpha_{m}} \Big(\rvdum^{\psi}_{\timeidx} \Bigm|\state_{\timeidx}=\state \Big) - H_{k,\timeidx}^{\psi_{k,\timeidx}}(\state;\policyparams) \Bigg\Vert \nonumber\\
    &\qquad\quad + 2 \esssup_{\state \in \statespace} \; \Big\Vert H_{k,\timeidx+1}(\state;\policyparams) - H_{k,\timeidx+1}^{\psi_{k,\timeidx+1}}(\state;\policyparams) \Big\Vert \nonumber\\
    &\qquad\quad + 2 \sum_{m'=1}^{k-1} p_{m'} \sum_{l=1}^{m'} \esssup_{\state \in \statespace} \; \Big\Vert H_{l,\timeidx+1}(\state;\policyparams) - H_{l,\timeidx+1}^{\psi_{l,\timeidx+1}}(\state;\policyparams)\Big\Vert \nonumber\\
    \begin{split}
    \text{\scriptsize{[\cref{eq:proof-approx-lastHm-induction}]}}&\quad < \esssup_{\state \in \statespace} \; \Bigg\Vert \riskmeas^{\spectrum} \Big(\rvdum^{\psi}_{\timeidx} \Bigm|\state_{\timeidx}=\state \Big) - \sum_{m=1}^{k-1} p_m \; \VaR_{\alpha_{m}} \Big(\rvdum^{\psi}_{\timeidx} \Bigm|\state_{\timeidx}=\state \Big) - H_{k,\timeidx}^{\psi_{k,\timeidx}}(\state;\policyparams) \Bigg\Vert \\
    &\qquad\quad + 2 \left( \varepsilon_{k,\timeidx+1} + \sum_{m'=1}^{k-1} p_{m'} \sum_{l=1}^{m'} \varepsilon_{l,\timeidx+1} \right).
    \label{eq:proof-spectral-induction-Hk}
    \end{split}
\end{align}
Using the universal approximation theorem on \cref{eq:proof-spectral-induction-Hm,eq:proof-spectral-induction-Hk}, for any $\varepsilon'_{m,\timeidx} > 0$, there exist \ANN{}s $H_{1,\timeidx}^{\psi_{1,\timeidx}},\ldots,H_{k,\timeidx}^{\psi_{k,\timeidx}}$ such that for any $m \in \{1,\ldots,k\}$, we have
\begin{equation}
    \esssup_{\state \in \statespace} \; \Big\Vert H_{m,\timeidx}(\state;\policyparams) - H_{m,\timeidx}^{\psi_{m,\timeidx}}(\state;\policyparams) \Big\Vert < \varepsilon'_{m,\timeidx} + 2 \left( \varepsilon_{k,\timeidx+1} + \sum_{m'=1}^{k-1} p_{m'} \sum_{l=1}^{m'} \varepsilon_{l,\timeidx+1} \right) =: \varepsilon_{m,\timeidx}.
    \label{eq:proof-approx-lastHm-end}
\end{equation}

This completes the proof by induction. 
Given \new{global errors $\varepsilon^{*}_1,\ldots,\varepsilon^{*}_{k}$} and sufficiently large neural network structures in terms of depth and number of nodes per layer, we can train the \ANN{}s to construct sequences $\{\varepsilon_{1,\timeidx}\}_{\timeidx \in \periodspace}, \ldots, \{\varepsilon_{k,\timeidx}\}_{\timeidx \in \periodspace}$ such that \new{$\varepsilon_{m,0} < \varepsilon^{*}_{m}$ for all $m=1,\ldots,k$}.
\end{proof}

\new{\begin{corollary}
    Suppose $\policy^{\policyparams}$ is a fixed policy, and its value function $\valuefunc_{\timeidx}(\state;\policyparams)$ is given in \cref{eq:lastV-spectral,eq:otherV-spectral}.
    Then for any $\varepsilon^{*}_1,\ldots,\varepsilon^{*}_{k} > 0$, there exist \ANN{}s denoted $H_{1,\timeidx}^{\psi_1},\ldots,H_{k,\timeidx}^{\psi_k}$ such that for any $\timeidx \in \periodspace$, we have
    \begin{equation}
        \esssup_{\state \in \statespace} \; \Bigg\Vert \valuefunc_{\timeidx}(\state; \policyparams) - \left( H_{k,\timeidx}^{\psi_k}(\state;\policyparams) + \sum_{m=1}^{k-1} p_m \sum_{l=1}^{m} H_{l,\timeidx}^{\psi_l}(\state;\policyparams) \right)\Bigg\Vert < \varepsilon^{*}.
    \end{equation}
    \label{thm:univ-approx-V-spectral}
\end{corollary}}

\new{\begin{proof}
    As a consequence of \cref{thm:univ-approx-H-spectral}, for any $\varepsilon^{*}_{1},\ldots,\varepsilon^{*}_{k} > 0$, there exist $H_{1,\timeidx}^{\valueparams_1},\ldots,H_{k,\timeidx}^{\valueparams_k}$ such that
    \begin{equation}
        \esssup_{\state \in \statespace} \; \Big\Vert H_{m,\eplength}(\state;\policyparams) - H_{m,\eplength}^{\psi_{m}}(\state;\policyparams) \Big\Vert < \varepsilon^{*}_{m}, \quad  \forall m = 1,\ldots,k,
    \end{equation}
    for any $\timeidx\in\periodspace$.
    Therefore, for the value function at any period $\timeidx \in \periodspace$, we have
    \begin{align*}
        &\esssup_{\state \in \statespace} \; \Bigg\Vert \valuefunc_{\timeidx}(\state; \policyparams) - \left( H_{k,\timeidx}^{\psi_{k}}(\state;\policyparams) + \sum_{m=1}^{k-1} p_m \sum_{l=1}^{m} H_{l,\timeidx}^{\psi_{l}}(\state;\policyparams) \right) \Bigg\Vert \\
        \text{\scriptsize{[\cref{eq:V-spectral}]}}&\quad = \esssup_{\state \in \statespace} \; \Bigg\Vert  H_{k,\timeidx}(\state;\policyparams) + \sum_{m=1}^{k-1} p_m \sum_{l=1}^{m} H_{l,\timeidx}(\state;\policyparams) \\
        &\qquad\qquad\qquad -  H_{k,\timeidx}^{\psi_{k}}(\state;\policyparams) - \sum_{m=1}^{k-1} p_m \sum_{l=1}^{m} H_{l,\timeidx}^{\psi_{l}}(\state;\policyparams) \Bigg\Vert \\
        \text{\scriptsize{[\triangleineq{}]}}&\quad \leq \esssup_{\state \in \statespace} \; \Big\Vert H_{k,\timeidx}(\state;\policyparams) - H_{k,\timeidx}^{\psi_{k}}(\state;\policyparams) \Big\Vert \\
        &\qquad\quad + \sum_{m=1}^{k-1} p_m \sum_{l=1}^{m} \esssup_{\state \in \statespace} \; \Big\Vert H_{l,\timeidx}(\state;\policyparams) - H_{l,\timeidx}^{\psi_{l}}(\state;\policyparams) \Big\Vert \\
        \text{\scriptsize{[\cref{eq:proof-approx-lastHm-end}]}}&\quad < \varepsilon^{*}_{k} + \sum_{m=1}^{k-1} p_m \sum_{l=1}^{m} \varepsilon^{*}_{l}.
    \end{align*}
    Since $\varepsilon^{*}_{1},\ldots,\varepsilon^{*}_{k}$ are arbitrary, this completes the proof.
\end{proof}}

\subsection{Policy Update}

The update of the policy described in \cref{ssec:update-policy} requires us to derive the gradient of the value function -- the following corollary is a consequence of \cref{thm:gradient-V}.
\begin{corollary}
    Suppose the logarithm of transition probabilities $\log \PP^{\policyparams}(\action,\statedum | \state)$ is a differentiable function in $\policyparams$ when $\PP^{\policyparams}(\action,\statedum | \state) \neq 0$, and its gradient wrt $\policyparams$ is bounded for any $(\action, \state) \in \actionspace \times \statespace$.
    Then, for any state $\state \in \statespace$, the gradient of the value function at period \new{$\eplength$} given in \cref{eq:lastV-spectral} is
    \begin{subequations}
    {\small{
    \begin{equation}
    	\newmath{ \grad{\policyparams} \valuefunc_{\eplength} (\state; \policyparams) = \sum_{m=1}^{k-1} \frac{p_m}{1-\alpha_{m}} \EE_{\PP^{\policyparams}(\cdot,\cdot | \state_{\eplength}=\state)} \bigg[
    	\Big( \costfunc_{\eplength}^{\policyparams} - \lambda_{m}^{*} \Big)_{+}
    	\bigg( \grad{\policyparams} \log \policy^{\policyparams} (\action | \state_{\eplength})\Big\rvert_{\action=\action_{\eplength}^{\policyparams}} \bigg) \bigg] },
    	\label{eq:score-gradient-reinforce-spectral}
    \end{equation}
    }}and the gradient of the value function at periods \new{$\timeidx \in \periodspace \setminus \{\eplength\}$} given in \cref{eq:otherV-spectral} is
    {\small{
    \begin{equation}
    \begin{split}
    	\grad{\policyparams} \valuefunc_{\timeidx} (\state; \policyparams) &= \sum_{m=1}^{k-1} \frac{p_m}{1-\alpha_{m}} \EE_{\PP^{\policyparams}(\cdot,\cdot | \state_{\timeidx}=\state)} \bigg[
    	\Big( \costfunc_{\timeidx}^{\policyparams} + \valuefunc_{\timeidx+1}(\state_{\timeidx+1}^{\policyparams}; \policyparams) - \lambda_{m}^{*} \Big)_{+}
    	\bigg( \grad{\policyparams} \log \policy^{\policyparams} (\action | \state_{\timeidx})\Big\rvert_{\action=\action_{\timeidx}^{\policyparams}} \bigg) \bigg] \\
    	&\qquad\qquad\qquad + 
    	\EE_{\PP^{\policyparams}(\cdot,\cdot | \state_{\timeidx}=\state)} \bigg[
    	\bigg( \grad{\policyparams} \valuefunc_{\timeidx+1}(\statedum; \policyparams)\Big\rvert_{\statedum=\state_{\timeidx+1}^{\policyparams}} \bigg) \;
    	\weight_{m}^{*}(\action_{\timeidx}^{\policyparams}, \state_{\timeidx+1}^{\policyparams})
    	\bigg],
    \end{split}	\label{eq:score-gradient-reinforce-spectral-2}
    \end{equation}
    }}
    \end{subequations}
    where $(\weight^{*}_{m}, \lambda^{*}_{m})$ are any saddle-points of the Lagrangian functions for all $\CVaR_{\alpha_{m}}$ in \cref{eq:lastV-spectral,eq:otherV-spectral}, respectively.
    \label{thm:gradient-V-spectral}
\end{corollary}

\begin{proof}
Spectral risk measures are also coherent, and we use the representation theorem and the Envelope theorem for saddle-point problems in an analogous way to the proof of \cref{thm:gradient-V}.
For any period $\timeidx \in \periodspace \setminus \{\eplength-1\}$, we have
\begin{subequations}
\begin{align}
	\newmath{\grad{\policyparams} \valuefunc_{\eplength}(\state;\policyparams)} &=
    \sum_{m=1}^{k-1} p_m \; \newmath{\grad{\policyparams} \CVaR_{\alpha_m} \Big(\costfunc_{\eplength}^{\policyparams} \Bigm|\state_{\eplength} = \state \Big)},\quad \text{and} \label{eq:grad-V-spectral1}
    \\
	\grad{\policyparams} \valuefunc_{\timeidx}(\state;\policyparams) &=
    \sum_{m=1}^{k-1} p_m \; \grad{\policyparams} \CVaR_{\alpha_m} \Big(\costfunc_{\timeidx}^{\policyparams} +
    \valuefunc_{\timeidx+1}(\state_{\timeidx+1}^{\policyparams};\policyparams)
    \Bigm|\state_{\timeidx} = \state \Big). \label{eq:grad-V-spectral2}
\end{align}
\end{subequations}
Using \cref{thm:gradient-V} on \cref{eq:grad-V-spectral1,eq:grad-V-spectral2} yields the desired result.
\end{proof}

Since saddle-points for each individual $\CVaR_{\alpha_{m}}(\rv)$ satisfy $\weight_{m}^{*}(\omega) = \frac{1}{1-\alpha_{m}} \Ind_{\rv(\omega) > \lambda_{m}^{*}}$ where $\lambda_{m}^{*}$ is any $\alpha_{m}$-quantile of $\rv$, one can use the sum of \ANN{}s $\sum_{l=1}^{m} H_{l,\timeidx}^{\psi_l}$, $m=1,\ldots,k-1$, estimating respectively $\VaR_{\alpha_{m}}$, to obtain an estimation of a global saddle-point of the spectral risk measure.
We thus replace the loss function $\Ll^{\policyparams}$ in \cref{eq:loss-func-gradient-reinforce} by
{\small{
\begin{equation}
\begin{split}
    &\newmath{\sum_{\timeidx \in \periodspace \setminus \{\eplength\}}} \sum_{m=1}^{k-1} \sum_{\batchidx=1}^{\Nbatchs}
    \frac{p_m}{1-\alpha_{m}} \,
	\bigg( \costfunc_{\timeidx}^{(\batchidx)} + \valuefunc_{\timeidx+1}^{\valueparams}(\state_{\timeidx+1}^{(\batchidx)}; \policyparams) - \sum_{l=1}^{m} H_{l,\timeidx}^{\psi_l}(\state_{\timeidx}^{(\batchidx)}; \policyparams) \bigg)_{+}
	\bigg( \grad{\policyparams} \log \policy^{\policyparams} (\action | \state_{\timeidx}^{(\batchidx)})\Big\rvert_{\action=\action_{\timeidx}^{(\batchidx)}} \bigg)
    \\
	&\qquad + \sum_{m=1}^{k-1} \sum_{\batchidx=1}^{\Nbatchs}
    \frac{p_m}{1-\alpha_{m}} \,
    \newmath{\bigg( \costfunc_{\eplength}^{(\batchidx)} - \sum_{l=1}^{m} H_{l,\eplength}^{\psi_l}(\state_{\eplength}^{(\batchidx)}; \policyparams) \bigg)_{+}
	\bigg( \grad{\policyparams} \log \policy^{\policyparams} (\action | \state_{\eplength}^{(\batchidx)})\Big\rvert_{\action=\action_{\eplength}^{(\batchidx)}} \bigg)}.
\end{split} \label{eq:loss-func-gradient-spectral} \tag{L4}
\end{equation}
}}

For any spectral risk measure having a spectrum with finite support, one can use our proposed methodology and risk-aware actor-critic algorithm in \cref{algo:actor-critic} by (i) creating additional \ANN{}s for the different quantiles; and (ii) replacing the loss functions \cref{eq:loss-critic,eq:loss-func-gradient-reinforce} with \cref{eq:loss-func-dynamic-spectral,eq:loss-func-gradient-spectral}, respectively.

\endgroup


\section{Conclusion}
\label{sec:conclusion}

In this paper, we developed a novel setting to solve finite-horizon \RL{} problems where the objective function consists of a time-consistent, dynamic spectral risk measure of costs induced by a randomised policy.
Our proposed approach, which makes use of the conditional elicitability of spectral risk functionals with a finite support spectrum, provides an improvement over the existing nested simulation framework.
Indeed, it requires less memory and computation resources by estimating the risk and performing a policy gradient method using only full episodes.
It also gives a risk-aware \RL{} algorithm for solving this class of sequential decision making problems in scenarios where simulating additional transitions is too costly or even infeasible.

Part of our future work consists of extending this framework and providing risk-aware \RL{} alternatives.
For instance, from a risk-aware perspective, an agent may wish to consider deterministic policies to avoid introducing extra randomness in the optimisation problem, or robustify its actions against the uncertainty from the environment.
It is also of great interest to understand how such algorithms perform in real-life applications from direct interactions with real environments, and quantify the gain in performance of time-consistent policies as opposed to precommitment strategies.
Thus providing insights on how to devise a universal model-agnostic approach for solving \RL{} problems with time-consistent dynamic risk measures.

\appendix
\begingroup
\allowdisplaybreaks

\section{Hyperparameters}
\label{sec:appendix-hyperparams}

In this section, we expand on the hyperparameters of our \RL{} algorithm used in the different examples of \cref{sec:experiments}.
We emphasise that the choices of hyperparameters, optimisation rules and learning rate schedulers depend on the specific experiment, and one may tune them to accelerate the learning process.
We also refer the reader to our \href{https://github.com/acoache/RL-ElicitableDynamicRisk}{Github repository RL-ElicitableDynamicRisk} for further information on the code implementation and architecture.

\subsection{Statistical Arbitrage Example}

For each \CVaR{} with threshold $\alpha$, our actor-critic uses the following hyperparameters during the training phase.
All \ANN{}s, i.e. the policy $\policy^{\policyparams}$, the \VaR{} $H_{1,\timeidx}^{\psi_1}$, and the difference between \CVaR{} and \VaR{} $H_{2,\timeidx}^{\psi_2}$, consist of five layers of 16 hidden nodes each with SiLU activation functions.
$\policy^{\policyparams}$ has a linear transformation of a sigmoid activation function that maps to $(-\action_{\min}, \action_{\max})$ for its output layer, $H_{1,\timeidx}^{\psi_1}$ has no activation function, and $H_{2,\timeidx}^{\psi_2}$ has a softplus activation function to ensure the outputs are non-negative.
The learning rate for $\policy^{\policyparams}$ initially starts at $4 \times 10^{-3}$, and decays by 0.95 every 50 epochs until it reaches $5 \times 10^{-4}$.
The learning rates for $H_{1,\timeidx}^{\psi_1}$ and $H_{2,\timeidx}^{\psi_2}$ are of the order of $5 \times 10^{-3}$ and decay by 0.95 every 100 epochs.

The critic procedure -- i.e. estimation of the value function -- is performed for 1,000 epochs, where the target networks are replaced every 400 epochs, and it uses mini-batches of 750 full episodes.
The actor procedure -- i.e. update of the policy -- is performed for 30 epochs with mini-batches of $500/(1-\alpha)$ full episodes.
We execute the whole algorithm for 1,500 iterations, where a single iteration corresponds to both the critic and actor procedures.

We train the models on the \href{https://docs.alliancecan.ca/wiki/Niagara}{Niagara servers, managed by the Digital Research Alliance of Canada}. Every Python script runs on a single CPU node composed of 40 cores, i.e. two sockets with 20 Intel Skylake cores at 2.4GHz.
The training times for the elicitable and nested simulation approaches are approximately 5 and 9 hours, respectively.

\subsection{Portfolio Allocation Example}

The \ANN{} $\policy^{\policyparams}$ giving the mean of the Gaussian distribution consists of five layers of 16 hidden nodes each with SiLU activation functions, a softplus output activation function, and a learning rate that starts at $5 \times 10^{-3}$ and decays by 0.97 every 100 epochs until it reaches $3 \times 10^{-4}$.
The \ANN{}s $H_{1,\timeidx}^{\psi_1},H_{2,\timeidx}^{\psi_2}$ characterising the value function are both composed of five layers of 16 hidden nodes each with SiLU activation functions -- $H_{1,\timeidx}^{\psi_1}$ has no activation function, while $H_{2,\timeidx}^{\psi_2}$ has a softplus activation function.
Their learning rates initially start at $5 \times 10^{-3}$ and decay by 0.95 every 50 epochs.

For the experiments where price dynamics are given by either one of the \SDE{}s in \cref{eq:price-GBM,eq:price-meanrev}, the actor-critic algorithm runs for 2,000 iterations, alternating between the critic procedure -- i.e. estimation of the value function for 1,000 epochs with mini-batches of 1,000 full episodes, where the target networks are replaced every 300 epochs -- and the actor procedure -- i.e. update of the policy for 10 epochs with mini-batches of 1,000$/(1-\alpha)$ full episodes.
For the experiments where price dynamics are simulated from the VECM, the actor-critic algorithm instead runs for 4,000 iterations, and we increase the number of epochs for the critic procedure to 2,000 epochs.
The results shown in \cref{fig:portfolio-VECM} are obtained after approximately 24 hours of training on the Niagara servers.


\section{Tables}
\label{sec:appendix-tables}

\new{This section presents a statistical analysis of the first set of experiments conducted in  \cref{ssec:portfolio-allocation}.
The results are presented in \cref{tab:dynamic-cvar}, which provide the \ANN{} estimates of the dynamic conditional value-at-risk (CVaR) for the optimal policies over 14 training runs.
Each row of the table corresponds to the optimal strategy for the \CVaR{} level reported in the left-most column, and each column represents a specific \CVaR{} level.  The reported number is the estimate of the \CVaR{} at the level reported in the column heading. For instance, the first row in \cref{tab:dynamic-cvar} uses the optimal strategy obtained when minimising the dynamic \CVaR{} at level $0.01$, and the third column reports the dynamic \CVaR{} at level $0.05$ as $-0.071$. The cells colored red correspond to the smallest value within that column.
Notably, in \cref{tab:dynamic-cvar}, each column is minimised on the row corresponding to the column's \CVaR{} level as expected. For example, the column labeled $\alpha= 0.05$ is minimised in the second row corresponding to the optimal dynamic $\CVaR_{0.05}$ strategy.
}
\begin{table}[htbp]
    \centering
    \begin{minipage}[b]{0.8\textwidth}
        \vspace{1ex}\centering
        \begin{tabular}{cc c c c c} 
        \toprule\toprule
        \rowcolor{blue!10} Criterion & $\alpha=0$ & $\alpha=0.01$ & $\alpha=0.05$ & $\alpha=0.1$ & $\alpha=0.9$ \\ 
        \midrule
        \rowcolor{blue!10} $\CVaR_{0.01}$ & {\color{red}-0.101} & {\color{red}-0.097} & -0.071 & -0.006 & 0.641 
        \\
        \rowcolor{blue!10}
        $\CVaR_{0.05}$ & -0.097 & -0.089 & {\color{red}-0.075} & -0.021 & 0.518 
        \\
        \rowcolor{blue!10}
        $\CVaR_{0.1}$ & -0.058 & -0.049 & -0.040 & {\color{red}-0.024} & 0.295 
        \\
        \rowcolor{blue!10}
        $\CVaR_{0.9}$ & -0.046 & -0.041 & -0.032 & -0.011 & {\color{red}0.267} \\
        \bottomrule\bottomrule
        \end{tabular}
        ~\\[0.5em]
	\end{minipage}
	\caption{\new{Approximated dynamic \CVaR{}s at threshold $\alpha$ (columns) of the costs for the first set of experiments in \cref{ssec:portfolio-allocation} with different criteria (rows).}}
	\label{tab:dynamic-cvar}
\end{table}

\section{Parameter Estimates of the VECM}
\label{sec:appendix-vecm}

In this section, we provide a brief overview of vector error correction models (VECMs), and show the parameter estimates obtained with our dataset for replication purposes.
For more information on VECMs, we refer the reader to Chapter 7 of \cite{lutkepohl2005new} and the code from the \href{https://www.statsmodels.org/dev/vector_ar.html#vector-error-correction-models-vecm}{Python library named statsmodels.tsa.vector\_ar.vecm}.

In our case, the VECM has the form
\begin{equation*}
    \Delta Y_{\tau} = \alpha \beta\tr Y_{\tau-1} + \Gamma_1 \Delta Y_{\tau-1} + \cdots + \Gamma_{k_{ar}-1} \Delta Y_{\tau-k_{ar}+1} + C D_{\tau} + u_\tau, \label{eq:vecm-model}
\end{equation*}
where $Y_{\tau}$ is the $d$-dimensional process of interest, $\alpha,\beta$ are $(d \times r)$ matrices of rank $r$, $k_{ar}$ the number of lagged differences, $C$ the parameter estimates of the deterministic terms outside the co-integration relation $D_\tau$, and $u_\tau$ a $d$-dimensional white noise with covariance matrix $\Sigma_{u}$.
Without loss of generality, we order alphabetically the asset prices in the random vector $Y_{\tau}$, i.e. the $d=8$ dimensions correspond respectively to assets AAL, AMZN, CCL, FB, IBM, INTC, LYFT and OXY.

When fitting this model on a normalised version of the stock tickers with the transformation $Y_{\tau}^{(i)} = \log(\price_{\tau}^{(i)} / \price_{0}^{(i)})$, we estimate $\hat{k}_{ar} = 0$ (which implies $\hat{\Gamma} = \mathbf{0}$) and $\hat{r}=2$ using the Bayesian information criterion (BIC).
We also have the following parameter estimates:
\begin{align*}
    \hat{\alpha} &= \begin{bmatrix}
  -9.1 & -4.0\\
  1.4 & 0.7\\
  -3.7 & -2.3\\
  2.1 & 1.1\\
  -1.4 & -0.5\\
  -0.6 & 0.2\\
  -1.7 & -0.9\\
  -1.6 & -0.3\\
\end{bmatrix} \times 10^{-2}, \qquad
\hat{\beta} = \begin{bmatrix}
  1.0 & 0.0\\
  0.0 & 1.0\\
  -2.1 & 3.3\\
  3.5 & -9.5\\
  -0.6 & 3.1\\
  1.9 & -4.2\\
  -1.0 & 1.6\\
  1.3 & -3.1\\
\end{bmatrix}, \\
    \hat{C} &= \begin{bmatrix}
  -1.3 & 1.5 & -7.6 & 2.9 & 1.0 & 5.2 & -1.3 & 4.6\\
\end{bmatrix} \times 10^{-5}, \\
    \diag \hat{\Sigma}_u &= \begin{bmatrix}
  4.3 & 1.8 & 5.1 & 2.2 & 1.9 & 2.5 & 4.1 & 5.2\\
\end{bmatrix} \times 10^{-2},
    \\
    \widehat{\text{Corr}}_u &= \begin{bmatrix}
  1.0 & 0.1 & 0.7 & 0.2 & 0.4 & 0.3 & 0.4 & 0.4\\
  0.1 & 1.0 & 0.2 & 0.6 & 0.3 & 0.5 & 0.2 & 0.2\\
  0.7 & 0.2 & 1.0 & 0.3 & 0.5 & 0.4 & 0.6 & 0.5\\
  0.2 & 0.6 & 0.3 & 1.0 & 0.4 & 0.5 & 0.3 & 0.3\\
  0.4 & 0.3 & 0.5 & 0.4 & 1.0 & 0.5 & 0.4 & 0.5\\
  0.3 & 0.5 & 0.4 & 0.5 & 0.5 & 1.0 & 0.3 & 0.3\\
  0.4 & 0.2 & 0.6 & 0.3 & 0.4 & 0.3 & 1.0 & 0.4\\
  0.4 & 0.2 & 0.5 & 0.3 & 0.5 & 0.3 & 0.4 & 1.0\\
\end{bmatrix},
\end{align*}
where $\text{Corr}_u := (\diag{\Sigma_u})^{-\frac{1}{2}} \, \Sigma_u \, (\diag{\Sigma_u})^{-\frac{1}{2}}$.

\endgroup

\section*{Acknowledgments}
We are grateful to participants of the Victoria Seminar series.
\new{We also thank the anonymous referees for their thoughtful comments during the review process.}

\bibliographystyle{siamplain}
\bibliography{bib-files/references}

\end{document}